\documentclass[a4paper,usenames,dvipsnames,twoside,11pt]{article}
\usepackage[margin=3.0cm]{geometry}
\usepackage[round]{natbib}
\usepackage[utf8x]{inputenc} 
\usepackage[T1]{fontenc}    
\usepackage[colorlinks=True,linkcolor=blue,citecolor=blue]{hyperref}       

\usepackage{url}            
\usepackage{booktabs}       
\usepackage{amsfonts}       
\usepackage{nicefrac}       
\usepackage{microtype}      
\usepackage{amsmath}
\usepackage{amssymb}
\usepackage{amsthm}
\usepackage{bm}
\usepackage{graphicx}
\usepackage{caption}
\usepackage{subcaption}
\usepackage{soul}
\usepackage{cancel}

\usepackage{tikz}
\usetikzlibrary{bayesnet,calc,angles,quotes}

\newcommand*{\rb}[1]{}

\title{Kernel interpolation with continuous volume sampling}

\usepackage[colorlinks=True,linkcolor=blue,citecolor=blue]{hyperref}

\usepackage{mathtools}
\usepackage{bbm}
\usepackage{amssymb}
\usepackage{amsmath}
\usepackage{amsfonts}
\usepackage{bm}

\usepackage{microtype}
\usepackage{graphicx}
\usepackage{booktabs} 

\DeclareMathOperator*{\KDPP}{\mathfrak{K}}
\DeclareMathOperator{\Tr}{Tr}

\DeclareMathOperator{\Det}{Det}

\DeclareMathOperator{\Span}{\mathrm{Span}}

\DeclareMathOperator{\VS}{\mathrm{VS}}

\DeclareMathOperator{\Tran}{\intercal}

\DeclareMathOperator{\EX}{\mathbb{E}}

\DeclareMathOperator{\F}{\mathcal{F}}
\DeclareMathOperator{\X}{\mathcal{X}}
\DeclareMathOperator{\Ltwo}{\mathbb{L}_{2}(\mathrm{d} \omega)}
\DeclareMathOperator{\Mu}{\mathrm{d}\omega(x)}
\DeclareMathOperator{\MuTen}{\otimes\mathrm{d}\omega(x_{i})}
\DeclareMathOperator{\Ns}{\mathbb{N}^{*}}

\def\UN{\:\mathcal{U}_N}
\def\UNm{\:\mathcal{U}_N^m}
\def\ind{\mathbbm{1}}

  \newtheorem{theorem}{Theorem}
  \newtheorem{lemma}[theorem]{Lemma}
  \newtheorem{proposition}[theorem]{Proposition}

  \newtheorem{corollary}[theorem]{Corollary}
  \newtheorem{definition}[theorem]{Definition}

  \newtheorem{assumption}{Assumption}

\author{
Ayoub Belhadji$^{\dagger}$\footnote{Corresponding author: \href{mailto:ayoub.belhadji@centralelille.fr}{ayoub.belhadji@centralelille.fr}}, R\'emi Bardenet$^{\dagger}$, Pierre Chainais$^{\dagger}$\\
\small $^{\dagger}$ Univ. Lille, CNRS, Centrale Lille, UMR 9189 - CRIStAL, 59651 Villeneuve d’Ascq, France \\
}

\date{}

\begin{document}
\maketitle




\begin{abstract}

A fundamental task in kernel methods is to pick nodes and weights, so as to approximate a given function from an RKHS by the weighted sum of kernel translates located at the nodes. This is the crux of kernel density estimation, kernel quadrature, or interpolation from discrete samples. Furthermore, RKHSs offer a convenient mathematical and computational framework. We introduce and analyse continuous volume sampling (VS), the continuous counterpart --~for choosing node locations~-- of a discrete distribution introduced in \citep{DeVe06}. Our contribution is theoretical: we prove almost optimal bounds for interpolation and quadrature under VS. While similar bounds already exist for some specific RKHSs using ad-hoc node constructions, VS offers bounds that apply to any Mercer kernel and depend on the spectrum of the associated integration operator. We emphasize that, unlike previous randomized approaches that rely on regularized leverage scores or determinantal point processes, evaluating the pdf of VS only requires pointwise evaluations of the kernel. VS is thus naturally amenable to MCMC samplers.


\end{abstract}


\tableofcontents

\section{Introduction}
\label{sec:introduction}
Kernel approximation is a recurrent task in machine learning  \citep*{HaTiFr09}[Chapter 5], signal processing \citep{Uns00} or numerical quadrature \citep{Lar72}. Expressed in its general form, we are given a reproducing kernel Hilbert space $\mathcal{F}$ (RKHS; \citealp{BeTh11}) of functions over $\mathcal{X}$, with a symmetric kernel $k: \X \times \X \rightarrow \mathbb{R}_{+}$, and an element $\mu : \X \rightarrow \mathbb{R}$ of $\mathcal{F}$. We ask for conditions on a \emph{design} $\bm{x} = (x_{1},\dots,x_N)\in \mathcal{X}^N$, and on the corresponding weights $w_1,\dots,w_N$, such that the RKHS norm
\begin{equation}
\left\Vert \mu - \sum_{i=1}^N w_{i}k(x_{i},.)\right\Vert_{\mathcal{F}}
\label{e:residual}
\end{equation}
is small. In other words, $\mu$ should be well reconstructed in $\mathcal{F}$ by the weighted design.

Measuring the error in RKHS norm has a computational advantage. Indeed, minimizing \eqref{e:residual} boils down to minimizing a quadratic form; and given a design $\bm{x}$ such that $\Det \bm{K}(\bm{x}) = \Det (k(x_i,x_j)) >0$, Equation~\eqref{e:residual} has a unique set of minimizing weights. The minimizer corresponds to the weights $\hat{\bm{w}} = \bm{K}(\bm{x})^{-1} \mu(\bm{x})$, where $\mu(\bm{x})\in\mathbb{R}^N$ contains the evaluation of $\mu$ at the $N$ design nodes $x_{i}$. Whenever the weights are chosen to be $\hat{\bm{w}}$, the sum in \eqref{e:residual} takes the same values as $\mu$ at the nodes $x_{i}$, and the optimal value of \eqref{e:residual} is thus called \emph{interpolation error}; otherwise we speak of \emph{approximation error}.
Note that when the kernel $k$ is bounded, guarantees in RKHS norm translate to guarantees in the supremum norm.

In this work, we propose and analyze the interpolation based on a random design drawn from a distribution called \emph{continuous volume sampling}, which favours designs $\bm{x}$ with a large value of $\Det \bm{K}(\bm{x})$. After introducing this new distribution, we prove non-asymptotic guarantees on the interpolation error which depend on the spectrum of the kernel $k$. Previous kernel-based randomized designs, both i.i.d. \citep{Bac17} and repulsive \citep{BeBaCh19}, can be hard to compute in practice since they require access to the Mercer decompostion of $k$. We show here that continuous volume sampling enjoys similar error bounds as well as some additional interpretable geometric properties, while having a joint density that can be evaluated as soon as one can evaluate the RKHS kernel $k$. In particular, this opens the possibility of Markov chain Monte Carlo samplers \citep{ReGh19}.

The rest of the article is organized as follows. Section~\ref{s:relatedWork} reviews kernel-based interpolation. In Section~\ref{sec:VS_DPP}, we define continuous volume sampling and relate it to projection determinantal point processes, as used by \cite{BeBaCh19}. Section~\ref{sec:main_results} contains our main results while Section~\ref{sec:steps_proof} contains sketches of all proofs with pointers to the appendices for missing details. Finally, in Section~\ref{sec:discussion}, we conclude and discuss some consequences of our results beyond kernel interpolation.

\paragraph{Notation and assumptions.}\label{paragraph:notation}
We assume that $\mathcal{X}$ is equipped with a Borel measure $\mathrm{d}\omega$, and that the support of $\mathrm{d}\omega$ is $\mathcal{X}$.
 Let $\mathbb{L}_{2}(\mathrm{d}\omega)$ be the Hilbert space of square integrable, real-valued functions on $\mathcal{X}$, with inner product $\langle \cdot, \cdot \rangle_{\mathrm{d}\omega}$, and associated norm $\|.\|_{\mathrm{d}\omega}$.
%


\begin{assumption}\label{hyp:integrable_diagonal}
$\displaystyle \int_{\X} k(x,x) \mathrm{d}\omega(x) < +\infty$.
\end{assumption}
Under Assumption~\ref{hyp:integrable_diagonal}, define the integral operator
\begin{equation}
\bm{\Sigma} f (\cdot) = \int_{\mathcal{X}} k(\cdot,y)f(y) \mathrm{d}\omega(y), \quad f \in \mathbb{L}_{2}(\mathrm{d}\omega).
\end{equation}
By construction, $\bm{\Sigma}$ is self-adjoint, positive semi-definite, and trace-class \citep{Sim05}.
%
For $m \in \Ns$, denote by $e_{m}$ the $m$-th eigenfunction of $\bm{\Sigma}$, normalized so that $\|e_{m}\|_{\mathrm{d}\omega} = 1$, and $\sigma_{m}$ the corresponding eigenvalue. Assumption~\ref{hyp:integrable_diagonal} implies that the embedding operator $I_{\mathcal{F}}: \mathcal{F} \longrightarrow \mathbb{L}_{2}(\mathrm{d}\omega)$ is compact; moreover, since $\mathrm{d}\omega$ is of full support in $\mathcal{X}$, $I_{\mathcal{F}}$ is injective \citep{StCh08}. This implies a Mercer-type decomposition of $k$,
\begin{equation}
k(x,y)= \sum\limits_{m \in \mathbb{N}^{*}}\sigma_{m}e_{m}(x)e_{m}(y),
\end{equation}
where $\mathbb{N}^{*} = \mathbb{N} \smallsetminus \{0\}$ and the convergence is pointwise \citep{StSc12}. The eigenvalues $(\sigma_m)$ are assumed to be non-increasing. Moreover, for $m \in \mathbb{N}^{*}$, we write $e_{m}^{\mathcal{F}} = \sqrt{\sigma_{m}}e_{m}$. Since $I_{\mathcal{F}}$ is injective, $(e_{m}^{\mathcal{F}})_{m \in \mathbb{N}^{*}}$ is an orthonormal basis of $\mathcal{F}$ \citep{StSc12}. Unless explicitly stated, we assume that $\mathcal{F}$ is dense in $\mathbb{L}_{2}(\mathrm{d}\omega)$, so that $(e_{m})_{m \in \mathbb{N}^{*}}$ is an orthonormal basis of $\mathbb{L}_{2}(\mathrm{d}\omega)$. For more intuition, under these assumptions, $f \in \mathcal{F}$ if and only if $\sum_m \sigma_{m}^{-1} \langle f,e_{m} \rangle_{\mathbb{L}_{2}(\mathrm{d} \omega)}^{2}$ converges.
For $\bm{x} \in \X^{N}$, we define $\bm{K}(\bm{x}) := k(x_{i},x_{j})_{i,j \in [N]}$. If $\Det \bm{K}(\bm{x})>0$,  the subspace $\mathcal{T}(\bm{x}) = \Span k(x_{i},.)_{i \in [N]}$ is of dimension $N$; we denote by $\Pi_{\mathcal{T}(\bm{x})}$ the $\langle.,. \rangle_{\F}$-orthogonal projection on $\mathcal{T}(\bm{x})$.
Finally, for $N \in \Ns$, we will often sum over the sets
\begin{align}
	\UNm &=  \{ U \subset \Ns, |U| = N,\: m \notin U \},\\
  \UN &=  \{ U \subset \Ns, |U| = N \}.
\end{align}
Finally, define the approximation error
\begin{equation}\label{def:E}
	\mathcal{E}(\mu;\bm{x},\bm{w}) = \|\mu - \sum\limits_{i \in [N]} w_{i} k(x_{i},.)\|_{\F},
\end{equation}
where $[N]=\{1,\dots,N\}$. If $\Det \bm{K}(\bm{x}) >0$, let $\hat{\bm{w}} = \bm{K}(\bm{x})^{-1} \mu(\bm{x})$ and define the interpolation error
\begin{align}\label{def:Ei}
	\mathcal{E}(\mu;\bm{x}) &= \|\mu - \sum\limits_{i \in [N]} \hat{w} _{i} k(x_{i},.)\|_{\F}\\
  & = \|\mu - \Pi_{\mathcal{T}(\bm{x})} \mu\|_{\F}.
\end{align}

\section{Related work}\label{s:relatedWork}
This section reviews some results on kernel interpolation to better situate our contributions. The literature on this topic is prolific and cannot be covered in details here. In particular, we start by reviewing results on optimal kernel quadrature, a particular case of kernel interpolation.

\subsection{Interpolation for optimal kernel quadrature}\label{sec:review_optimal_kernel_quadrature}

Given $g \in \Ltwo$, kernel quadrature deals with approximating the integrals
\begin{equation}
\:\int_{\X} fg \,\mathrm{d}\omega \approx \sum\limits_{i \in [N]} w_{i}f(x_{i}),\quad f\in\F,
\end{equation}
where the weights $w_{i}$ do not depend on $f$.
In principle, it is easy to control the error uniformly in $f\in\mathcal{F}$. Indeed,
\begin{equation}\label{eq:upper_bound_integration_error}
\left| \int_{\X} fg \,\mathrm{d}\omega - \sum_{i \in [N]} w_{i}f(x_{i})\right| \leq \|f\|_{\F} \, \mathcal{E}(\mu_{g};\bm{x},\bm{w}),
\end{equation}
where
$ \displaystyle
\mu_{g} = \int_{\X}g(x)k(x,.) \mathrm{d}\omega(x){=\bm{\Sigma} g}
$
is the so-called \emph{embedding}\footnote{In case $g =1$, $\mu_{g}$ is classically called the mean-element of the measure $\mathrm{d} \omega$ \citep{SmGrSoSc07}.} of $g$ in the RKHS $\F$.

An upper bound on the approximation error of $\mu_{g}$ implies an upper bound on the integration error that is uniform over any bounded subset of $\F$. This observation sparked intense research on the kernel approximation of embeddings $\mu_{g}$. Among kernel approximation results, we pay a particular attention to interpolation, i.e., approximation with optimal weights. In the sequel, we call \emph{optimal kernel quadrature} the quadrature based on optimal weights $\hat{\bm{w}}$ minimizing \eqref{e:residual} for a given set of nodes.

\cite{Boj81} proved that, for $g=1$, the interpolation of $\mu_{g}$ using the uniform grid over $\X = [0,1]$ has an error in $\mathcal{O}(N^{-2s})$ if $\F$ is the periodic Sobolev space of order $s$, and that any set of nodes leads to that rate at least. A similar rate was proved for $g$ not constant \citep{NoUlWo15} even though it is only asymptotically optimal in that case.

In the quasi-Monte Carlo (QMC) literature, several designs were investigated for $\X = [0,1]^{d}$, $g = 1$ and $\F$ that may not even be a Hilbert space; see \citep{DiPi10}. In this context, the term \emph{QMC quadrature rule} means a low discrepancy sequence, loosely speaking a ``well-spread" set of nodes, along with uniform weights $w_{i} \equiv 1/N$. If $\F$ is a Korobov space of order $s \geq 1$, the Halton sequence of nodes \citep{Hal64} leads to $\mathcal{E}(\mu_1; \bm{x}, (1/N))^2$ in $\mathcal{O}(\log(N)^{2d} N^{-2})$ and higher-order digital nets converge faster as $\mathcal{O}(\log(N)^{2sd} N^{-2s})$ \citep{DiPi14}[Theorem 5].

These rates are naturally inherited if the uniform weights are replaced by the respective optimal weights $\hat{\bm{w}}$, as observed by \cite{BOGOS2019}. In particular, \cite{BOGOS2019} emphasize that the bound for higher-order digital nets attains the optimal rate in this RKHS.
For optimal kernel quadrature based on Halton sequences, this inheritance argument does not explain the fast $\mathcal{O}(\log(N)^{2sd} N^{-2s})$ rates observed empirically by \cite{Oett17}.

Beside the hypercube, optimal kernel quadrature has been considered on the hypersphere equipped with the uniform measure \citep{EhGrCh19}, or on $\mathbb{R}^{d}$ equipped with the Gaussian measure \citep{KaSa19}. In these works, the design construction is adhoc for the space $\X$ and
$g$ is usually assumed to be constant. Another approach is offered by optimisation algorithms that we review in Section~\ref{sec:sequential_algos}. Before that, we clarify the subtle difference between optimal kernel quadrature and kernel interpolation.

\subsection{Kernel interpolation beyond embeddings}\label{sec:review_kernel_interpolation_beyond_mu}

Besides the approximation of the embeddings $\mu_{g}$ discussed in Section~\ref{sec:review_optimal_kernel_quadrature}, theoretical guarantees for the kernel interpolation of a general $\mu\in\F$ are sought \emph{per se}. The Shannon reconstruction formula for bandlimited signals \citep{Sha48} is implicitly an interpolation by the sinc kernel.
The RKHS approach for sampling in signal processing was introduced in \citep{Yao67} for the Hilbert space of bandlimited signals;
see also \citep{NaWa91} for generalizations.
Remarkably, in those RKHSs, every $\mu \in \F$ is an embedding $\mu_{g}$ for some $g \in \Ltwo$: $k$ is a projection kernel of infinite rank. In general, for a trace-class kernel, the subspace spanned by the embeddings $\mu_{g}$ is strictly included in $\F$. More precisely, every $\mu_{g}$ satisfies
\begin{equation}
\|\bm{\Sigma}^{-1/2} \mu_{g}\|_{\F} = \|\bm{\Sigma}^{1/2} g\|_{\F} = \|g\|_{\Ltwo} < +\infty.
\end{equation}
This condition is more restrictive than what is required for a generic $\mu$ to belong to $\F$, i.e., $\|\mu\|_{\F}< +\infty$, so that kernel interpolation is more general than optimal kernel quadrature. The proposed approach will permit to deal with any $\mu\in\F$.

Scattered data approximation \citep{Wend04} is another field where quantitative error bounds for kernel interpolation on $\X\subset\mathbb{R}^d$ are investigated; see \citep{ScWe06} for a modern review. In a few words, these bounds typically depend on quantities such as the \emph{fill-in distance} $\varphi(\bm{x}) = \sup_{y \in \X} \min_{i \in [N]}\|y-x_{i}\|_{2}$,
 so that the interpolation error converges to zero  as $N\rightarrow \infty$ if $\varphi(\bm{x})$ goes to zero.
 Any node set can be considered, as long as $\varphi(\bm{x})$ is small.
Using these techniques, \cite{OaGi16} proposed another application of kernel interpolation: the construction of functional control variates in Monte Carlo integration.
Finally, note that the application of these techniques is restricted to compact domains: the fill-in distance is infinite if $\X$ is not compact, even for ``well-spread" sets of nodes.




\subsection{Optimization algorithms}\label{sec:sequential_algos}

Optimization approaches offer a variety of algorithms for the design of the interpolation nodes.
\cite{DeM03} and \cite{DeScWe05} proposed greedily maximizing the so-called   \emph{power function}
\begin{equation}
p(x;\bm{x}) = \left[k(x,x) - k_{\bm{x}}(x)^{\Tran} \bm{K}(\bm{x})^{-1} k_{\bm{x}}(x)\right]^{1/2},
\label{e:power}
\end{equation}
where $k_{\bm{x}}(x) = (k(x,x_{i}))_{i \in [N]}$.
This algorithm leads to an interpolation error that goes to zero with $N$ for a kernel of class $\mathcal{C}^{2}$ \citep{DeScWe05}. Later, \cite{SaHa17} proved better convergence rates for smoother kernels. Again, these results assume that the domain $\X$ is compact.
Other greedy algorithms were proposed in the context of Bayesian quadrature (BQ) such as Sequential BQ \citep{HuDu12}, or Frank-Wolfe BQ \citep{BrOaGiOs15}.
These algorithms sequentially minimize $\mathcal{E}(\mu_{g}; \bm{x})$, for a fixed $g \in \Ltwo$. The nodes are thus adapted to one particular $\mu_{g}$ by construction.
%
In general, each step of these greedy algorithms requires to solve a non-convex problem with many local minima \citep{Oett17}[Chapter 5]. In practice, costly approximations must be employed such as local search in a random grid \citep{LaLiBa15}.
An alternative approach, that is very related to our contribution and has raised a lot of recent interest, is to observe that the squared power function \eqref{e:power} can be upper bounded by the inverse of $\Det \bm{K}(\bm{x})$ \citep{Sch05,Tan19}.
 Designs that maximize $\Det \bm{K}(\bm{x})$ are called \emph{Fekete points}; see e.g. \citep{BoMa02,BoDe11}.
\cite{Tan19} proposed to approximate $\Det \bm{K}(\bm{x})$ using the Mercer decomposition of $k$, followed by a rounding of the solution of a $D$-experimental design problem, yet without a theoretical analysis of the interpolation error.
\cite{KaSaTa19} proved that for the uni-dimensional Gaussian kernel, the approximate objective function of \citep{Tan19} is actually convex. Moreover, \cite{KaSaTa19} analyze their interpolation error; see also Section~\ref{sec:main_theorems_2}.
Finally, we emphasize that these algorithms require the knowledge of a Mercer-type decomposition of $k$ so that they cannot be implemented for any kernel; moreover, the approximate objective function may be non-convex in general.

%


\subsection{Random designs}

In this section, we survey random node designs with uniform-in-$g$ approximation guarantees for the embeddings $\mu_{g}$ in the RKHS norm.
\cite{Bac17} studied the quadrature resulting from sampling i.i.d. nodes $(x_j)$ from some proposal distribution $q$. He proved that when the proposal is chosen to be
\begin{equation}
	q_{\lambda}^*(x) \propto \sum\limits_{m \in \Ns} \frac{\sigma_{m}}{\sigma_{m}+\lambda} e_{m}(x)^{2},
\label{e:proposalBach}
\end{equation}
with $\lambda >0$,
and the number of points $N$ satisfies $N \geq 5 d_\lambda \log(16 d_\lambda / \delta)$ with $d_\lambda = \Tr \bm{\Sigma}(\bm{\Sigma} + \lambda \bm{I})^{-1}$, then with probability larger than $1-\delta$,
\begin{equation}\label{eq:Bach_bound}
\sup\limits_{\|g\|_{\mathrm{d}\omega} \leq 1} \inf\limits_{ \|\bm{w}\|^{2}\leq \frac{4}{N}} \Big\| \mu_{g} - \sum\limits_{j \in [N]} \frac{w_{j}}{q_\lambda(x_{j})^{1/2}} k(x_{j},.)\Big\|_{\mathcal{F}}^{2} \leq 4\lambda .
\end{equation}

The bound in \eqref{eq:Bach_bound} gives a control on the approximation error of $\mu_{g}$ by the subspace spanned by the $k(x_{j},.)$, and this control is uniform over $g$ in the unit ball of $\mathbb{L}_{2}(\mathrm{d}\omega)$.
Note that for a fixed value of $\lambda$, the upper bound in \eqref{eq:Bach_bound} guarantees that the approximation error is smaller than $4\lambda$. It does not however guarantee that the error goes to zero as $N$ increases since it appears that $\lambda$ should decrease as $N$ increases.
This coupling of $N$ and $\lambda$ combined with the condition $N \geq d_{\lambda} \log d_{\lambda}$ makes it intricate to derive a convergence rate from \eqref{eq:Bach_bound}.
 Moreover, the optimal density $q_{\lambda}^*$ is only implicitly available in general through the limit in \eqref{e:proposalBach}, which makes sampling and pointwise evaluation difficult in practice.

\cite{BeBaCh19} proposed a related kernel-based quadrature, but using nodes sampled from a repulsive joint distribution called a \emph{projection determinantal point process} (DPP); see \citep{HoKrPeVi06} and our Section~\ref{sec:VS_DPP}. In particular, the repulsion is characterized by the first eigenfunctions $(e_{n})_{n \in [N]}$ of the integration operator $\bm{\Sigma}$. The weights $\hat{\bm{w}}$ are chosen again by minimizing the residual error \eqref{e:residual},
which gives the uniform bound
\begin{equation}\label{eq:proj_DPP_bound}
\EX \sup\limits_{\|g\|_{\mathrm{d}\omega} \leq 1} \mathcal{E}(\mu_{g}; \bm{x})^{2} \leq 2(N^{2} r_{N} + o(N^{2} r_{N})) ,
\end{equation}
where $r_{N} = \sum_{m \geq N+1} \sigma_{m}$. This result can be improved by further restricting $g$ to be an eigenfunction of $\bm{\Sigma}$, leading to
\begin{equation}\label{eq:proj_DPP_bound_eigen}
\EX \sup\limits_{g \in \{e_{n};~ n\geq 1\}} \mathcal{E}(\mu_{g}; \bm{x})^{2} \leq 2(N r_{N} + o(N r_{N})).
\end{equation}

Now for smooth kernels, such as the Gaussian kernel or the Sobolev kernel with a large regularity parameter, the upper bounds in \eqref{eq:proj_DPP_bound} and \eqref{eq:proj_DPP_bound_eigen} do converge to $0$ as $N$ goes to $+\infty$. Furthermore, sampling from the recommended projection DPP can be implemented easily, although it still requires the knowledge of the Mercer decomposition of $k$, unlike the method that we introduce here in Section~\ref{sec:VS_DPP}.


Since the bounds in \eqref{eq:proj_DPP_bound} and \eqref{eq:proj_DPP_bound_eigen} are uniform-in-$g$, they also concern interpolation. One downside of the analysis in \citep{BeBaCh19} is that these upper bounds are rather pessimistic: experimental results suggest faster rates in $\mathcal{O}(\sigma_{N})$. If one could prove these rates, then kernel quadrature or interpolation using DPPs would reach known lower bounds, which we now quickly survey.


\subsection{Lower bounds}\label{sec:lower_bounds}
When investigating upper bounds for kernel interpolation errors, it is useful to remember existing lower bounds, so as to evaluate the tightness of one's results. In particular, $N$-widths theory \citep{Pin12} implies lower bounds for kernel interpolation errors, which once again show the importance of the spectrum of $\bm{\Sigma}$.

 The $N$-width of $\mathcal{S} = \{ \mu_{g} = \bm{\bm{\Sigma}}g, \: \|g\|_{\Ltwo} \leq 1\}$ with respect to the couple $(\Ltwo, \F)$ \citep[Chapter 1.7]{Pin12} is defined as the square root of
\begin{equation}
	d_{N}(\mathcal{S})^{2}  = \inf\limits_{\substack{Y \subset \F\\ \mathrm{dim}Y = N}}\, \sup\limits_{\|g\|_{\mathrm{d}\omega} \leq 1}\, \inf\limits_{y \in Y}\, \|\bm{\Sigma} g -y \|_{\F}^{2}.  \\
\end{equation}
In interpolation, we do use a subspace $Y \subset \F$ spanned by $N$ independent functions $k(x_{i},.)$, so that
\begin{equation}
\sup\limits_{\|g\|_{\mathrm{d}\omega} \leq 1} \mathcal{E}(\bm{\Sigma} g;\bm{x})^{2} \geq d_{N}(\mathcal{S})^{2}.
\end{equation}
Applying \citep[Theorem 2.2, Chapter 4]{Pin12} to the adjoint of the embedding operator $I_{\F}$ \cite{StSc12}[Lemma 2.2], it comes
%
$d_{N}(\mathcal{S})^{2}~=~\sigma_{N+1}$.
One may object that some QMC sequences seem to breach this lower bound. For example, in the Korobov space $(d = 2, s \geq 1)$, $\sigma_{N+1} = \mathcal{O}(\log(N)^{2s} N^{-2s})$ \citep{Bac17}, while the interpolation of $\mu_{g}$ with $g = 1$ using a Fibonacci lattice leads to an error in $\mathcal{O}(\log (N) N^{-2s}) = o(\sigma_{N+1})$ \citep{BiTeYu12}[Theorem~4]. But this is the rate for one particular $\mu_{g}$, and it cannot be achieved uniformly in $g$.


\section{Volume sampling and DPPs}\label{sec:VS_DPP}
In this section, we introduce a repulsive distribution that we call \emph{continuous volume sampling} (VS) and compare it to projection determinantal point processes (DPPs; \citep{HoKrPeVi06}). Both continuous VS and projection DPPs are parametrized using a reference measure $\mathrm{d} \omega$ and a repulsion kernel $\KDPP : \X \times \X \rightarrow \mathbb{R}_{+}$.

\subsection{Continuous volume sampling}
\begin{definition}[Continuous volume sampling]\label{def:VS}
Let $N \in \mathbb{N}^{*}$ and $\bm{x} = \{x_{1}, \dots, x_{N}\}\subset \X$. We say that $\bm{x}$ follows the volume sampling distribution
 if $(x_{1}, \dots ,x_{N})$ is a random variable of $\X ^{N}$, the law of which is absolutely continuous with respect to $\otimes_{i \in [N]} \mathrm{d}\omega$, and the density writes
\begin{equation}
  \label{e:CVS_density}
f_{\VS}(x_{1}, \dots ,x_{N}) \propto \Det \bm{K}(\bm{x}).
\end{equation}
\end{definition}

Two remarks are in order. First, under Assumption~\ref{hyp:integrable_diagonal}, the density $f_{\VS}$ in  \eqref{e:CVS_density} indeed integrates to 1. Indeed, Hadamard's inequality yields

\begin{align}
\int_{\X ^{N}} \Det \bm{K}(\bm{x}) \MuTen & \leq \int_{\X ^{N}} \prod\limits_{i \in [N]} k(x_{i},x_{i}) \otimes\mathrm{d}\omega(x_{i}) \nonumber \\
& = \left(\int_{\X}k(x,x) \mathrm{d}\omega(x) \right)^{N}  < +\infty.
\end{align}

Second, the determinant in \eqref{e:CVS_density} is invariant to permutations, so that continuous volume sampling can indeed be seen as defining a random set $\bm{x} = \{x_1,\dots,x_N\}$.

In the following, we denote, for any symmetric and continuous kernel $\tilde{k}$ satisfying Assumption~\ref{hyp:integrable_diagonal},
\begin{equation}
Z_{N}(\tilde{k}) := \int_{\X^{N}} \Det \bm{\tilde{K}}(\bm{x}) \otimes\mathrm{d}\omega(x_{i}).
\end{equation}


\subsection{Continuous volume sampling as a mixture of DPPs}


Definition~\ref{def:VS} could be mistaken with the definition of a determinantal point process (DPP; \citealp{Mac75}). However, the cardinal of a DPP sample is a sum of Bernoulli random variables \citep{HoKrPeVi06}, while volume sampling is supported on subsets of $\X$ with cardinality exactly equal to $N$. This property is convenient for approximation tasks where the number of nodes $N$ is fixed. While it is not a DPP, volume sampling is actually a mixture of DPPs.

\begin{proposition} \label{prop:VS_decomposition}
For $U \subset \mathbb{N}^{*}$ define the projection kernel
\begin{equation}
\mathfrak{K}_{U}(x,y) = \sum\limits_{u \in U} e_{u}(x)e_{u}(y).
\end{equation}
For $N \in \mathbb{N}^{*}$, we have
\begin{equation}
  \label{e:mixture}
f_{\VS}(x_{1}, \dots, x_{N}) \propto \sum\limits_{ U \in \UN} \prod\limits_{u \in U} \sigma_{u} \Det (\mathfrak{K}_{U}(x_{i},x_{j}))_{(i,j)},
\end{equation}
and the normalization constant is equal to
\begin{equation}\label{eq:normalization_constant_VS}
Z_{N}(k) = \mathrm N! \sum\limits_{ U \in \mathcal{U}_{N}} \prod\limits_{u \in U} \sigma_{u}.
\end{equation}
\end{proposition}
The proof of this proposition is given in Appendix~\ref{sec:proof_VS_decomposition}. Observe that
for every $U \subset \UN$,
\begin{equation}
(x_{1}, \dots , x_{N}) \mapsto \frac{1}{N!} \Det (\mathfrak{K}_{U}(x_{i},x_{j}))_{(i,j) \in [N]\times [N] },
\end{equation}
defines a well-normalized probability distribution on $\X^{N}$, called the {\em projection DPP} associated to the marginal kernel $\mathfrak{K}_{U}$  \citep{HoKrPeVi06}. Among all DPPs, only projection DPPs have a deterministic cardinality, equal to the rank of $\mathfrak{K}_{U}$ \citep{HoKrPeVi06}. Interestingly, the largest weight in the mixture \eqref{e:mixture} corresponds to the projection DPP of marginal kernel $\mathfrak{K}_{[N]}$ proposed in \citep{BeBaCh19} for kernel quadrature. The following lemma gives an upper bound on this weight $\delta_N$ using the eigenvalues of $\bm{\Sigma}$.
\begin{lemma}\label{lemma:projection_DPP_weight}
For $N \in \Ns$, define
\begin{equation}
\delta_{N} = \prod\limits_{n \in [N]} \sigma_{n} \bigg/\sum\limits_{ U \in \: \UN} \prod\limits_{u \in U} \sigma_{u}.
\end{equation}
Then for all $N \in \Ns$, $\displaystyle \delta_{N} \leq \sigma_{N} / r_{N}$.
\end{lemma}
In particular, if the spectrum of $k$ decreases polynomially, then $\delta_{N} = \mathcal{O}(1/N)$, so that as $N$ grows, volume sampling becomes more different from the projection DPP of \cite{BeBaCh19}. In contrast, if the spectrum decays exponentially, then $\delta_{N} = \mathcal{O}(1)$.

\subsection{Sampling algorithms}

A projection DPP can be sampled exactly as long as one can evaluate the corresponding projection kernel $\KDPP$ \citep{HoKrPeVi06}. For kernel quadrature \citep{BeBaCh19}, evaluating $\KDPP$ requires the knowledge of the Mercer decomposition of the RKHS kernel $k$. The algorithm of \cite{HoKrPeVi06} implements the chain rule for projection DPPs, and each conditional is sampled using rejection sampling; see \cite{GaBaVa19} for recommendations on proposals. This suggests using the mixture in Proposition~\ref{prop:VS_decomposition} to sample from the volume sampling distribution.
Again, such an algorithm requires explicit knowledge of the Mercer decomposition of the kernel or at least a decomposition onto an orthonormal basis of $\F$ as in \citep{KaSaTa19}. This is a strong requirement that is undesirable in practice.

The fact that the joint pdf \eqref{e:CVS_density} only requires evaluating $k$ pointwise suggests that volume sampling is \emph{fully kernelized}, in the sense that a sampling algorithm should be able to bypass the need for a kernel decomposition, thus making the method very widely applicable.
One could proceed by rejection sampling. Yet the acceptance ratio would likely scale poorly with $N$.  A workaround would be to use an MCMC sampler similar to what was proposed in \citep{ReGh19}. We leave investigating the efficiency of such an MCMC approach to volume sampling to future work.


\section{Main results}\label{sec:main_results}
In this section, we give a theoretical analysis of kernel interpolation on nodes that follow the continuous volume sampling distribution. We state our main result in Section~\ref{sec:main_theorems}, an uniform-in-$g$ upper bound of $\EX_{\VS} \|\mu_{g} - \Pi_{\mathcal{T}(\bm{x})} \mu_{g}\|_{\F}^{2}$. We give an upper bound for a general $\mu \in \F$ in Section~\ref{sec:main_theorems_2}.


\subsection{The interpolation error for embeddings $\mu_{g}$}\label{sec:main_theorems}
%
%

The main theorem of this article decomposes the expected error for an embedding $\mu_g$ in terms of the expected errors $\epsilon_m$ for eigenfunctions of the kernel.
\begin{theorem}\label{thm:main_result_1}
Let $\displaystyle g = \sum\limits_{m\in \Ns} g_m e_m$ satisfy $\| g\|_{\mathrm{d}\omega} \leq 1$.  Then under Assumption~\ref{hyp:integrable_diagonal},
\begin{equation}\label{eq:main_result_EX_VS_err_mu}
\EX_{\VS} \|\mu_{g} - \Pi_{\mathcal{T}(\bm{x})} \mu_{g}\|_{\F}^{2} = \sum\limits_{m \in \mathbb{N}^{*}} g_{m}^{2} \epsilon_{m},
\end{equation}
where $\epsilon_{m} = \sigma_{m} \left(\sum\limits_{ U \in \: \UN} \prod\limits_{u \in U} \sigma_{u} \right)^{-1}  \sum\limits_{  U \in \: \UNm} \prod\limits_{u \in U} \sigma_{u}$.
In particular, the sequence $(\epsilon_m)_{m \in \Ns}$ is non-increasing and
\begin{equation}\label{eq:upper_bound_sup_epsilon}
\sup_{\| g\|_{\mathrm{d}\omega} \leq 1} \EX_{\VS} \|\mu_{g} - \Pi_{\mathcal{T}(\bm{x})} \mu_{g}\|_{\F}^{2} \leq \sup\limits_{m \in \mathbb{N}^{*}} \epsilon_{m} = \epsilon_1.
\end{equation}
Moreover,
\begin{equation}\label{eq:ineq_r_N}
\epsilon_{1} \leq \sigma_{N} \left(1+ \beta_{N}\right),
\end{equation}
where $\displaystyle \beta_{N} = \min_{M \in [2:N]} \left[(N-M+1)\sigma_N\right]^{-1} \sum_{m \geq M} \sigma_m$.
\end{theorem}

In other words, under continuous volume sampling, $\epsilon_{1}$ is a uniform upper bound on the expected squared interpolation error of \emph{any} embedding $\mu_{g}$ such that $\|g\|_{\mathrm{d}\omega} \leq 1$. We shall see in Section~\ref{sec:decomposition_error} that $\epsilon_m = \EX_{\VS} \|\mu_{e_{m}} - \Pi_{\mathcal{T}(\bm{x})} \mu_{e_{m}}\|_{\F}^{2}$.

Now, for $N_{0} \in \Ns$, a simple counting argument yields, for $m \geq N_{0}$, $\epsilon_{m} \leq \sigma_{N_{0}}$. Actually, for $m \geq N_{0}$, $\|\mu_{e_{m}}\|_{\F}^{2} \leq \sigma_{N_{0}}$, independently of the nodes.


%
Inequality~\eqref{eq:ineq_r_N} is less trivial and makes continuous volume sampling distribution worth of interest: the upper bound goes to $0$ as $N \rightarrow +\infty$, below the initial error $\sigma_{N_0}$.
 Moreover, the convergence rate is $\mathcal{O}(\sigma_{N})$, matching the lower bound of Section~\ref{sec:lower_bounds} if the sequence $(\beta_{N})_{N \in \Ns}$ is bounded. In the following proposition, we prove that it is the case as soon as the spectrum decreases polynomially (e.g., Sobolev spaces of finite smoothness) or exponentially (e.g., the Gaussian kernel).

\begin{proposition}\label{prop:constant_bound}
If $\sigma_{m} = m^{-2s}$ with $s >1/2$ then
\begin{equation}
\forall N \in \Ns, \: \beta_{N} \leq \left(1+\frac{1}{2s-1}\right)\left(1+\frac{1}{2s-1}\right)^{2s-1}.
\end{equation}
If $\sigma_{m} = \alpha^{m}$, with $\alpha \in [0,1[$, then
\begin{equation}
\forall N \in \Ns, \: \beta_{N} \leq \frac{\alpha}{1-\alpha}.
\end{equation}
\end{proposition}
In both cases, the proof uses the fact that
\begin{equation}
\beta_{N} \leq [(N-M_{N}+1)\sigma_N]^{-1} \sum_{m \geq M_{N}} \sigma_m,
\end{equation}
for a well designed sequence $M_{N}$. For example, if $\sigma_{m}~=~m^{-2s}$, we take $M_{N} = \lceil{N/c \rceil}$ with $c >1$; if $\sigma_{m} = \alpha^{m}$ we take $M_N = N$. We give a detailed proof in the appendices.

 For a general kernel, if an asymptotic equivalent of $\sigma_{N}$ is known \citep{Wid63,Wid64}, it should be possible to give an explicit construction of $M_N$. Indeed,
\begin{equation}
 \beta_{N} \leq \frac{\sigma_{M_{N}}}{\sigma_{N}} + [{(N-M_{N}+1)\sigma_{N}}]^{-1} \sum\limits_{m \geq N+1}\sigma_{m},
\end{equation}
and $M_{N}$ should be chosen to control both terms in the RHS.
%
Figure~\ref{fig:the_only_one_for_the_moment} illustrates the upper bound of Theorem~\ref{thm:main_result_1} and the constant of Proposition~\ref{prop:constant_bound} in case of the periodic Sobolev space of order $s=3$. We observe that $\EX_{\VS} \mathcal{E}(\mu_{e_m};\bm{x})^{2}$ respects the upper bound: it starts from the initial error level $\sigma_m$ and decreases according to the upper bound for $N \geq m$.



\subsection{The interpolation error of any element of $\F$}\label{sec:main_theorems_2}
Theorem~\ref{thm:main_result_1} dealt with the interpolation of an embedding $\mu_{g}$ of some function $g\in\Ltwo$. We now give a bound on the interpolation error for any $\mu\in\F$. We need the following assumption, which is relatively weak; see Proposition~\ref{prop:constant_bound} and the discussion that follows.

\begin{assumption}\label{hyp:beta_N_bounded}
	There exists $B >0$ such that $\beta_{N} \leq B$.
\end{assumption}

\begin{theorem}\label{thm:slow_rates}
Let $\mu \in \F$.
Assume that $\|\bm{\Sigma}^{-r}\mu \|_{\F} < +\infty$ for some $r \in [0,1/2]$. Then, under Assumption~\ref{hyp:beta_N_bounded},
\begin{equation}
\EX_{\VS} \mathcal{E}(\mu;\bm{x})^{2} \leq (2+B) \sigma_{N}^{2r} \|\bm{\Sigma}^{-r} \mu\|_{\F}^{2} = \mathcal{O}(\sigma_{N}^{2r}).
\end{equation}
%
\end{theorem}
In other words, the expected interpolation error depends on the smoothness parameter $r$. For $r =1/2$, we exactly recover the rate of Theorem~\ref{thm:main_result_1}. In contrast, for $r<1/2$, the rate $\mathcal{O}(\sigma_{N}^{2r})$ is slower. For $r=0$, our bound is constant with $N$. Note that assuming more smoothness $(r>1/2)$ does not seem to improve the rate $\mathcal{O}(\sigma_{N})$.

Let us comment on this bound in two classical cases. First, consider the uni-dimensional Sobolev space of order $s$. Assumption~\ref{hyp:beta_N_bounded} is satisfied by Proposition~\ref{prop:constant_bound} and the squared error scales as $\mathcal{O}(N^{-4sr})$. Moreover, for this family of RKHSs, $\|\bm{\Sigma}^{-r}.\|_{\F}$ can be seen as the norm in the Sobolev space of order $(2r+1)s$, and we recover a result in \citep{ScWe06}[Theorem 7.8] for quasi-uniform designs. By using the norm in the RKHS $\F$ of rougher functions, we upper bound the interpolation error of $\mu$ belonging to the smoother RKHS $\bm{\Sigma}^{r} \F$. Second, we emphasize again that our result is agnostic to the choice of the kernel, as long as Assumption~\ref{hyp:beta_N_bounded} holds. In particular, Theorem~\ref{thm:slow_rates} applies to the Gaussian kernel: the rate is slower $\mathcal{O}(\sigma_{N}^{2r})$ yet still exponential. Finally, recall that for $f \in \F$
\begin{equation}
|f(x)|^{2} = |\langle f, k(x,.)\rangle_{\F}|^{2} \leq\|f\|_{\F}^{2} k(x,x),
\end{equation}
so that, bounds on the RKHS norm imply bounds on the uniform norm if the kernel $k$ is bounded. Therefore, for $r \in [0,1/2]$, our result improves on the rate $\mathcal{O}(N^{2}\sigma_{N}^{2r})$ of approximate Fekete points \citep{KaSaTa19}.


\subsection{Asymptotic unbiasedness of kernel quadrature}\label{sec:unbiased_property}
As explained in Section~\ref{sec:review_optimal_kernel_quadrature}, kernel interpolation is widely used for the design of quadratures.
In that setting, one more advantage of continuous volume sampling is the consistency of its estimator. This is the purpose of the following result.
 \begin{theorem}\label{thm:EX_VS_integration_error}
Let $f \in \F$, and $g \in \Ltwo$. Then
\begin{align}
\mathcal{B}_{N}(f,g) & \triangleq \EX_{\VS} \left( \int_{\X}fg\,\mathrm{d}\omega -  \sum\limits_{i \in N} \hat{w}_{i}f(x_{i}) \right). \nonumber\\
 & = \sum\limits_{n \in \Ns} \langle f,e_{n} \rangle_{\mathrm{d}\omega} \langle g,e_{n} \rangle_{\mathrm{d}\omega}\left(1- \EX_{\VS}\tau_{n}^{\F}(\bm{x}) \right) . 
\end{align}
Moreover,
$\mathcal{B}_{N}(f,g) \rightarrow 0$ as $N \rightarrow +\infty$.
\end{theorem}
Compared to the upper bound on the integration error given by \eqref{eq:upper_bound_integration_error}, the bias term in Theorem~\ref{thm:EX_VS_integration_error} takes into account the interaction between $f$ and $g$. For example, if for all $n \in \Ns$, $\langle f,e_{n} \rangle_{\mathrm{d}\omega} \langle g,e_{n} \rangle_{\mathrm{d}\omega} = 0$, the quadrature is unbiased for every $N$.
Theorem~\ref{thm:EX_VS_integration_error} is a generalization of a known property of regression estimators based on volume sampling in the discrete setting \citep{BeTe90,DeWa17,DeWaHs18,DeWaHs19}.

\begin{figure}[]
    \centering
\includegraphics[width= 0.8\textwidth]{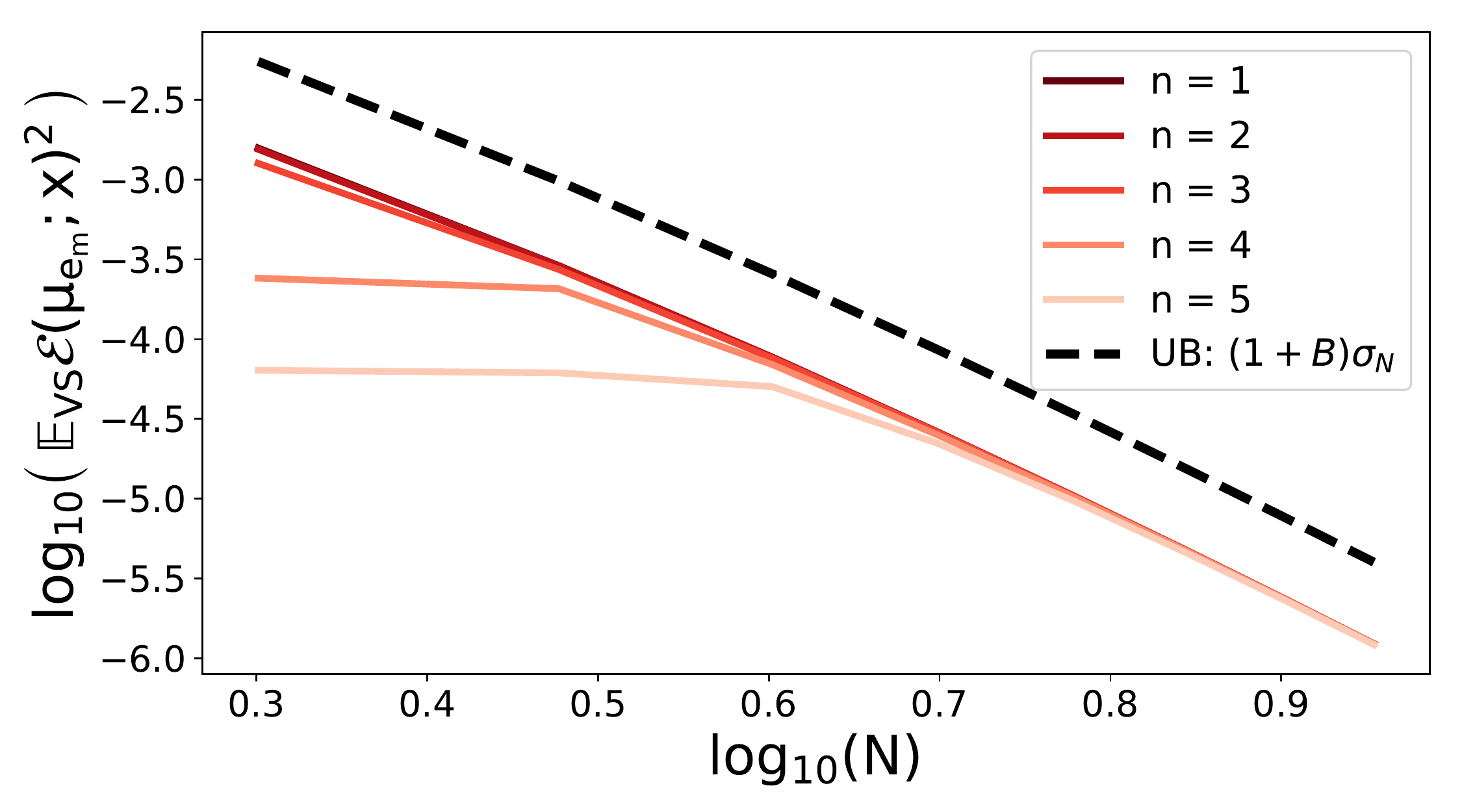}\\
\caption{\label{fig:the_only_one_for_the_moment} The value of $\EX_{\VS} \mathcal{E}(\mu_{e_{m}};\bm{x})^{2}$ for $m \in \{1,2,3,4,5\}$ for the periodic Sobolev space ($s =3,d=1$) compared to the theoretical upper bound (UB) of Theorem~\ref{thm:main_result_1}.}
\end{figure}

\section{Sketch of the proofs}\label{sec:steps_proof}
 The proof of Theorem~\ref{thm:main_result_1} decomposes into three steps. First, in Section~\ref{sec:decomposition_error}, we write $\mathcal{E}(\mu_{g};\bm{x})^{2}$ as a function of the square of the interpolation errors $\mathcal{E}(\mu_{e_{m}};\bm{x})^{2}$ of the embeddings
$\mu_{e_{m}}$. Then, in Section~\ref{sec:closed_formulas}, we give closed formulas for $\EX_{\VS} \mathcal{E}(\mu_{e_{m}};\bm{x})^{2}$ in terms of the eigenvalues of $\bm{\Sigma}$. Finally, the inequality \eqref{eq:ineq_r_N} is proved using an upper bound on the ratio of symmetric polynomials \citep{GuSi12}. The details are given in Appendix~\ref{app:proof_ineq_r_N}. Finally, The proofs of Theorem~\ref{thm:slow_rates} and Theorem~\ref{thm:EX_VS_integration_error} are straightforward consequences of Theorem~\ref{thm:main_result_1}. The details are given in Appendix~\ref{app:proof_slow_rates} and Appendix~\ref{app:proof_bias}.

\subsection{Decomposing the interpolation error}
\label{sec:decomposition_error}
Let $\bm{x} \in \mathcal{X}^{N}$ such that $\Det \bm{K}(\bm{x}) > 0$. For $m_{1}, m_{2} \in \mathbb{N}^{*}$, let the \emph{cross-leverage score} between $m_{1}$ and $m_{2}$ associated to $\bm{x}$ be
\begin{equation}\label{eq:cross_lvs_def}
\tau_{m_{1},m_{2}}^{\F}(\bm{x}) = e_{m_{1}}^{\F}(\bm{x})^{\Tran} \bm{K}(\bm{x})^{-1} e_{m_{2}}^{\F}(\bm{x}).
\end{equation}
When $m_1=m_2=m$, we speak of the $m$-th leverage score\footnote{Our definition is consistent with the leverage scores used in matrix subsampling  \citep{DrMaMu06}. Loosely speaking, $\tau_{m}^{\F}(\bm{x})$ is the leverage score of the $m$-th column of the semi-infinite matrix  $(e_{n}^{\F}(x_{i}))_{(i,n) \in [N] \times \Ns}$.} associated to $\bm{x}$, and simply write $\tau_{m}^{\F}(\bm{x})$.
By Lemma~\ref{lemma:lvs_identities},
 the $m$-th leverage score is related to the interpolation error of the $m$-th eigenfunction $e_{m}^{\F}$. Indeed,
\begin{equation}
\|e_{m}^{\F} - \Pi_{\mathcal{T}(\bm{x})} e_{m}^{\F}\|_{\F}^{2} = 1- \tau_{m}^{\F}(\bm{x}) \in [0,1].
\end{equation}
Similarly, for the cross-leverage score,
\begin{equation}
\langle \Pi_{\mathcal{T}(\bm{x})} e_{m_{1}}^{\F}, \Pi_{\mathcal{T}(\bm{x})} e_{m_{2}}^{\F} \rangle_{\F} = \tau_{m_{1},m_{2}}^{\F}(\bm{x}) \in [-1,1].
\end{equation}

For $g \in \Ltwo$, the interpolation error of the embedding $\mu_{g}$ can be expressed using the (cross-)leverage scores.
\begin{lemma}\label{lemma:error_decomposition}
If $\Det \bm{K}(\bm{x}) > 0$, then,
\begin{equation}
\label{e:error_decomposition_in_lemma}
\mathcal{E}(\mu_{g};\bm{x})^{2}  = \sum\limits_{m \in \mathbb{N}^{*}}  g_{m}^{2} \sigma_{n}\bigg(1- \tau_{m}^{\F}(\bm{x})\bigg) - \sum\limits_{m_{1}\neq m_{2} \in \mathbb{N}^{*}}  g_{m_{1}}g_{m_{2}} \sqrt{\sigma_{m_{1}}} \sqrt{\sigma_{m_{2}}} \tau_{m_{1},m_{2}}^{\F}(\bm{x}).
\end{equation}
\end{lemma}
In particular, with probability one, a design sampled from the continuous volume sampling distribution in Definition~\ref{def:VS} satisfies \eqref{e:error_decomposition_in_lemma}. Furthermore, we shall see that the expected value of the (cross-) leverage scores has a simple expression.

\subsection{Explicit formulas for expected leverage scores}
\label{sec:closed_formulas}
Proposition~\ref{prop:EX_VS_lvs} expresses expected leverage scores in terms of the spectrum of the integration operator.
\begin{proposition}\label{prop:EX_VS_lvs}
For $m~\in~\mathbb{N}^{*}$,
\begin{equation}\label{eq:lvs_formula}
\EX_{\VS} \tau_{m}^{\F}(\bm{x})  = \frac{1}{\sum\limits_{U \in \: \UN} \prod\limits_{u \in U}\sigma_{u}}  \sum\limits_{\substack{U \in \: \UN \\ m \in U}} \prod\limits_{u \in U}\sigma_{u}.
\end{equation}
Moreover, for $m_{1},m_{2} \in \mathbb{N}^{*}$ such that $m_{1} \neq m_{2}$, we have
\begin{equation}\label{eq:cross_lvs_zero}
\EX_{\VS} \tau_{m_{1},m_{2}}^{\F}(\bm{x}) = 0.
\end{equation}
\end{proposition}

In Appendix~\ref{app:proof_main_result_1}, we combine Lemma~\ref{lemma:error_decomposition} with Proposition~\ref{prop:EX_VS_lvs}. This concludes the proof of Theorem~\ref{thm:main_result_1} by Beppo Levi's monotone convergence theorem.


It remains to prove Proposition~\ref{prop:EX_VS_lvs}. Again, we proceed in two steps. First, our Proposition~\ref{thm:EX_lvs_phi} yields a characterization of $ \EX_{\VS}\tau_{m}^{\F}(\bm{x})$ and $\EX_{\VS}\tau_{m_{1},m_{2}}^{\F}(\bm{x})$ in terms of the spectrum of three perturbed versions of the integration operator $\bm{\Sigma}$. Second, we give explicit forms of these spectra in Proposition~\ref{prop:perturbed_kernels_spectrum} below. The idea is to express $\EX_{\VS} \tau_{m}(\bm{x})^{\F}$ as the normalization constant \eqref{eq:normalization_constant_VS} of a perturbation of the kernel $k$.
The same goes for $\EX_{\VS} \tau_{m_{1},m_{2}}^{\F}(\bm{x})$.

Let $t\in\mathbb{R}_{+}$ and $\bm{\Sigma}_{t}$, $\bm{\Sigma}_{t}^{+}$ and $\bm{\Sigma}_{t}^{-}$ be the integration operators\footnote{We drop from the notation the dependencies on $m,m_{1}$ and $m_{2}$ for simplicity.} on $\Ltwo$, respectively associated with the kernels
\begin{equation}\label{eq:k_t}
k_{t}(x,y) = k(x,y) + t e_{m}^{\F}(x) e_{m}^{\F}(y),
\end{equation}
\begin{equation}\label{eq:k_t_plus}
k_{t}^{+}(x,y) = k(x,y)
+ t \left( e_{m_{1}}^{\F}(x) + e_{m_{2}}^{\F}(x) \right) \left( e_{m_{1}}^{\F}(y) + e_{m_{2}}^{\F}(y) \right),
\end{equation}
\begin{equation}\label{eq:k_t_minus}
k_{t}^{-}(x,y) = k(x,y) + t \left( e_{m_{1}}^{\F}(x) - e_{m_{2}}^{\F}(x) \right) \left( e_{m_{1}}^{\F}(y) - e_{m_{2}}^{\F}(y) \right).
\end{equation}
By Assumption~\ref{hyp:integrable_diagonal}, and by the fact that $(e_{m})_{m \in \Ns}$ is an orthonormal basis of $\Ltwo$, all three kernels also have integrable diagonals (see Assumption~\ref{hyp:integrable_diagonal}).
In particular, they define RKHSs that can be embedded in $\Ltwo$. Moreover, recalling the definition \eqref{eq:normalization_constant_VS} of the normalization constant $Z_{N}$ of volume sampling, the following quantities are finite
\begin{align}
\phi_{m}(t) = & Z_{N}(k_{t}) , \quad\phi_{m_{1},m_{2}}^{+}(t) = Z_{N}(k_{t}^{+}),&\text{ and }\quad \phi_{m_{1},m_{2}}^{-}(t) = Z_{N}(k_{t}^{-}).
\end{align}
Remember that by Proposition~\ref{prop:VS_decomposition},
\begin{equation}\label{eq:Z_kt}
\phi_{m}(t) = N! \sum\limits_{U \in \: \UN} \prod\limits_{u \in \: U} \tilde{\sigma}_{u}(t),
\end{equation}
where $\displaystyle \{\tilde{\sigma}_{u}(t), \: u \in \Ns\}$ is the set of eigenvalues\footnote{For a given value of $t$, the eigenvalues $\tilde{\sigma}_{u}(t)$ are not necessarily decreasing in $u$. We give explicit formulas for these eigenvalues in Proposition~\ref{prop:perturbed_kernels_spectrum}, and the order satisfied for $t=0$ is not necessarily preserved for $t>0$. This does not change anything to the argument since these eigenvalues only appear in quantities such as $\phi_{m}(t)$ which are invariant under permutation of the eigenvalues.} of $\bm{\Sigma}_{t}$. Similar identities are valid for $\phi_{m_{1},m_{2}}^{+}(t)$ and $\phi_{m_{1},m_{2}}^{-}(t)$ with the eigenvalues of $\bm{\Sigma}_{t}^{+}$ and $\bm{\Sigma}_{t}^{-}$ respectively.

\begin{proposition}\label{thm:EX_lvs_phi}
The functions $\phi_m$, $\phi_{m_1,m_2}^{+}$ and $\phi_{m_1,m_2}^{-}$ are right differentiable in zero. Furthermore,
\begin{equation}\label{eq:EX_lvs_n_delta_phi_n}
\EX_{\VS} \tau_{m}^{\F}(\bm{x})  = \frac{1}{Z_{N}(k)}  \frac{\partial \phi_{m}}{ \partial t }\bigg|_{t = 0^{+}},
\end{equation}
and
\begin{equation}\label{eq:EX_VS_cross_lvs_formula}
\EX_{\VS} \tau_{m_{1},m_{2}}^{\F}(\bm{x}) = \frac{1}{4Z_{N}(k)} \bigg(\frac{\partial \phi_{m_{1},m_{2}}^{+}}{ \partial t } - \frac{\partial \phi_{m_{1},m_{2}}^{-}}{ \partial t }\bigg)\bigg|_{t = 0^{+}}.
\end{equation}
\end{proposition}
 The details of the proof are postponed to Appendix~\ref{app:proof_EX_lvs_phi}. We complete this proposition with a description of the spectrum of the operators $\bm{\Sigma}_{t}$, $\bm{\Sigma}_{t}^{+}$ and $\bm{\Sigma}_{t}^{-}$ using the spectrum of $\bm{\Sigma}$.
\begin{proposition}\label{prop:perturbed_kernels_spectrum}
The eigenvalues of $\bm{\Sigma}_{t}$ write
\begin{equation}
\tilde{\sigma}_{u}(t) =  \left\{
    \begin{array}{ll}
        \sigma_{u} & \mbox{if} \: u \neq m, \\
        (1 + t)\sigma_{u} & \mbox{if}\: u = m.
    \end{array}
\right.
\end{equation}
Moreover, the eigenvalues of $\bm{\Sigma}_{t}^{+}$ and $\bm{\Sigma}_{t}^{-}$ satisfy
\begin{equation}\label{eq:equality_positive_negative_spectrum}
\{\tilde{\sigma}_{u}^{+}(t), \: u \in \Ns \} = \{\tilde{\sigma}_{u}^{-}(t), \: u \in \Ns \}.
\end{equation}
\end{proposition}
The proof is based on the observation that the perturbations in \eqref{eq:k_t}, \eqref{eq:k_t_plus}, and \eqref{eq:k_t_minus} only affect a principal subspace of dimension $1$ or $2$; see Appendix~\ref{app:proof_perturbed_kernels_spectrum}.

Combining the characterization of $\EX_{\VS} \tau_{m}^{\F}(\bm{x})$ and $\EX_{\VS} \tau_{m_{1},m_{2}}^{\F}(\bm{x})$ given in Proposition~\ref{thm:EX_lvs_phi}, and Proposition~\ref{prop:perturbed_kernels_spectrum}, we prove Proposition~\ref{prop:EX_VS_lvs}; see details in Appendix~\ref{sec:proof_EX_VS_lvs}.

\section{Conclusion}\label{sec:discussion}

We deal with interpolation in RKHSs using random nodes and optimal weights. This problem is intimately related to kernel quadrature, though interpolation is more general.
We introduced continuous volume sampling (VS), a repulsive point process that is a mixture of DPPs, although not a DPP itself.
VS comes with a set of advantages. First, interpretable bounds on the interpolation error can be derived under minimalistic assumptions. Our bounds are close to optimal since they share the same decay rate as known lower bounds. Moreover, we provide explicit evaluations of the constants appearing in our bounds for some particular RKHS (Sobolev, Gaussian). Second,
VS provides a fully kernelized approach, which may permit sampling without knowing the Mercer decomposition of the kernel \citep{ReGh19} unlike previous work on random designs. Investigating efficient samplers is deferred to future work.


Volume sampling was originally introduced on a finite domain \citep{DRVW06}, and has been used for matrix subsampling for linear regression and low-rank approximations \citep{DeWa17,BeBaCh18}. Similarly to \citep{BeBaCh19}, our work is another direct connection between subsampling columns of a matrix and the continuous problem of interpolating an element of an RKHS.

Compared to \citep{KaSaTa19}, our analysis yields a sharper upper bound while circumventing the analysis of the Lebesgue constant.
Yet, it would be interesting to analyze this constant under continuous volume sampling to get an idea of the stability of the algorithm. Another potential extension of this work would be the analysis of interpolation under regularization as in \citep{Bac17}.

\subsection*{Acknowledgments}
AB acknowledges support from ANR grant \textsc{BoB} (ANR-16-CE23-0003), Région Hauts-de-France, and Centrale Lille. RB acknowledges support from ERC grant \textsc{Blackjack} (ERC-2019-STG-851866).

\newpage

\bibliography{bibliography}
\newpage
\appendix

\section{Technical results borrowed from other papers}
Thoughout our proof, we use a few technical results from the literature, which we gather here for ease of reference.
\subsection{The Jacobi identity}
The following proposition is a direct consequence of the rank one-update for determinants, see e.g. \citep[Theorem 3.11]{MaSpSr15}.
\begin{proposition}[Jacobi identity]\label{thm:jacobi_identity}
Let $\bm{A}, \bm{B} \in \mathbb{R}^{N \times N}$. If $\Det\bm{A}\neq 0$, then
\begin{equation}
\partial_{t} \Det (\bm{A}+t\bm{B})|_{t = 0} = \Det(\bm{A}) \Tr(\bm{A}^{-1}\bm{B}).
\end{equation}
In particular, we have
\begin{equation}
\partial_{t} \Det (\bm{A}+t\bm{B})|_{t = 0^{+}} = \Det(\bm{A}) \Tr(\bm{A}^{-1}\bm{B}).
\end{equation}
\end{proposition}

\subsection{The Markov brothers' inequality}
The following proposition is known as the Markov brother's inequality, see e.g. \citep{Sha04}.
\begin{proposition}[Markov brothers]\label{thm:Markov_Brothers_inequality}
Let $P$ be a polynomial of degree smaller than $N$. Then
\begin{equation}
\max_{\tau \in [-1,1]} |P'(\tau)| \leq N^{2}\max_{\tau \in [-1,1]}|P(\tau)|.
\end{equation}
\end{proposition}
We shall actually use a straightforward corollary.
\begin{corollary}\label{cor:Markov_Brothers_inequality_2}
Let $P$ be a polynomial of degree smaller than $N$. Then
\begin{equation}
\max_{\tau \in [0,1]} |P'(\tau)| \leq 2N^{2}\max_{\tau \in [0,1]}|P(\tau)|.
\end{equation}
\end{corollary}
\begin{proof}
Define the polynomial $Q(x) = P((x+1)/2)$, so that
\begin{equation}
Q'(x) = \frac{1}{2}P'((x+1)/2), \quad x\in[-1,1].
\end{equation}
In particular,
$$\max_{\tau \in [0,1]} |P(\tau)| = \max_{\tau \in [-1,1]} |Q(\tau)|,$$
so that
%
\begin{align}
\max_{\tau \in [0,1]} |P'(\tau)| & = \max_{\tau \in [-1,1]} 2|Q'(\tau)| \leq 2N^2\max_{\tau \in [-1,1]} |Q(\tau)| \leq 2N^2\max_{\tau \in [0,1]} |P(\tau)|.
\end{align}
\end{proof}
\subsection{An inequality on the ratio of symmetric polynomials}
\label{s:inequality_symmetric_polys}
Recall that, for $d \in \Ns$, $\mathbb{R}^{d}$ is naturally embedded in the set of sequences $\mathbb{R}^{\Ns}$.

Now, let $M \in \Ns$, and let $\bm{\lambda} \in \mathbb{R}^{\Ns}_{+}$ such that $\sum_{m \in \Ns} \lambda_{m} < +\infty$. By MacLaurin’s inequality \footnote{The inequality is usually stated for $\bm{\lambda} \in \mathbb{R}^{d}_{+}$ for some $d \in \Ns$. Taking limits immediatedly yields \eqref{e:macLaurin}.}, see e.g. \citep[Chapter 12]{Ste04},
\begin{equation}
  \label{e:macLaurin}
\forall M \in \Ns, \: \sum\limits_{U \in \mathcal{U}_{M}} \prod\limits_{u \in U} \lambda_{u} \leq \left(\sum\limits_{m \in \Ns} \lambda_{m}\right)^{M} < +\infty.
\end{equation}
In the following, we denote by $p_{M}(\bm{\lambda})$ the elementary symmetric polynomial of order $M$ on the sequence $\bm{\lambda}$,
\begin{equation}
p_{M}(\bm{\lambda}) = \sum\limits_{U \in \mathcal{U}_{M}} \prod\limits_{u \in U} \lambda_{u}.
\end{equation}
In particular, the following identity relates $p_M$ and $p_{M+1}$.
\begin{equation}
\forall M \geq 2, \: \forall m \in \Ns, \: p_{M}(\bm{\lambda}) = \lambda_{m} p_{M-1}(\bm{\lambda}^{\overline{\{m\}}}) + p_{M}(\bm{\lambda}^{\overline{\{m\}}}),
\end{equation}
where we denote, for $S \subset \Ns$, $\bm{\lambda}^{\overline{S}} = (\lambda^{\overline{S}}_{m})_{m \in \Ns} = (\lambda_{m} \ind_{m\notin S})_{m \in \Ns}$. Proposition~\ref{thm:schur_convex_volume_sampling} further relates two consecutive elementary polynomials.

\begin{proposition}[Theorem 3.1 of \citealp{GuSi12}]\label{thm:schur_convex_volume_sampling}
Let $M \in \Ns$ and $L \geq M+1$. Let $\bm{\lambda} \in \mathbb{R}_{+}^{L}$ be a nonincreasing sequence
\begin{equation}
\lambda_{1} \geq \lambda_{2} \geq \dots \geq \lambda_{L}.
\end{equation}
Assume that $\lambda_{L} >0$, then
\begin{equation}
  \label{e:GuSi12}
 \forall M' \leq M, \quad \frac{p_{M+1}(\bm{\lambda})}{p_{M}(\bm{\lambda})} \leq \frac{\sum_{m \geq M'+1} \lambda_{m}}{M+1-M'}.
\end{equation}
\end{proposition}

We will actually use an immediate consequence of Proposition~\ref{thm:schur_convex_volume_sampling}.
\begin{corollary}\label{cor:schur_convex_volume_sampling}
Let $M \in \Ns$ and $\bm{\lambda} \in \mathbb{R}_{+}^{\Ns}$ be a nonincreasing sequence such that $\sum \lambda_{m} < +\infty$ and $\lambda_{m} >0$ for all $m \in \Ns$. Then \eqref{e:GuSi12} still holds.
\end{corollary}

\begin{proof}
Define, for $L \in \Ns$,
\begin{equation}
\bm{\lambda}_{L} = (\lambda_{\ell})_{\ell \in [L]} \in \mathbb{R}_{+}^{L}.
\end{equation}
%
By Proposition~\ref{thm:schur_convex_volume_sampling},
\begin{align}\label{eq:ineq_lambda_L}
\forall M' \leq M, \:\: \forall L \geq M+1, \: \frac{p_{M+1}(\bm{\lambda}_{L})}{p_{M}(\bm{\lambda}_{L})} & \leq \frac{1}{M+1-M'}\sum\limits_{m = M'+1}^{L} \lambda_{m}\\
& \leq \frac{1}{M+1-M'}\sum\limits_{m = M'+1}^{+\infty} \lambda_{m}.
\end{align}
Letting $L\rightarrow \infty$ allows us to conclude.
\end{proof}

For the last result, recall the definition of the (cross-)leverage scores $\tau_{m_{1},m_{2}}$ in \eqref{eq:cross_lvs_def}. We slightly adapt a result by \cite{BeBaCh19} .
\begin{lemma}\label{lemma:lvs_identities}
Let $\bm{x} \in \X^{N}$ satisfy $\Det \bm{K}(\bm{x}) > 0$. For $m, m_{1}, m_{2} \in \Ns$ such that $m_{1} \neq m_{2}$,
\begin{equation}\label{eq:lvs_identity}
\tau_{m}^{\F}(\bm{x}) = \|\Pi_{\mathcal{T}(\bm{x})}e_{m}^{\F}\|_{\F}^{2} = e_{m}^{\F}(\bm{x})^{\Tran} \bm{K}(\bm{x})^{-1} e_{m}^{\F}(\bm{x}),
\end{equation}
and
\begin{equation}\label{eq:cross_lvs_identity}
\tau_{m_{1},m_{2}}^{\F}(\bm{x}) = \langle \Pi_{\mathcal{T}(\bm{x})}e_{m_{1}}^{\F}, \Pi_{\mathcal{T}(\bm{x})}e_{m_{2}}^{\F} \rangle_{\F} = e_{m_1}^{\F}(\bm{x})^{\Tran} \bm{K}(\bm{x})^{-1} e_{m_2}^{\F}(\bm{x}).
\end{equation}
In particular,
\begin{equation}
  \label{eq:lvs_bound}
  \tau_{m}^{\F}(\bm{x})\text{  and  }|\tau_{m_{1},m_{2}}^{\F}(\bm{x})|\text{  are in  }[0,1].
\end{equation}
\end{lemma}
\begin{proof}
The proof of \eqref{eq:lvs_identity} is given in \citep{BeBaCh19}[Lemma 4 of Appendix D]. The proof of \eqref{eq:cross_lvs_identity} is straightforward following the same lines. $\Pi_{\mathcal{T}(\bm{x})}$ is an orthogonal projection with respect to $\langle .,. \rangle_{\F}$ and
\begin{equation}
\|e_{m}^{\F}\|_{\F} = \|e_{m_{1}}^{\F}\|_{\F} = \|e_{m_{2}}^{\F}\|_{\F} = 1,
\end{equation}
so that \eqref{eq:lvs_bound} follows from the Cauchy-Schwarz inequality.
\end{proof}

\section{Proofs}
 \subsection{Proof of Proposition~\ref{prop:VS_decomposition}}\label{sec:proof_VS_decomposition}

Proposition~\ref{prop:VS_decomposition} states that continuous volume sampling is a mixture of projection determinantal point processes. We adapt a result in \citep[Chapter 5]{KuTa12} for finite volume sampling to the infinite-dimensional case. The idea of the proof is to apply the Cauchy-Binet identity to a sequence of kernels of finite rank that approximate $k$.

First, recall from Section~\ref{sec:introduction} the Mercer decomposition of $k$,
\begin{equation}\label{eq:Mercer_decomposition_1}
k(x,y) = \lim_{M\rightarrow \infty}\sum\limits_{m\in [M]} \sigma_{m} e_{m}(x)e_{m}(y) = \lim_{M\rightarrow \infty} k_M(x,y), \quad\forall x,y\in\mathcal{X}.
\end{equation}
where kernel $k_M$ has rank $M$.

Now, let $\bm{x} = (x_{1}, \dots, x_{N}) \in \mathcal{X}^{N}$, and define
$\bm{K}_{M}(\bm{x}) = (k_{M}(x_{i},x_{j}))_{i,j \in [N]}$. By continuity of the determinant and by \eqref{eq:Mercer_decomposition_1}, it comes
\begin{equation}\label{eq:truncated_det_limit}
  \lim\limits_{M \rightarrow \infty} \Det \bm{K}_{M}(\bm{x}) = \Det \bm{K}(\bm{x}).
\end{equation}
By construction,
\begin{equation}
  \bm{K}_{M}(\bm{x}) = \bm{F}_{M}(\bm{x})^{\Tran}\bm{\Sigma}_{M}\bm{F}_{M}(\bm{x}),
\end{equation}
where $\bm{F}_{M}(\bm{x}) = (e_{m}(x_{i}))_{(m,i) \in [M]\times[N]}$ and $\Sigma_{M}$ is a diagonal matrix containing the first $M$ eigenvalues $(\sigma_{m})_{m\in[M]}$ on its diagonal. The Cauchy-Binet identity yields
\begin{equation}\label{eq:Cauchy_Binet_truncated_kernel}
  \Det \bm{K}_{M}(\bm{x})  = \sum\limits_{\substack{U \subset [M]\\ |U| = N}} \Det^{2} (e_{u}(x_{i}))_{(u,i)\in U \times [N]} \prod\limits_{u \in U} \sigma_{u}.
\end{equation}
Let now $\lambda_u= \prod_{u \in U} \sigma_{u}$ and $\bm{E}_{U}(\bm{x}) = (e_{u}(x_{i}))_{(u,i)\in U \times [N]}$,
we combine \eqref{eq:truncated_det_limit} and \eqref{eq:Cauchy_Binet_truncated_kernel} to obtain
\begin{align}
\Det \bm{K}(\bm{x}) & = \lim\limits_{M \rightarrow \infty} \sum_{\substack{U \subset [M]\\ |U| = N}} \lambda_u \Det^{2} (e_{u}(x_{i}))_{(u,i)\in U \times [N]}  \\
&= \sum_{U\in \UN} \lambda_u \Det^{2} (e_{u}(x_{i}))_{(u,i)\in U \times [N]}\\
&= \sum_{U\in \UN} \lambda_u \Det \left( \bm{E}_{U}(\bm{x})^{\Tran} \bm{E}_{U}(\bm{x}) \right)\\
&= \sum_{U\in \UN} \lambda_u \Det ({\KDPP}_{U}(x_{i},x_{j}))_{i,j \in [N]},
\end{align}
where ${\KDPP}_{U}(x,y) \triangleq \sum_{u\in U} e_u(x)e_u(y)$. Since ${\KDPP}_{U}$ is a projection kernel, writing the determinant as a sum over permutations easily yields, for all $U\in\UN$,
\begin{equation}
\int_{\X^{N}} \Det({\KDPP}_{U}(x_{i},x_{j}))_{i,j \in [N]} \otimes_{i \in [N]} \mathrm{d}\omega(x_{i}) = N!,
\end{equation}
see e.g. Lemma 21 in \citep{HoKrPeVi06}. Finally, the monotone convergence theorem allows us to conclude
\begin{equation}
\int_{\X^{N}} \Det \bm{K}(\bm{x}) \otimes_{i \in [N]} \mathrm{d}\omega(x_{i}) = N! \sum\limits_{\substack{U \subset \Ns\\ |U| = N}} \prod\limits_{u \in U } \sigma_{u}.
\end{equation}

\subsection{Proof of Lemma~\ref{lemma:projection_DPP_weight}}
Lemma~\ref{lemma:projection_DPP_weight} gives an upper bound on the biggest weight $\delta_N$ in the mixture of Proposition~\ref{prop:VS_decomposition}. The proof is straightforward, as
\begin{align}\label{eq:r_N_delta_N_inequality}
r_{N} \prod\limits_{\ell \in [N]} \sigma_{\ell} & = \sigma_{N} \sum\limits_{m \geq N+1} \sigma_{m} \prod\limits_{\ell \in [N-1]} \sigma_{\ell} \nonumber\\
& \leq \sigma_{N} \sum\limits_{\substack{U \subset \Ns\\ |U| = N}} \prod\limits_{u \in U} \sigma_{u}.
\end{align}
This immediately yields $\delta_{N} \leq \sigma_{N}/r_{N}$.

\subsection{Proof of Lemma~\ref{lemma:error_decomposition}}\label{app:proof_error_decomposition}
Lemma~\ref{lemma:error_decomposition} decomposes the interpolation error in terms of (cross-)leverage scores. Let $g \in \Ltwo$ satisfy $\|g\|_{\mathrm{d}\omega} \leq 1$. Since $\Pi_{\mathcal{T}(\bm{x})}$ is an orthogonal projection with respect to $\langle .,. \rangle_{\F}$, we have
\begin{equation}\label{eq:reconstruction_error}
\|\mu_{g} - \Pi_{\mathcal{T}(\bm{x})} \mu_{g}\|_{\F}^{2} = \|\mu_{g}\|_{\F}^{2} - \| \Pi_{\mathcal{T}(\bm{x})} \mu_{g}\|_{\F}^2
\end{equation}
Now, $\mu_{g} = \bm{\Sigma} g = \sum\limits_{m \in \mathbb{N}^{*}} \sqrt{\sigma_{m}} g_{m} e_{m}^{\F}$, so that \eqref{eq:reconstruction_error} becomes
\begin{align}
\|\mu_{g} - \Pi_{\mathcal{T}(\bm{x})} \mu_{g}\|_{\F}^{2} & = \sum\limits_{m \in \mathbb{N}^{*}} \sigma_{m} g_{m}^{2} - \left\Vert \sum_{m \in \Ns} \Pi_{\mathcal{T}(\bm{x})}  \sqrt{\sigma_{m}} g_{m} e_{m}^{\F}\right\Vert_{\F}^2 \nonumber\\
& = \sum\limits_{m \in \mathbb{N}^{*}} \sigma_{m} g_{m}^{2} - \sum\limits_{m_{1}, m_2} g_{m_{1}} g_{m_{2}} \sqrt{\sigma_{m_{1}}} \sqrt{\sigma_{m_{2}}} \langle  \Pi_{\mathcal{T}(\bm{x})}   e_{m_{1}}^{\F}, \Pi_{\mathcal{T}(\bm{x})} e_{m_{2}}^{\F} \rangle_{\F}. \label{eq:reconstruction_norm_mu_onb}
\end{align}
Lemma~\ref{lemma:lvs_identities} allows us to recognize leverage scores in \eqref{eq:reconstruction_norm_mu_onb}. Taking out of the second sum in \eqref{eq:reconstruction_norm_mu_onb} the terms for which $m_1=m_2$ to put them in the first sum concludes the proof of Lemma~\ref{lemma:error_decomposition}.

\subsection{Proof of Theorem~\ref{thm:main_result_1}}\label{app:proof_main_result_1}
The proof of \eqref{eq:main_result_EX_VS_err_mu} relies on the identity
\begin{equation}\label{eq:main_result_EX_VS_err_mu_2}
\EX_{\VS} \|\mu_{g} - \Pi_{\mathcal{T}(\bm{x})} \mu_{g}\|_{\F}^{2} = \sum\limits_{m \in \mathbb{N}^{*}} g_{m}^{2} \epsilon_{m},
\end{equation}
and the fact that $(\epsilon_{m})$ is a non-increasing sequence. We prove these two results in turn, after what we prove \eqref{eq:ineq_r_N}.

\subsubsection{Proof of \eqref{eq:main_result_EX_VS_err_mu_2}}
Let $\bm{x} \in \X^{N}$ such that $\Det \bm{K}(\bm{x}) >0 $. Lemma~\ref{lemma:error_decomposition} yields
\begin{equation}
\|\mu_{g} - \Pi_{\mathcal{T}(\bm{x})} \mu_{g}\|_{\F}^{2}  = \sum\limits_{m \in \mathbb{N}^{*}} g_{m}^{2} \sigma_{m}\bigg(1- \tau_{m}^{\F}(\bm{x})\bigg) - \sum\limits_{\substack{m_{1},m_{2} \in \mathbb{N}^{*} \\ m_{1} \neq m_{2}}}  g_{m_{1}}g_{m_{2}} \sqrt{\sigma_{m_{1}}} \sqrt{\sigma_{m_{2}}} \tau_{m_{1},m_{2}}^{\F}(\bm{x}).
\end{equation}
First, we prove that
\begin{equation}\label{eq:interchange_EX_VS_lvs}
\EX_{\VS} \sum\limits_{m \in \Ns} g_{m}^{2} \sigma_{m}\bigg(1- \tau_{m}^{\F}(\bm{x})\bigg) = \sum\limits_{m \in \Ns} g_{m}^{2} \sigma_{m}\bigg(1- \EX_{\VS}\tau_{m}^{\F}(\bm{x})\bigg).
\end{equation}
By Lemma~\ref{lemma:lvs_identities},
\begin{equation}
\forall m \in \Ns, \:\: g_{m}^{2} \sigma_{m}\bigg(1- \tau_{m}^{\F}(\bm{x})\bigg)\geq 0, \label{e:summable_term}
\end{equation}
so that \eqref{eq:interchange_EX_VS_lvs} follows from the Beppo Levi's monotone convergence theorem.

Second, it remains to prove that
\begin{equation}\label{eq:interchange_EX_VS_cross_lvs}
\EX_{\VS} \sum\limits_{\substack{m_{1},m_{2} \in \Ns \\ m_{1} \neq m_{2}}}  g_{m_{1}}g_{m_{2}} \sqrt{\sigma_{m_{1}}} \sqrt{\sigma_{m_{2}}} \tau_{m_{1},m_{2}}^{\F}(\bm{x})
= 0.
\end{equation}
Again, Lemma~\ref{lemma:lvs_identities} guarantees that, for $m_{1},m_{2} \in \Ns$ such that $m_{1} \neq m_{2}$,
\begin{equation}
| g_{m_{1}}g_{m_{2}} \sqrt{\sigma_{m_{1}}} \sqrt{\sigma_{m_{2}}} \tau_{m_{1},m_{2}}^{\F}(\bm{x})| \leq | g_{m_{1}}g_{m_{2}} | \sqrt{\sigma_{m_{1}}} \sqrt{\sigma_{m_{2}}}.
\end{equation}
Since
\begin{align}
\sum\limits_{m_{1} \neq m_{2} \in \Ns} |g_{m_{1}}g_{m_{2}}| \sqrt{\sigma_{m_{1}}} \sqrt{\sigma_{m_{2}}}
& \leq  \left(\sum\limits_{m \in \Ns} |g_{m}| \sqrt{\sigma_{m}} \right)^{2} \nonumber \\
& \leq \sum\limits_{m \in \Ns}g_{m}^{2} \sum\limits_{m \in \Ns} \sigma_{m} \nonumber \\
& < + \infty,
\end{align}
the dominated convergence theorem yields
$$
\EX_{\VS} \sum\limits_{\substack{m_{1},m_{2} \in \Ns \\ m_{1} \neq m_{2}}}  g_{m_{1}}g_{m_{2}} \sqrt{\sigma_{m_{1}}} \sqrt{\sigma_{m_{2}}} \tau_{m_{1},m_{2}}^{\F}(\bm{x})
= \sum\limits_{\substack{m_{1},m_{2} \in \mathbb{N}^{*} \\ m_{1} \neq m_{2}}}  g_{m_{1}}g_{m_{2}} \sqrt{\sigma_{m_{1}}} \sqrt{\sigma_{m_{2}}} \EX_{\VS} \tau_{m_{1},m_{2}}^{\F}(\bm{x}),
$$
but this is equal to zero by Proposition~\ref{prop:EX_VS_lvs}.

\subsubsection{Proof that $(\epsilon_{m})$ is nonincreasing}
\label{s:epsilon_non_increasing}
Let $m \in \Ns$. By definition,
\begin{align}
\epsilon_{m} & = \sigma_{m} \frac{\sum_{U \in \mathcal{U}_{N}^{m}} \prod_{u \in U} \sigma_{u}}{\sum_{U \in \mathcal{U}_{N}} \prod_{u \in U} \sigma_{u}} = \sigma_{m} \frac{p_{N}(\bm{\sigma}^{\overline{\{m\}}})}{p_{N}(\bm{\sigma})},
\end{align}
where we use a notation introduced in Section~\ref{s:inequality_symmetric_polys}. This leads to
\begin{align}
  \epsilon_m & = \sigma_{m} \frac{ \sigma_{m+1}p_{N-1}(\bm{\sigma}^{\overline{\{m,m+1 \}}}) + p_{N}(\bm{\sigma}^{\overline{\{m,m+1 \}}})}{p_{N}(\bm{\sigma})},
\end{align}
and, similarly,
\begin{equation}
\epsilon_{m+1} = \sigma_{m+1} \frac{ \sigma_{m}p_{N-1}(\bm{\sigma}^{\overline{\{m,m+1 \}}}) + p_{N}(\bm{\sigma}^{\overline{\{m,m+1 \}}})}{p_{N}(\bm{\sigma})}.
\end{equation}
Taking the ratio, it comes
\begin{align}
\frac{\epsilon_{m}}{\epsilon_{m+1}} & = \frac{\sigma_{m} \bigg( \sigma_{m+1}p_{N-1}\left(\bm{\sigma}^{\overline{\{m,m+1 \}}} \right) + p_{N} \left(\bm{\sigma}^{\overline{\{m,m+1 \}}} \right)\bigg)}{\sigma_{m+1} \bigg( \sigma_{m}p_{N-1}\left(\bm{\sigma}^{\overline{\{m,m+1 \}}} \right) + p_{N} \left(\bm{\sigma}^{\overline{\{m,m+1 \}}}\right)\bigg)}\\
& = \frac{1  + \frac{1}{\sigma_{m+1}}\frac{p_{N} \left(\bm{\sigma}^{\overline{\{m,m+1 \}}} \right) }{ p_{N-1} \left(\bm{\sigma}^{\overline{\{m,m+1 \}}} \right)} }{1  + \frac{1}{\sigma_{m} }\frac{p_{N} \left(\bm{\sigma}^{\overline{\{m,m+1 \}}} \right) }{p_{N-1} \left(\bm{\sigma}^{\overline{\{m,m+1 \}}}\right)}} \geq 1,
\end{align}
because $1/\sigma_{m+1} \geq 1/\sigma_{m}$.

\subsubsection{Proof of \eqref{eq:ineq_r_N}}\label{app:proof_ineq_r_N}

We have
$\epsilon_{1} = \epsilon_{N}\epsilon_{1}/\epsilon_{N} \leq \sigma_{N} \epsilon_{1}/\epsilon_{N}$ since a simple counting argument yields $ \epsilon_{N} \leq \sigma_{N}$.
Along the lines of Section~\ref{s:epsilon_non_increasing},
\begin{align}
\frac{\epsilon_{1}}{\epsilon_{N}} & = \frac{1  + \frac{1}{\sigma_{N}}\frac{p_{N} \left(\bm{\sigma}^{\overline{\{1,N \}}} \right) }{ p_{N-1} \left(\bm{\sigma}^{\overline{\{1,N \}}} \right)} }{1  + \frac{1}{\sigma_{1} }\frac{p_{N} \left(\bm{\sigma}^{\overline{\{1,N \}}} \right) }{p_{N-1} \left(\bm{\sigma}^{\overline{\{1,N \}}} \right)}}
\leq 1  + \frac{1}{\sigma_{N}}\frac{p_{N} \left(\bm{\sigma}^{\overline{\{1,N \}}} \right) }{ p_{N-1} \left(\bm{\sigma}^{\overline{\{1,N \}}} \right)}.
\end{align}
Now, $\bm{\sigma}^{\overline{\{1,N \}}}$ is a sequence of positive real numbers and the $\bm{\Sigma}$ is trace-class. Then, by Corollary~\ref{cor:schur_convex_volume_sampling}, for $M \in [N-1]$,
\begin{equation}
\frac{p_{N} \left(\bm{\sigma}^{\overline{\{1,N \}}}  \right) }{ p_{N-1} \left(\bm{\sigma}^{\overline{\{1,N \}}}  \right)} \leq \frac{1 }{N-M} \sum_{ m \geq M} \sigma_{m+2} = \frac{1 }{N+1-(M+1)} \sum_{ m+1 \geq M+1} \sigma_{m+2}.
\end{equation}
Taking $M'= M+1$ concludes the proof of \eqref{eq:ineq_r_N}.

\subsection{Proof of Proposition~\ref{prop:constant_bound}}
\subsubsection{The case of a polynomially-decreasing spectrum}
Assume that $\sigma_{m} = m^{-2s}$ with $s>1/2$.
Let $N \in \Ns$ and $M_{N} = \lceil{N/c \rceil} \in \{2,\dots,N\}$, with $c \in [1,N[$.
We have
\begin{align}
\min_{M \in [2:N]}\frac{\sum_{m \geq M} \sigma_{m+1}}{(N-M+1)\sigma_N} & \leq \frac{\sum_{m \geq M_N} \sigma_{m+1}}{(N-M_N+1)\sigma_N} \\
& \leq \frac{\sum_{m \geq \lceil{N/c\rceil}} \sigma_{m+1}}{(N-\lceil{N/c\rceil}+1)\sigma_N} \\
& \leq \frac{\sum_{m \geq \lceil{N/c\rceil}} \sigma_{m+1}}{(N-N/c+1)\sigma_N}\\
& \leq \frac{\sum_{m \geq \lceil{N/c\rceil}} (m+1)^{-2s}}{(N-N/c+1)N^{-2s}}.\\
\end{align}
Now,
\begin{equation}
\forall m \in \Ns, \:\: (m+1)^{-2s} \leq \int_{m}^{m+1}t^{-2s} \mathrm{d}t = \frac{1}{2s-1}(m^{1-2s}-(m+1)^{1-2s}),
\end{equation}
so that
\begin{equation}
\sum\limits_{m \geq \lceil{N/c\rceil}} (m+1)^{-2s} \leq \frac{1}{2s-1}\lceil{N/c\rceil}^{1-2s}.
\end{equation}
Recall that $2s>1$, so that
\begin{equation}
\frac{1}{2s-1}\lceil{N/c\rceil}^{1-2s} \leq \frac{1}{2s-1}(N/c)^{1-2s},
\end{equation}
and
\begin{align}
\min_{M \in [2:N]}\frac{\sum_{m \geq M} \sigma_{m+1}}{(N-M+1)\sigma_N}
& \leq \frac{1}{2s-1}\frac{(N/c)^{1-2s}}{(N-N/c+1)N^{-2s}}\\
& \leq \frac{c^{2s}}{2s-1}\frac{N}{(cN-N+c)} \label{eq:the_bound_to_be_bounded}.
\end{align}

Note that $c$ is a free parameter that belongs to $[1,N]$ \footnote{The inequality in \eqref{eq:the_bound_to_be_bounded} is valid for $c = N$ by continuity.} that we can optimize in the upper bound: $\frac{c^{2s}}{2s-1}\frac{N}{(cN-N+c)}$. For this purpose, denote
\begin{equation}
\phi_{N}(c) = \frac{c^{2s}}{2s-1}\frac{N}{(cN-N+c)}.
\end{equation}
For every $N \in \Ns$, $\phi_{N}$ is differentiable in $]0,+\infty[$ and
\begin{equation}
\phi_{N}^{'}(c) = \frac{N}{2s-1}\frac{c^{2s-1}}{(cN-N+c)^{2}} \left( (2s-1)(N+1) c - 2s N \right),
\end{equation}
so that $\phi_{N}^{'}$ vanishes in $c_{N}^{*} = \frac{2s}{2s-1} \frac{N}{N+1}$; it is negative in $]0,c_{N}^{*}[$ and positive in $]c_{N}^{*}, +\infty[$. We distinguish three cases:

If $c_{N}^{*}< 1$, $N < 2s-1$ and $\phi_{N}^{'}$ is positive on $[1,N]$ so that $\phi_{N}$ increases in $[1,N]$ and we take $c = 1$ in \eqref{eq:the_bound_to_be_bounded}:
\begin{equation}
\phi_{N}(1) = \frac{N}{2s-1} < 1.
\end{equation}

If $c_{N}^{*} \in [1,N]$, $c_{N}^{*}$ is the unique minimizer of $\phi_{N}$ in $[1,N]$ and we take $c = c_{N}^{*}$ in \eqref{eq:the_bound_to_be_bounded} so that:
\begin{align}
\phi_{N}(c_{N}^{*}) &= \left(\frac{2s}{2s-1}\right)^{2s} \left(\frac{N}{N+1}\right)^{2s}\\
& \leq \left(\frac{2s}{2s-1}\right)^{2s}\\
& \leq \left(1+\frac{1}{2s-1}\right)\left(1+\frac{1}{2s-1}\right)^{2s-1}.
\label{e:useful_for_limits_as_well}
\end{align}

Finally, if $c_{N}^{*} > N$, $N < \frac{1}{2s-1}$, $\phi_{N}$ is decreasing in $[1,N]$ and we take $c = N$ in \eqref{eq:the_bound_to_be_bounded} so that:
\begin{align}
\phi_{N}(N) = \frac{N^{2s-1}}{2s-1} &\leq \frac{1}{2s-1} \left(\frac{1}{2s-1}\right)^{2s-1} \\
&\leq \left(1+\frac{1}{2s-1}\right) \left(1+\frac{1}{2s-1}\right)^{2s-1}. \label{e:useful_for_limits}
\end{align}




In the three cases, $\beta_{N}$ is upper bounded by $\left(1+\frac{1}{2s-1}\right) \left(1+\frac{1}{2s-1}\right)^{2s-1}$. The artificial two-factor form of \eqref{e:useful_for_limits_as_well} and \eqref{e:useful_for_limits} is there to make limits clearer. In particular, the RHS goes to $e$ as $s\rightarrow \infty$.

\subsubsection{The case of an exponentially decreasing spectrum}
Assume that $\sigma_{m} = \alpha^{m}$ with $\alpha \in [0,1[$.
Let $N \in \Ns$, and $M_N = N \in \{ 2,\dots, N\}$.
We have
\begin{align}
\min_{M \in [2:N]}\frac{\sum_{m \geq M} \sigma_{m+1}}{(N-M+1)\sigma_N}
& \leq \frac{\sum_{m \geq M_N} \sigma_{m+1}}{(N-M_N+1)\sigma_N}\\
& \leq \frac{\sum_{m \geq N} \sigma_{m+1}}{\sigma_N}\\
& \leq \frac{\sum_{m \geq N} \alpha^{m+1}}{\alpha^{N}}\\
& \leq \alpha^{N+1}\frac{\sum_{m \geq 0} \alpha^{m}}{\alpha^{N}}\\
& \leq \frac{\alpha}{1-\alpha}.
\end{align}

\subsection{Proof of Proposition~\ref{prop:perturbed_kernels_spectrum}}\label{app:proof_perturbed_kernels_spectrum}

We start with deriving the spectrum\footnote{
All the integration operators in this article, and specifically in this section, are self-adjoint and compact. The spectrum of such operators is the union of $\{0\}$ (the essential spectrum) and the set of eigenvalues \citep{Bre10}[Theorem 6.8] . Yet, the proof of the Mercer decomposition, only involves the set of eigenvalues \citep{StSc12}. For this reason, we use the term ``spectrum" to refer to the set of eigenvalues.} of the trace-class, self-adjoint operator
\begin{equation}
\bm{\Sigma}_{t} = \bm{\Sigma} + t e_{m}^{\F} \otimes e_{m}^{\F},
\end{equation}
where $e_{m}^{\F} \otimes e_{m}^{\F}$ is defined by
\begin{equation}
\forall g \in \Ltwo, \:\: e_{m}^{\F} \otimes e_{m}^{\F} \,g(\cdot) = e_{m}^{\F}(\cdot) \int_{\X}g(y)e_{m}^{\F}(y) \mathrm{d}\omega(y).
\end{equation}
The two operators $\bm{\Sigma}$ and $e_{m}^{\F} \otimes e_{m}^{\F}$ are co-diagonalizable in the basis $(e_{m})_{m \in \Ns}$, thus their linear combination $\bm{\Sigma}_{t}$ diagonalizes in this basis too. In other words, for $u \in \Ns$, $e_{u}$ is an eigenfunction of $\bm{\Sigma}_{t}$ and
\begin{equation}
\bm{\Sigma}_{t} e_{u} = \bm{\Sigma} e_{u} + t e_{m}^{\F} \otimes e_{m}^{\F} (e_{u}) = (\sigma_{u} + t \delta_{u,m} \sigma_{u}) e_{u}.
\end{equation}
Therefore, the set $\{\sigma_{u}(1+t\delta_{u,m}), u \in \Ns\}$ is included in the spectrum of $\bm{\Sigma}_t$. Since $(e_{m})_{m \in \Ns}$ is an orthonormal basis of $\Ltwo$ and correspond to the eigenfunctions of $\bm{\Sigma}_{t}$ associated to the elements of $\{\sigma_{u}(1+t\delta_{u,m}), u \in \Ns\}$, then the spectrum of $\bm{\Sigma}_t$ is exactly the set $\{\sigma_{u}(1+t\delta_{u,m}), u \in \Ns\}$.\footnote{$\bm{\Sigma}_{t}$ is self-adjoint, and has no zero eigenvalue by assumption. Thus, any new eigenfunction that is not in our basis needs to be orthogonal to all basis elements, and is thus zero.}
We now turn to deriving the spectrum of the trace-class, self-adjoint operator $\bm{\Sigma}_{t}^{+}$; the case of $\bm{\Sigma}_{t}^{-}$ follows the same lines and will be omitted for brevity.
We will prove that there exists an orthonormal basis $(f_{m})_{m \in \Ns}$ of $\Ltwo$ such that every $f_m$ is an eigenfunction of $\bm{\Sigma}_{t}^{+}$. If $t=0$, $\bm{\Sigma}_{t}^{+} = \bm{\Sigma}$ and $(e_{m})_{m \in \Ns}$ is already an orthonormal basis of $\Ltwo$. We assume in the following that $t>0$.

Consider the operator $\Delta_{t}^{+}$ defined on $\Ltwo$ by
\begin{equation}
\Delta_{t}^{+}g(\cdot) = t \bigg(e_{m_{1}}^{\F}(\cdot)+e_{m_{2}}^{\F}(\cdot) \bigg) \int_{\X} g(y) \bigg(e_{m_{1}}^{\F}(y)+e_{m_{2}}^{\F}(y) \bigg) \mathrm{d}\omega(y).
\end{equation}
We can write $\bm{\Sigma}_{t}^{+} = \bm{\Sigma} + \Delta_{t}^{+}$, but this time, if $t>0$, $\bm{\Sigma}$ and $\Delta_{t}^{+}$ do not commute. In particular, they are not co-diagonalizable, and a more detailed analysis is necessary. First, by construction of $\Delta_{t}^{+}$,
$$
\Delta_{t}^{+} e_m = 0, \quad m\notin\{m_1,m_2\},
$$
so that for any $m\notin\{m_1,m_2\}$, $\bm{\Sigma}_{t}^{+}$ and $\bm{\Sigma}$ have $e_m$ for eigenfunction, with the same eigenvalue $\sigma_m$. Observe that
\begin{equation}
\Ltwo = \Span (e_{m_1},e_{m_2}) \oplus \Span (e_{m})_{m \notin \{m_1,m_2\}}.
\end{equation}
Therefore, the rest of the proof consists in completing $(e_m)_{m\notin\{m_1,m_2\}}$ into an orthonormal basis of $\Ltwo$, by finding two orthonormal eigenfunctions of $\bm{\Sigma}_{t}^{+}$ in $\Span (e_{m_{1}},e_{m_{2}})$.
Since we assumed in Section~\ref{sec:introduction} that the eigenvalues of $\bm{\Sigma}$ are nonzero, we note that $\Span (e_{m_{1}},e_{m_{2}}) = \Span (e_{m_{1}}^{\F},e_{m_{2}}^{\F})$. Expressing the new eigenfunctions in terms of $e_{m_{1}}^{\F}$ and $e_{m_{2}}^{\F}$ will turn out to be more convenient.

First, note that
\begin{align}
\bm{\Sigma}_{t}^{+}e_{m_{1}}^{\mathcal{F}}(\cdot) & = \bm{\Sigma} e_{m_{1}}^{\mathcal{F}}(\cdot) + t \int_{\mathcal{X}} \left(e_{m_{1}}^{\mathcal{F}}(\cdot) + e_{m_{2}}^{\mathcal{F}}(\cdot) \right)\left( e_{m_{1}}^{\mathcal{F}}(y) + e_{m_{2}}^{\mathcal{F}}(y) \right) e_{m_{1}}^{\mathcal{F}}(y) \mathrm{d}\omega(y) \\
& = \sigma_{m_{1}} e_{m_{1}}^{\mathcal{F}}(\cdot) + t \sigma_{m_{1}} \left(e_{m_{1}}^{\mathcal{F}}(\cdot) + e_{m_{2}}^{\mathcal{F}}(\cdot) \right)\\
& = (1+t) \sigma_{m_{1}}e_{m_{1}}^{\mathcal{F}} + t \sigma_{m_{1}} e_{m_{2}}^{\mathcal{F}}.\label{eq:sigma_t_plus_action}
\end{align}
Similarly,
\begin{align}
\bm{\Sigma}_{t}^{+}e_{m_{2}}^{\mathcal{F}}(.) & = t \sigma_{m_{2}} e_{m_{1}}^{\mathcal{F}} + (1+t) \sigma_{m_{2}}e_{m_{2}}^{\mathcal{F}}.\label{eq:sigma_t_plus_action_2}
\end{align}

Now, let $v = \lambda_{1} e_{m_{1}}^{\mathcal{F}} + \lambda_{2} e_{m_{2}}^{\mathcal{F}}$, so that, by \eqref{eq:sigma_t_plus_action} and \eqref{eq:sigma_t_plus_action_2},
\begin{align}
\bm{\Sigma}_{t}^{+} v & = \lambda_{1} \bigg((1+t) \sigma_{m_{1}} e_{m_{1}}^{\mathcal{F}} + t \sigma_{m_{1}}  e_{m_{2}}^{\mathcal{F}} \bigg) + \lambda_{2} \bigg((1+t) \sigma_{m_{2}} e_{m_{2}}^{\mathcal{F}}  + t \sigma_{m_{2}}  e_{m_{1}}^{\mathcal{F}} \bigg) \nonumber \\
\label{eq:sigma_t_plus_action_v}
& = \bigg( \lambda_{1} (1+t) \sigma_{m_{1}} + \lambda_{2} t \sigma_{m_{2}} \bigg) e_{m_{1}}^{\mathcal{F}} + \bigg( \lambda_{2} (1+t) \sigma_{m_{2}} + \lambda_{1}t\sigma_{m_{1}} \bigg) e_{m_{2}} ^{\mathcal{F}},
\end{align}

Solving for eigenvalues, we look for $\mu \in \mathbb{R}$ such that $\bm{\Sigma}_{t}^{+} v = \mu v$, or equivalently
$$
\bigg\{\begin{array}{ccc}
  (1+t) \sigma_{m_{1}} \lambda_{1} + t \sigma_{m_{2}} \lambda_{2} &=& \mu \lambda_{1}, \\
t \sigma_{m_{1}} \lambda_{1} +  (1+t) \sigma_{m_{2}} \lambda_{2} &=& \mu \lambda_{2}.
\end{array}
$$
This is just saying that $\mu$ should be an eigenvalue of the matrix
\begin{align}\label{eq:positive_matrix}
\left(\begin{array}{cc}
(1+t) \sigma_{m_{1}} & t \sigma_{m_{2}}\\
t \sigma_{m_{1}} & (1+t) \sigma_{m_{2}}\\
\end{array}\right),
\end{align}
which yields two solutions,
\begin{equation}
{\mu}_{1}^{+} = (1+t)\frac{\sigma_{m_{1}}+ \sigma_{m_{2}}}{2} + \frac{1}{2} \sqrt{(1+t)^{2}(\sigma_{m_{1}}-\sigma_{m_{2}})^{2}+4 \sigma_{m_{1}}\sigma_{m_{2}}t^{2}}\,,
\end{equation}
and
\begin{equation}
{\mu}_{2}^{+} = (1+t)\frac{\sigma_{m_{1}}+ \sigma_{m_{2}}}{2} - \frac{1}{2} \sqrt{(1+t)^{2}(\sigma_{m_{1}}-\sigma_{m_{2}})^{2}+4 \sigma_{m_{1}}\sigma_{m_{2}}t^{2}}\,.
\end{equation}
These solutions are distinct since $t>0$, and the corresponding normalized eigenfunctions $v_{1}^{+}$ and $v_{2}^{+}$ are orthogonal with respect to $\langle.,.\rangle_{\mathrm{d}\omega}$ since $\bm{\Sigma}_t^+$ is self-adjoint. Finally, we define the set of eigenfunctions of $\bm{\Sigma}_{t}^{+}$ by
the system $(e_{m})_{m \notin \{m_{1},m_{2}\}} \cup (v_{1}^{+},v_{2}^{+})$ that is an orthonormal basis of $\Ltwo$. Therefore, the spectrum of the compact operator $\bm{\Sigma}_{t}^{+}$ is exactly the set
\begin{equation}
\{ \sigma_{m}, \: m \notin  \{m_{1},m_{2}\}\} \cup \{\mu_{1}^{+}, \mu_{2}^{+} \}.
\end{equation}

Along the same lines, one can show that the eigenvalues of $\bm{\Sigma}_{t}^{-}$ restricted to $\text{Span}(e_{m_{1}}^{\F}, e_{m_{2}}^{\F})$ satisfy
\begin{equation}
 \lambda^{2} - (1+t)(\sigma_{m_{1}}+\sigma_{m_{2}}) \lambda - \sigma_{m_{1}}\sigma_{m_{2}}t^{2} = 0 .
\end{equation}
For $t>0$, this equation again admits two distinct solutions
\begin{equation}
\hat{\mu}_{1}^{-} = (1+t)\frac{\sigma_{m_{1}}+ \sigma_{m_{2}}}{2} + \frac{1}{2} \sqrt{(1+t)^{2}(\sigma_{m_{1}}-\sigma_{m_{2}})^{2}+4 \sigma_{m_{1}}\sigma_{m_{2}}t^{2}},
\end{equation}
and
\begin{equation}
\hat{\mu}_{2}^{-} = (1+t)\frac{\sigma_{m_{1}}+ \sigma_{m_{2}}}{2} - \frac{1}{2} \sqrt{(1+t)^{2}(\sigma_{m_{1}}-\sigma_{m_{2}})^{2}+4 \sigma_{m_{1}}\sigma_{m_{2}}t^{2}}.
\end{equation}
so that the spectrum of $\bm{\Sigma}_{t}^{-}$ is exactly the set
\begin{equation}
\{ \sigma_{m}, \: m \notin  \{m_{1},m_{2}\}\} \cup \{\mu_{1}^{-}, \mu_{2}^{-} \} = \{ \sigma_{m}, \: m \notin  \{m_{1},m_{2}\}\} \cup \{\mu_{1}^{+}, \mu_{2}^{+} \}.
\end{equation}
In other words, the two operators $\bm{\Sigma}_{t}^{+}$ and $\bm{\Sigma}_{t}^{-}$ share the same eigenvalues.


\subsection{Proof of Proposition~\ref{thm:EX_lvs_phi}}\label{app:proof_EX_lvs_phi}
\subsubsection{The expected value of the $m$-th leverage score \label{sec:proof_EX_VS_lvs_n_diff_phi}}
Let $m \in \Ns$. On the one hand, recall that $\tau_{m}^{\F}(\bm{x}) = e_{m}^{\mathcal{F}}(\bm{x})^{\Tran}\bm{K}(\bm{x})^{-1}e_{m}^{\mathcal{F}}(\bm{x})$, so that, by Definition~\ref{def:VS},
\begin{equation}
  \label{e:just_def_of_VS}
\EX_{\VS} \tau_{m}^{\F}(\bm{x}) = \left( N!\sum\limits_{U\in\UN}\prod\limits_{u \in U}\sigma_{u} \right)^{-1} \int_{\mathcal{X}^{N}} e_{m}^{\mathcal{F}}(\bm{x})^{\Tran}\bm{K}(\bm{x})^{-1}e_{m}^{\mathcal{F}}(\bm{x}) \Det \bm{K}(\bm{x}) \otimes_{i \in [N]} \mathrm{d}\omega(x_{i}).
\end{equation}
We have
\begin{align}
\Det \bm{K}(\bm{x})\,e_{m}^{\mathcal{F}}(\bm{x})^{\Tran}\bm{K}(\bm{x})^{-1}e_{m}^{\mathcal{F}}(\bm{x}) & = \Det \bm{K}(\bm{x})\Tr \left( e_{m}^{\mathcal{F}}(\bm{x})^{\Tran}\bm{K}(\bm{x})^{-1}e_{m}^{\mathcal{F}}(\bm{x}) \right) \nonumber \\
& = \Det \bm{K}(\bm{x})\Tr \left(\bm{K}(\bm{x})^{-1}e_{m}^{\mathcal{F}}(\bm{x})e_{m}^{\mathcal{F}}(\bm{x})^{\Tran}\right)\nonumber\\
&= \partial_{t} \Det (\bm{K}(\bm{x})+t e_{m}^{\mathcal{F}}(\bm{x})e_{m}^{\mathcal{F}}(\bm{x})^{\Tran})|_{t = 0^{+}}\,,\label{e:derivative_inside_the_integral}
\end{align}
where the last line follows from the Jacobi identity of Theorem~\ref{thm:jacobi_identity}.

On the other hand, for $t>0$ and with the notation of Section~\ref{sec:closed_formulas}, let $ \bm{K}_{t}(\bm{x}) := (k_{t}(x_{i},x_{j}))_{i,j \in [N]} = \bm{K}(\bm{x})+t e_{m}^{\mathcal{F}}(\bm{x})e_{m}^{\mathcal{F}}(\bm{x})^{\Tran}$.
Since
\begin{align}
\int_{\X} k_{t}(x,x) \Mu
& = \int_{\X} k(x,x) \Mu + t \int_{\X} e_{m}^{\F}(x)^{2} \Mu = \sum\limits_{n \in \Ns} \sigma_{n} + t \sigma_{m} < \infty,
\end{align}
Hadamard's inequality yields the integrability of
$
\psi(.,t): \bm{x} \mapsto \Det \bm{K}_{t}(\bm{x}).
$
Finally, observe that
\begin{equation}
\phi_{m}(t) := Z_N(k_t) =  \int_{\X^{N}} \psi(\bm{x},t) \otimes_{i \in [N]} \mathrm{d}\omega(x_{i}).
\end{equation}

 If we prove that $\phi_{m}$ is right differentiable in zero, and that we can justify the interchange of the derivation and the integration operations, we will have equated the right derivative of $\phi_m$ in zero and \eqref{e:just_def_of_VS} using \eqref{e:derivative_inside_the_integral}; this will achieve proving the first equation in Proposition~\ref{thm:EX_lvs_phi}. To this purpose, we need to prove that $t \mapsto \psi(\bm{x},t)$ is right differentiable at zero, it is locally dominated by an integrable function and its derivative is locally dominated by an integrable function. Now, observe that $t \mapsto \psi(\bm{x},t)$ is a polynomial of degree smaller than $N$, so that it is differentiable, and Corollary~\ref{cor:Markov_Brothers_inequality_2} yields
\begin{equation}
\max_{\tau \in [0,1]} \left|\partial_t \psi(\bm{x},\tau)\right| \leq 2N^{2} \max_{\tau \in [0,1]} \left|\psi(\bm{x},\tau)\right|.
\end{equation}
 In other words, to dominate $\tau \mapsto |\partial_t \psi(\bm{x},\tau)|$ uniformly on $[0,1]$, it is sufficient to dominate $\tau \mapsto| \psi(\bm{x},\tau)|$ uniformly there. Now, let $\tau \in [0,1]$, we have
\begin{align}
\bm{K}_{1}(\bm{x}) - \bm{K}_{\tau}(\bm{x}) &  = \bm{K}(\bm{x}) + e_{m}^{\mathcal{F}}(\bm{x})e_{m}^{\mathcal{F}}(\bm{x})^{\Tran} - \bm{K}(\bm{x}) - \tau e_{m}^{\mathcal{F}}(\bm{x})e_{m}^{\mathcal{F}}(\bm{x})^{\Tran}\\
& = (1-\tau)e_{m}^{\mathcal{F}}(\bm{x})e_{m}^{\mathcal{F}}(\bm{x})^{\Tran} \in \mathcal{S}_{N}^{+}.
\end{align}
Thus
\begin{equation}
0\preceq \bm{K}_{\tau}(\bm{x}) \preceq \bm{K}_{1}(\bm{x})
\end{equation}
in the Loewner order, so that for any $\tau\in [0,1]$,
\begin{equation}
|\psi(\bm{x},\tau)| = \psi(\bm{x},\tau) =\Det \bm{K}_{\tau}(\bm{x}) \leq \Det \bm{K}_{1}(\bm{x}) = \psi(\bm{x},1) .
\end{equation}
We conclude by observing that $\bm{x} \mapsto \psi(\bm{x},1)$ is integrable on $\X^{N}$ by Proposition~\ref{prop:VS_decomposition}, and the fact that
\begin{equation}
\int_{\X} k_{1}(x,x) \mathrm{d}\omega(x) < +\infty.
\end{equation}



\subsubsection{The expected value of cross-leverage scores}
Let $m_{1},m_{2} \in \Ns$ such that $m_{1} \neq m_{2}$. We have
\begin{align}
\tau_{m_{1},m_{2}}^{\F}(\bm{x}) & = e_{m_{1}}^{\mathcal{F}}(\bm{x})^{\Tran}\bm{K}(\bm{x})^{-1}e_{m_{2}}^{\mathcal{F}}(\bm{x}) \nonumber \\
& = \frac{1}{4} \left(e_{m_{1}}^{\mathcal{F}}(\bm{x}) + e_{m_{2}}^{\mathcal{F}}(\bm{x})\right)^{\Tran}\bm{K}(\bm{x})^{-1}\left(e_{m_{1}}^{\mathcal{F}}(\bm{x}) + e_{m_{2}}^{\mathcal{F}}(\bm{x})\right)^{\Tran}  \nonumber \\
& - \frac{1}{4} \left(e_{m_{1}}^{\mathcal{F}}(\bm{x}) - e_{m_{2}}^{\mathcal{F}}(\bm{x})\right)^{\Tran}\bm{K}(\bm{x})^{-1}\left(e_{m_{1}}^{\mathcal{F}}(\bm{x}) - e_{m_{2}}^{\mathcal{F}}(\bm{x})\right)^{\Tran}.
\end{align}
Thus
\begin{equation}
\EX_{\VS} \tau_{m_{1},m_{2}}^{\F}(\bm{x}) = \frac{1}{4 Z_{N}(k)}\int_{\mathcal{X}^{N}} \left( \Psi^{+}(\bm{x}) - \Psi^{-}(\bm{x}) \right) \otimes_{i \in [N]}\mathrm{d}\omega(x_{i}),
\end{equation}
where
\begin{equation}
\Psi^{+}(\bm{x}) = \left(e_{m_{1}}^{\mathcal{F}}(\bm{x})+e_{m_{2}}^{\mathcal{F}}(\bm{x})\right)^{\Tran}\bm{K}(\bm{x})^{-1}\left(e_{m_{1}}^{\mathcal{F}}(\bm{x})+e_{m_{2}}^{\mathcal{F}}(\bm{x})\right) \Det \bm{K}(\bm{x}),
\end{equation}
and
\begin{equation}
\Psi^{-}(\bm{x}) = \left(e_{m_{1}}^{\mathcal{F}}(\bm{x})-e_{m_{2}}^{\mathcal{F}}(\bm{x})\right)^{\Tran}\bm{K}(\bm{x})^{-1}\left(e_{m_{1}}^{\mathcal{F}}(\bm{x})-e_{m_{2}}^{\mathcal{F}}(\bm{x})\right) \Det \bm{K}(\bm{x}).
\end{equation}
We proceed as in Section~\ref{sec:proof_EX_VS_lvs_n_diff_phi} and we use Proposition~\ref{thm:jacobi_identity} to prove that
\begin{align}
\Psi^{+}(\bm{x}) & = \partial_{t} \Det \left(\bm{K}(\bm{x})+t \left(e_{m_{1}}^{\mathcal{F}}(\bm{x}) + e_{m_{2}}^{\mathcal{F}}(\bm{x}) \right)\left(e_{m_{1}}^{\mathcal{F}}(\bm{x}) + e_{m_{2}}^{\mathcal{F}}(\bm{x}) \right)^{\Tran}\right)|_{t = 0^{+}} \nonumber \\
& = \partial_{t} \Det \left(\bm{K}_{t}^{+}(\bm{x})\right)|_{t = 0^{+}}.
\end{align}
and
\begin{align}
\Psi^{-}(\bm{x}) & = \partial_{t} \Det \left(\bm{K}(\bm{x})+t \left(e_{m_{1}}^{\mathcal{F}}(\bm{x}) - e_{m_{2}}^{\mathcal{F}}(\bm{x}) \right)\left(e_{m_{1}}^{\mathcal{F}}(\bm{x}) - e_{m_{2}}^{\mathcal{F}}(\bm{x}) \right)^{\Tran} \right)|_{t = 0^{+}} \nonumber \\
& = \partial_{t} \Det \left(\bm{K}_{t}^{-}(\bm{x})\right)|_{t = 0^{+}}.
\end{align}
In order to prove that $\phi_{m_{1},m_{2}}^{+}$ and $\phi_{m_{1},m_{2}}^{-}$ are right differentiable in zero along with the second equation in Proposition~\ref{thm:EX_lvs_phi}, one can follow the same steps as in the end of Section~\ref{sec:proof_EX_VS_lvs_n_diff_phi}. In particular, the interchange of the derivation and the integration operations follows from the same arguments, upon noting that both $\int_{\X} k_{t}^{+}(x,x) \mathrm{d}\omega(x)$ and $\int_{\X} k_{t}^{-}(x,x) \mathrm{d}\omega(x)$ are finite.

\subsection{Proof of Proposition~\ref{prop:EX_VS_lvs}}\label{sec:proof_EX_VS_lvs}
The proof is a straightforward computation now that we have Proposition~\ref{thm:EX_lvs_phi} and Proposition~\ref{prop:perturbed_kernels_spectrum}.

\subsubsection{The expected value of the $m$-th leverage score}
Let $m \in \Ns$. We have by Proposition~\ref{thm:EX_lvs_phi} and Proposition~\ref{prop:VS_decomposition},
\begin{equation}
\EX_{\VS} \tau_{m}^{\F}(\bm{x})  = \frac{1}{N!\sum\limits_{\substack{U \subset \mathbb{N}^{*}\\ |U| = N}} \prod\limits_{u \in U}\sigma_{u}}  \frac{\partial \phi_{m}}{ \partial t }|_{t = 0^{+}},
\label{e:expression_of_expected_LS}\end{equation}
where
\begin{equation}
\phi_{m}(t) = \int_{\mathcal{X}^{N}} \Det \bigg(\bm{K}(\bm{x})+t e_{m}^{\mathcal{F}}(\bm{x})e_{m}^{\mathcal{F}}(\bm{x})^{\Tran} \bigg) \otimes_{i =1}^{N} \mathrm{d}\omega(x_{i}).
\end{equation}
Now by Proposition~\ref{prop:perturbed_kernels_spectrum} and Proposition~\ref{prop:VS_decomposition},
\begin{equation}
\phi_{m}(t) = N! \sum\limits_{U\in\UN} \prod\limits_{u \in U} \tilde{\sigma}_{u}(t),
\end{equation}
where for $u \in \mathbb{N}^{*}$, $\tilde{\sigma}_{u}(t) = \sigma_{u} + t\delta_{m,u}\sigma_{u}$.
Therefore,
\begin{align}
\phi_{m}(t) & = N!\sum\limits_{\substack{U \in\UN \\ m \in U}} \prod\limits_{u \in U} \tilde{\sigma}_{u}(t) + N!\sum\limits_{\substack{U \in\UN \\ m \notin U}} \prod\limits_{u \in U} \tilde{\sigma}_{u}(t) \\
& = N!\,\sigma_{m}(t+1)\sum\limits_{\substack{U\in \mathcal{U}_{N-1}\\ m \notin U}} \prod\limits_{u \in U} \sigma_{u} + N!\sum\limits_{\substack{U \in\UN \\ m \notin U}} \prod\limits_{u \in U} \sigma_{u}.
\end{align}
Thus,
\begin{equation}
\frac{\partial \phi_{m}}{\partial t}|_{t = 0} = N!\,\sigma_{m}\sum\limits_{\substack{U\in \mathcal{U}_{N-1}\\ m \notin U}} \prod\limits_{u \in U} \sigma_{u},
\end{equation}
so that \eqref{e:expression_of_expected_LS} becomes
\begin{align}
\EX_{\VS} \tau_{m}^{\F}(\bm{x}) & = \left(\cancel{N!}\sum\limits_{U\in\UN}\prod\limits_{u \in U}\sigma_{u} \right)^{-1} \cancel{N!}\,\sigma_{m}\sum\limits_{\substack{U\in \mathcal{U}_{N-1}\\ m \notin U}} \prod\limits_{u \in U} \sigma_{u},
\end{align}
which concludes the proof.

\subsubsection{The expected value of cross-leverage scores}
Let $m_{1},m_{2} \in \Ns$ such that $m_{1} \neq m_{2}$. We have by Proposition~\ref{thm:EX_lvs_phi} and Proposition~\ref{prop:VS_decomposition},
\begin{equation}
\EX_{\VS} \tau_{m_{1},m_{2}}^{\F}(\bm{x})  = \frac{1}{4N!\sum\limits_{U\in\UN} \prod\limits_{u \in U}\sigma_{u}} \left( \frac{\partial \phi_{m_{1},m_{2}}^{+}}{ \partial t } - \frac{\partial \phi_{m_{1},m_{2}}^{-}}{ \partial t }\right)\bigg|_{t = 0^{+}} ,
\label{e:expression_of_expected_cross_LS}
\end{equation}
where
\begin{equation}
\phi_{m_{1},m_{2}}^{+}(t) =  \int_{\mathcal{X}^{N}} \Det \bigg(\bm{K}(\bm{x})+t \left(e_{m_{1}}^{\mathcal{F}}(\bm{x}) + e_{m_{2}}^{\mathcal{F}}(\bm{x}) \right) \left(e_{m_{1}}^{\mathcal{F}}(\bm{x}) + e_{m_{2}}^{\mathcal{F}}(\bm{x}) \right)^{\Tran} \bigg) \otimes_{i =1}^{N} \mathrm{d}\omega(x_{i}),
\end{equation}
and
\begin{equation}
\phi_{m_{1},m_{2}}^{-}(t) =  \int_{\mathcal{X}^{N}} \Det \bigg(\bm{K}(\bm{x})+t \left(e_{m_{1}}^{\mathcal{F}}(\bm{x}) - e_{m_{2}}^{\mathcal{F}}(\bm{x}) \right) \left(e_{m_{1}}^{\mathcal{F}}(\bm{x}) - e_{m_{2}}^{\mathcal{F}}(\bm{x}) \right)^{\Tran} \bigg) \otimes_{i =1}^{N} \mathrm{d}\omega(x_{i}).
\end{equation}
Now by Proposition~\ref{prop:perturbed_kernels_spectrum}, for $t\geq 0$,
\begin{equation}
\phi_{m_{1},m_{2}}^{+}(t) = N!\sum\limits_{U\in\UN} \prod\limits_{u \in U} \tilde{\sigma}_{u}^{+}(t) = N!\sum\limits_{U\in\UN} \prod\limits_{u \in U} \tilde{\sigma}_{u}^{-}(t) = \phi_{m_{1},m_{2}}^{-}(t).
\end{equation}
Plugging this back into \eqref{e:expression_of_expected_cross_LS} yields $\EX_{\VS} \tau_{m_{1},m_{2}}^{\F}(\bm{x}) = 0$.

\subsection{Proof of Theorem~\ref{thm:slow_rates}}\label{app:proof_slow_rates}
\subsubsection{A decomposition result for the error}
We start with a lemma.
\begin{lemma}\label{lemma:delayed_bounds}
Let $\mu \in \F$ such that $\|\mu\|_{\F} \leq 1$. Under Assumption~\ref{hyp:beta_N_bounded},
\begin{equation}\label{eq:EX_VS_slow_rates}
\EX_{\VS} \mathcal{E}(\mu;x)^{2} \leq (1+B) \sum\limits_{m \in [N]} \frac{\sigma_{N}}{\sigma_{m}} \langle \mu,e_{m}^{\F} \rangle_{\F}^{2} + \sum\limits_{m \geq N+1} \langle \mu,e_{m}^{\F} \rangle_{\F}^{2}.
\end{equation}
\end{lemma}
\begin{proof}
Using the same arguments as in the proof of Lemma~\ref{lemma:error_decomposition} in Section~\ref{app:proof_error_decomposition}, it comes that,
for $\bm{x} \in \mathcal{X}^{N}$ such that $\Det \bm{K}(\bm{x}) > 0$,
\begin{equation}\label{eq:general_error_decomposition}
\|\mu - \Pi_{\mathcal{T}(\bm{x})} \mu\|_{\F}^{2}  = \sum\limits_{m \in \mathbb{N}^{*}} \langle \mu, e_{m}^{\F} \rangle_{\F}^{2}\bigg(1- \tau_{m}^{\F}(\bm{x})\bigg) - \sum\limits_{\substack{m_{1},m_{2} \in \mathbb{N}^{*} \\ m_{1} \neq m_{2}}}  \langle \mu, e_{m_{1}}^{\F} \rangle_{\F} \langle \mu, e_{m_{2}}^{\F} \rangle_{\F} \tau_{m_{1},m_{2}}^{\F}(\bm{x}).
\end{equation}
We want to take expectations in both sides of \eqref{eq:general_error_decomposition}. For the first term in the RHS, we prove, using the same arguments as for the proof of Theorem~\ref{thm:main_result_1} in Section~\ref{app:proof_main_result_1}, that
\begin{equation}\label{eq:EX_VS_lvs_interpolation}
\EX_{\VS} \sum\limits_{m \in \mathbb{N}^{*}} \langle \mu, e_{m}^{\F} \rangle_{\F}^{2}\bigg(1- \tau_{m}^{\F}(\bm{x})\bigg) = \sum\limits_{m \in \mathbb{N}^{*}} \langle \mu, e_{m}^{\F} \rangle_{\F}^{2}\bigg(1- \EX_{\VS}\tau_{m}^{\F}(\bm{x})\bigg).
\end{equation}
For the second term in the RHS of \eqref{eq:general_error_decomposition}, we need to justify that
\begin{align}
\EX_{\VS} \sum\limits_{\substack{m_{1},m_{2} \in \mathbb{N}^{*} \\ m_{1} \neq m_{2}}}  &\langle \mu, e_{m_{1}}^{\F} \rangle_{\F} \langle \mu, e_{m_{2}}^{\F} \rangle_{\F} \, \tau_{m_{1},m_{2}}^{\F}(\bm{x}) \nonumber\\
& = \sum\limits_{\substack{m_{1},m_{2} \in \mathbb{N}^{*} \\ m_{1} \neq m_{2}}}  \langle \mu, e_{m_{1}}^{\F} \rangle_{\F} \langle \mu, e_{m_{2}}^{\F} \rangle_{\F} \,\EX_{\VS}\tau_{m_{1},m_{2}}^{\F}(\bm{x}) = 0.
\label{eq:EX_VS_cross_lvs_interpolation}
\end{align}
This can be done using dominated convergence. Indeed, let $M \in \Ns$. We have
\begin{align}
\EX_{\VS} \sum\limits_{\substack{m_{1},m_{2} \in [M] \\ m_{1} \neq m_{2}}}  &\langle \mu, e_{m_{1}}^{\F} \rangle_{\F} \langle \mu, e_{m_{2}}^{\F} \rangle_{\F} \,\tau_{m_{1},m_{2}}^{\F}(\bm{x}) \nonumber\\
& =  \sum\limits_{\substack{m_{1},m_{2} \in [M] \\ m_{1} \neq m_{2}}}  \langle \mu, e_{m_{1}}^{\F} \rangle_{\F} \langle \mu, e_{m_{2}}^{\F} \rangle_{\F} \,\EX_{\VS} \tau_{m_{1},m_{2}}^{\F}(\bm{x}) = 0.
\label{eq:EX_VS_cross_lvs_interpolation_finite_sum}
\end{align}
Moreover,
\begin{align}
\Bigg\vert\sum\limits_{\substack{m_{1},m_{2} \in [M] \\ m_{1} \neq m_{2}}}  &\langle \mu, e_{m_{1}}^{\F} \rangle_{\F} \langle \mu, e_{m_{2}}^{\F} \rangle_{\F} \,\tau_{m_{1},m_{2}}^{\F}(\bm{x})\Bigg\vert\nonumber \\
& = \left|\sum\limits_{\substack{m_{1},m_{2} \in [M]}}  \langle \mu, e_{m_{1}}^{\F} \rangle_{\F} \langle \mu, e_{m_{2}}^{\F} \rangle_{\F} \,\tau_{m_{1},m_{2}}^{\F}(\bm{x}) - \sum\limits_{m \in [M]} \langle \mu, e_{m}^{\F} \rangle_{\F}^{2} \,\tau_{m}^{\F}(\bm{x})\right|  \nonumber\\
 & \leq \left|\sum\limits_{m_{1},m_{2} \in [M]}  \langle \mu, e_{m_{1}}^{\F} \rangle_{\F} \langle \mu, e_{m_{2}}^{\F} \rangle_{\F} \,\tau_{m_{1},m_{2}}^{\F}(\bm{x})\right|
 +  \left|\sum\limits_{m \in [M]} \langle \mu, e_{m}^{\F} \rangle_{\F}^{2} \,\tau_{m}^{\F}(\bm{x})\right| \nonumber\\
& = \left\| \Pi_{\mathcal{T}(\bm{x})} \sum\limits_{m \in [M]} \langle \mu, e_{m}^{\F} \rangle_{\F} e_{m}^{\F}\right\|_{\F}^{2}  +  \left|\sum\limits_{m \in [M]} \langle \mu, e_{m}^{\F} \rangle_{\F}^{2} \,\tau_{m}^{\F}(\bm{x})\right| \nonumber\\
& \leq \left\|\sum\limits_{m \in [M]} \langle \mu, e_{m}^{\F} \rangle_{\F} e_{m}^{\F}\right\|_{\F}^{2}  +  \sum\limits_{m \in [M]} \langle \mu, e_{m}^{\F} \rangle_{\F}^{2} \nonumber \\
& = 2 \|\mu\|_{\F}^{2} < +\infty.
\label{eq:EX_VS_cross_lvs_interpolation_domination}
\end{align}
Combining \eqref{eq:EX_VS_cross_lvs_interpolation_finite_sum} and \eqref{eq:EX_VS_cross_lvs_interpolation_domination}, we deduce \eqref{eq:EX_VS_cross_lvs_interpolation} by the dominated convergence theorem.

Finally, we combine \eqref{eq:EX_VS_lvs_interpolation} and \eqref{eq:EX_VS_cross_lvs_interpolation} to get
\begin{align}
\EX_{\VS}\|\mu - \Pi_{\mathcal{T}(\bm{x})} \mu\|_{\F}^{2} & = \sum\limits_{m \in \mathbb{N}^{*}} \langle \mu, e_{m}^{\F} \rangle_{\F}^{2}\bigg(1- \EX_{\VS}\tau_{m}^{\F}(\bm{x})\bigg) \nonumber\\
& = \sum\limits_{n \in [N]} \langle \mu, e_{n}^{\F} \rangle_{\F}^{2}\bigg(1- \EX_{\VS}\tau_{n}^{\F}(\bm{x})\bigg) + \sum\limits_{m \geq N+1} \langle \mu, e_{m}^{\F} \rangle_{\F}^{2}\bigg(1- \EX_{\VS}\tau_{m}^{\F}(\bm{x})\bigg).
\end{align}
On the one hand,
\begin{equation}
\forall m \geq N+1, \:\: 1- \EX_{\VS}\tau_{m}^{\F}(\bm{x}) \leq 1,
\end{equation}
and on the other hand, remember that by Theorem~\ref{thm:main_result_1}, the sequence $\epsilon_{m}$ is non-increasing, so that
\begin{align}
\forall n \in [N], \: \sigma_{n} (1-\EX_{\VS} \tau_{n}^{\F}(\bm{x})) & = \EX_{\VS} \|\mu_{e_{n}} - \Pi_{\mathcal{T}(\bm{x})} \mu_{e_{n}}\|_{\F}^{2} \\
& = \epsilon_{n} \\
& \leq \epsilon_{1},
\end{align}
and by \eqref{eq:ineq_r_N} in the same theorem one gets
\begin{equation}
\sigma_{n} (1-\EX_{\VS} \tau_{n}^{\F}(\bm{x})) \leq (1+\beta_{N}) \sigma_{N},
\end{equation}
so that
\begin{equation}
 (1-\EX_{\VS} \tau_{n}^{\F}(\bm{x})) \leq (1+\beta_{N})\frac{ \sigma_{N}}{\sigma_{n}}.
\end{equation}
Assumption~\ref{hyp:beta_N_bounded} yields \
\begin{equation}
\forall n \in [N],\:\: 1- \EX_{\VS}\tau_{n}^{\F}(\bm{x}) \leq  (1+B) \frac{\sigma_{N}}{\sigma_{n}}.
\end{equation}
This concludes the proof of the lemma.
\end{proof}
\subsubsection{The expected value of the interpolation error}
If there exists $r \in [0,1/2]$ such that $\displaystyle \| \Sigma^{-r} \mu \|_{\F} < +\infty$, we have
\begin{align}
\sum\limits_{m \geq N+1} \langle \mu,e_{m}^{\F} \rangle_{\F}^{2}
& = \sum\limits_{m \geq N+1} \sigma_{m}^{2r} \frac{ \langle \mu,e_{m}^{\F} \rangle_{\F}^{2}}{\sigma_{m}^{2r}} \\
& \leq \sigma_{N+1}^{2r} \sum\limits_{m \geq N+1} \frac{ \langle \mu,e_{m}^{\F} \rangle_{\F}^{2}}{\sigma_{m}^{2r}} \\
& \leq \sigma_{N+1}^{2r} \|\Sigma^{-r} \mu\|_{\F}^{2},
\end{align}
and
\begin{align}
(1+B) \sum\limits_{m \in [N]} \frac{\sigma_{N}}{\sigma_{m}} \langle \mu,e_{m}^{\F} \rangle_{\F}^{2}
& = (1+B) \sum\limits_{m \in [N]} \frac{\sigma_{N}}{\sigma_{m}^{1-2r+2r}} \langle \mu,e_{m}^{\F} \rangle_{\F}^{2} \\
& = (1+B) \sum\limits_{m \in [N]} \frac{\sigma_{N}}{\sigma_{m}^{1-2r}} \frac{\langle \mu,e_{m}^{\F} \rangle_{\F}^{2}}{\sigma_{m}^{2r}} \\
& \leq (1+B) \sigma_{N}^{2r} \sum\limits_{m \in [N]} \frac{\langle \mu,e_{m}^{\F} \rangle_{\F}^{2}}{\sigma_{m}^{2r}} \\
& = (1+B) \sigma_{N}^{2r} \|\Sigma^{-r} \mu\|_{\F}^{2}.
\end{align}
By Lemma~\ref{lemma:delayed_bounds}, $\displaystyle \EX_{\VS}\| \mu - \Pi_{\mathcal{T}(\bm{x})} \mu \|_{\F}^{2}$ converges at the slow rate $\mathcal{O}(\sigma_{N}^{2r})$.

On the other hand, if there exists $r > 1/2$ such that $\displaystyle \| \Sigma^{-r} \mu \|_{\F} < +\infty$, we have
\begin{align}
(1+B) \sum\limits_{m \in [N]} \frac{\sigma_{N}}{\sigma_{m}} \langle \mu,e_{m}^{\F} \rangle_{\F}^{2}
& = (1+B) \sum\limits_{m \in [N]} \frac{\sigma_{N}}{\sigma_{m}^{1-2r+2r}} \langle \mu,e_{m}^{\F} \rangle_{\F}^{2} \\
& \leq (1+B) \sigma_{N} \sigma_{1}^{2r-1} \sum\limits_{m \in [N]} \frac{\langle \mu,e_{m}^{\F} \rangle_{\F}^{2}}{\sigma_{m}^{2r}} \\
&  \leq (1+B) \sigma_{N} \sigma_{1}^{2r-1} \|\Sigma^{-r} \mu\|_{\F}^{2},
\end{align}
and
\begin{align}
\sum\limits_{m \geq N+1} \langle \mu,e_{m}^{\F} \rangle_{\F}^{2}
& = \sum\limits_{m \geq N+1} \sigma_{m}^{2r} \frac{ \langle \mu,e_{m}^{\F} \rangle_{\F}^{2}}{\sigma_{m}^{2r}} \\
& \leq \sigma_{N+1}^{2r} \sum\limits_{m \geq N+1} \frac{ \langle \mu,e_{m}^{\F} \rangle_{\F}^{2}}{\sigma_{m}^{2r}} \\
& \leq \sigma_{N+1}^{2r} \|\Sigma^{-r} \mu\|_{\F}^{2}.
\end{align}
This time, the bound in Lemma~\ref{lemma:delayed_bounds} is dominated by its first term, so that $\displaystyle \EX_{\VS}\| \mu - \Pi_{\mathcal{T}(\bm{x})} \mu \|_{\F}^{2}$ converges at the faster rate $\mathcal{O}(\sigma_{N})$.

\subsection{Proof of Theorem~\ref{thm:EX_VS_integration_error}}\label{app:proof_bias}
\subsubsection{Proof of the bias identity}
First, recall that, as $f$ and $g$ belong to $\Ltwo$, we have
\begin{equation}
\int_{\X}f(x)g(x) \Mu = \sum\limits_{m \in \Ns} \langle f,e_{m}\rangle_{\mathrm{d}\omega} \langle g,e_{m}\rangle_{\mathrm{d}\omega},
\end{equation}
thus, in order to prove the result, it is enough to prove that
\begin{equation}\label{eq:EX_VS_optimal_quadrature_formula}
\EX_{\VS} \sum\limits_{i \in [N]} \widehat{w}_{i}f(x_{i}) = \sum\limits_{m \in \Ns} \langle f,e_{m} \rangle_{\mathrm{d}\omega} \langle g,e_{m} \rangle_{\mathrm{d}\omega} \EX_{\VS}\tau_{m}^{\F}(\bm{x}).
\end{equation}
Let $\bm{x} \in \X^{N}$ such that $\Det \bm{K}(\bm{x})>0$. The optimal kernel quadrature weights satisfy
\begin{equation}
\widehat{\bm{w}} = \bm{K}(\bm{x})^{-1}\mu_{g}(\bm{x}),
\end{equation}
so that
\begin{align}
\sum\limits_{i \in N} \widehat{w}_{i}f(x_{i}) & = \widehat{\bm{w}}^{\Tran} f(\bm{x}) \\
& = \mu_{g}(\bm{x})^{\Tran} \bm{K}(\bm{x})^{-1} f(\bm{x})\\
& = \sum_{m_{1},m_{2} \in \Ns} \sigma_{m_{1}} \langle g, e_{m_{1}} \rangle_{\mathrm{d}\omega} \langle f, e_{m_{2}}^{\F} \rangle_{\F} \; e_{m_{1}}(\bm{x})^{\Tran} \bm{K}(\bm{x})^{-1}e_{m_{2}}^{\F}(\bm{x}) \\
& = \sum\limits_{m_{1},m_{2} \in \Ns} \sqrt{\sigma_{m_{1}}} \langle g, e_{m_{1}} \rangle_{\mathrm{d}\omega}  \langle f, e_{m_{2}}^{\F} \rangle_{\F}\;  e_{m_{1}}^{\F}(\bm{x})^{\Tran} \bm{K}(\bm{x})^{-1}e_{m_{2}}^{\F}(\bm{x}).
\label{e:last_tool}
\end{align}
We want to use the dominated convergence theorem to take expectations in \eqref{e:last_tool}.
Let $M \in \Ns$.
By Lemma~\ref{lemma:lvs_identities} and by the fact that $\Pi_{\mathcal{T}(\bm{x})}$ is an $\langle.,.\rangle_{\F}$-orthogonal projection, it comes
\begin{align}
\Bigg|\sum_{m_{1}, m_2 \in [M]} &\sqrt{\sigma_{m_{1}}} \langle g, e_{m_{1}} \rangle_{\mathrm{d}\omega} \langle f, e_{m_{2}}^{\F} \rangle_{\F}  \: e_{m_{1}}^{\F}(\bm{x})^{\Tran} \bm{K}(\bm{x})^{-1}e_{m_{2}}^{\F}(\bm{x}) \Bigg| \\
= & \left|\sum_{m_{1}, m_2 \in [M]} \sqrt{\sigma_{m_{1}}} \langle g, e_{m_{1}} \rangle_{\mathrm{d}\omega}  \langle f, e_{m_{2}}^{\F} \rangle_{\F}  \: \langle \Pi_{\mathcal{T}(\bm{x})}e_{m_{1}}^{\F}, \Pi_{\mathcal{T}(\bm{x})}e_{m_{2}}^{\F} \rangle_{\F} \right|\\
= & \left| \left\langle \Pi_{\mathcal{T}(\bm{x})} \sum\limits_{m_{1} \in [M] }  \sqrt{\sigma_{m_{1}}} \langle g, e_{m_{1}} \rangle_{\mathrm{d}\omega}   e_{m_{1}},\:  \Pi_{\mathcal{T}(\bm{x})} \sum\limits_{m_{2} \in [M]} \langle f, e_{m_{2}}^{\F} \rangle_{\F} e_{m_{2}}^{\F} \right\rangle_{\F} \right|\\
\leq & \left| \left\langle \sum\limits_{m_{1} \in [M] }  \sqrt{\sigma_{m_{1}}} \langle g, e_{m_{1}} \rangle_{\mathrm{d}\omega}   e_{m_{1}},  \sum\limits_{m_{2} \in [M]} \langle f, e_{m_{2}}^{\F} \rangle_{\F} e_{m_{2}}^{\F} \right\rangle_{\F} \right|\\
\leq & \left\| \sum\limits_{m_{1} \in [M] }  \sqrt{\sigma_{m_{1}}} \langle g, e_{m_{1}} \rangle_{\mathrm{d}\omega}   e_{m_{1}} \right\|_{\F}\: \left \| \sum\limits_{m_{2} \in [M]} \langle f, e_{m_{2}}^{\F} \rangle_{\F} e_{m_{2}}^{\F} \right\|_{\F}.
\end{align}
Now,
\begin{align}
\left\| \sum\limits_{m_{1} \in [M] }  \sqrt{\sigma_{m_{1}}} \langle g, e_{m_{1}} \rangle_{\mathrm{d}\omega}   e_{m_{1}} \right\|_{\F}  & \left \| \sum\limits_{m_{2} \in [M]} \langle f, e_{m_{2}}^{\F} \rangle_{\F} e_{m_{2}}^{\F} \right\|_{\F}\\
&  = \left\| \sum\limits_{m_{1} \in [M] }  \langle g, e_{m_{1}} \rangle_{\mathrm{d}\omega}   e_{m_{1}}^{\F} \right\|_{\F} \left \| \sum\limits_{m_{2} \in [M]} \langle f, e_{m_{2}}^{\F} \rangle_{\F} e_{m_{2}}^{\F} \right\|_{\F}\\
&  = \left\| \sum\limits_{m_{1} \in [M] }  \langle g, e_{m_{1}} \rangle_{\mathrm{d}\omega}   e_{m_{1}} \right\|_{\mathrm{d}\omega} \left \| \sum\limits_{m_{2} \in [M]} \langle f, e_{m_{2}}^{\F} \rangle_{\F} e_{m_{2}}^{\F} \right\|_{\F}\\
&  \leq \left\| \sum\limits_{m_{1} \in \Ns }  \langle g, e_{m_{1}} \rangle_{\mathrm{d}\omega}   e_{m_{1}} \right\|_{\mathrm{d}\omega} \left \| \sum\limits_{m_{2} \in \Ns} \langle f, e_{m_{2}}^{\F} \rangle_{\F} e_{m_{2}}^{\F} \right\|_{\F}\\
& < +\infty,
\end{align}
since $\sum_{m \in \Ns} \sqrt{\sigma_{m}}\langle g,e_{m} \rangle_{\mathrm{d}\omega} e_{m} \in \F$.
Dominated convergenve thus yields
\begin{align}
\EX_{\VS} \sum\limits_{m_{1},m_{2} \in \Ns} \sqrt{\sigma_{m_{1}}} & \langle g, e_{m_{1}} \rangle_{\mathrm{d}\omega}  \langle f, e_{m_{2}}^{\F} \rangle_{\F} \:e_{m_{1}}^{\F}(\bm{x})^{\Tran} \bm{K}(\bm{x})^{-1}e_{m_{2}}^{\F}(\bm{x}) \\
& = \sum\limits_{m_{1},m_{2} \in \Ns} \sqrt{\sigma_{m_{1}}} \langle g, e_{m_{1}} \rangle_{\mathrm{d}\omega}  \langle f, e_{m_{2}}^{\F} \rangle_{\F} \:\EX_{\VS} e_{m_{1}}^{\F}(\bm{x})^{\Tran} \bm{K}(\bm{x})^{-1}e_{m_{2}}^{\F}(\bm{x}).
\end{align}
Using Proposition~\ref{prop:EX_VS_lvs}, we continue our derivation as
\begin{align}
\EX_{\VS} \sum\limits_{m_{1},m_{2} \in \Ns} \sqrt{\sigma_{m_{1}}} & \langle g, e_{m_{1}} \rangle_{\mathrm{d}\omega}  \langle f, e_{m_{2}}^{\F} \rangle_{\F} \: e_{m_{1}}^{\F}(\bm{x})^{\Tran} \bm{K}(\bm{x})^{-1}e_{m_{2}}^{\F}(\bm{x}) \\
& = \sum\limits_{m \in \Ns} \sqrt{\sigma_{m}} \langle g, e_{m} \rangle_{\mathrm{d}\omega}  \langle f, e_{m}^{\F} \rangle_{\F} \: \EX_{\VS} e_{m}^{\F}(\bm{x})^{\Tran} \bm{K}(\bm{x})^{-1}e_{m}^{\F}(\bm{x})\\
& = \sum\limits_{m \in \Ns}  \langle g, e_{m} \rangle_{\mathrm{d}\omega}  \sqrt{\sigma_{m}}\langle f, e_{m}^{\F} \rangle_{\F} \: \EX_{\VS} \tau_{m}^{\F}(\bm{x}).
\end{align}
Finally, \eqref{eq:EX_VS_optimal_quadrature_formula} is obtained upon noting that
\begin{equation}
\forall m \in \Ns, \: \langle f,e_{m} \rangle_{\mathrm{d}\omega} = \sqrt{\sigma_{m}} \langle f,e_{m}^{\F} \rangle_{\F}.
\end{equation}
\subsubsection{Proof of the asymptotic unbiasedness of the quadrature }

The expected value of the bias writes
\begin{equation}
\EX_{\VS} \left( \int_{\X} f(x)g(x) \mathrm{d}\omega(x) - \sum\limits_{i \in [N]} \widehat{w}_{i}f(x_{i}) \right)= \sum\limits_{m \in \Ns} \langle f,e_{m} \rangle_{\mathrm{d}\omega}\langle g,e_{m} \rangle_{\mathrm{d}\omega} \left( 1- \EX_{\VS}\tau_{m}^{\F}(\bm{x}) \right).
\label{e:very_last_tool}
\end{equation}
Now, by Theorem~\ref{thm:main_result_1}, for $m \in \Ns$,
\begin{equation}
\EX_{\VS} \|\mu_{e_{m}} - \Pi_{\mathcal{T}(\bm{x})} \mu_{e_{m}}\|_{\F}^{2} \leq \epsilon_{1} \leq \sigma_{N}(1+\beta_{N}) \leq \sigma_{N} + \sum\limits_{n \geq N} \sigma_{n}.
\end{equation}
Thus
\begin{equation}
0 \leq 1-\EX_{\VS} \tau_{m}^{\F}(\bm{x}) = {\sigma_m}^{-1} \EX_{\VS} \|\mu_{e_{m}} - \Pi_{\mathcal{T}(\bm{x})} \mu_{e_{m}}\|_{\F}^{2} \leq {\sigma_{m}}^{-1}{\sigma_{N} + \sum\limits_{n \geq N} \sigma_{n}},
\end{equation}
so that
\begin{align}
\lim_{N \rightarrow \infty} \langle f,e_{m} \rangle_{\mathrm{d}\omega}\langle g,e_{m} \rangle_{\mathrm{d}\omega} \left(1- \EX_{\VS}\tau_{m}^{\F}(\bm{x}) \right) & = \langle f,e_{m} \rangle_{\mathrm{d}\omega}\langle g,e_{m} \rangle_{\mathrm{d}\omega} (1-\lim_{N \rightarrow \infty} \EX_{\VS}\tau_{m}^{\F}(\bm{x})) = 0.
\end{align}
To conclude, it is thus enough to apply the dominated convergence theorem to \eqref{e:very_last_tool}. By Lemma~\ref{lemma:lvs_identities}, $\tau_{m}^{\F}(\bm{x}) \in [0,1]$, so that $1-\EX_{\VS}\tau_{m}^{\F}(\bm{x}) \in [0,1]$.
In particular, for all $N \in \Ns$,
\begin{align}
|\langle f,e_{m} \rangle_{\mathrm{d}\omega}\langle g,e_{m} \rangle_{\mathrm{d}\omega} \left(1- \EX_{\VS}\tau_{m}^{\F}(\bm{x}) \right)| & \leq |\langle f,e_{m} \rangle_{\mathrm{d}\omega}\langle g,e_{m} \rangle_{\mathrm{d}\omega}|\\
& \leq \frac{1}{2}\left(\left\langle f,e_{m} \right\rangle_{\mathrm{d}\omega}^{2} + \left\langle g,e_{m} \right\rangle_{\mathrm{d}\omega}^{2}\right),
\end{align}
which is the generic term of a convergent series as $f,g\in\Ltwo$. This concludes the proof.
%
\section{More concrete examples of RKHSs}
In this section, we illustrate the bound of Theorem~\ref{thm:main_result_1} and the constants of Proposition~\ref{prop:constant_bound} on more examples.
\subsection{The uni-dimensional periodic Sobolev spaces}

Consider the uni-dimensional periodic Sobolev space of smoothness parameter $s \in \{1,2,3,4,5\}$.
The eigenvalues have a polynomial decay; see \citep{Wah90}. We take for $m \in \Ns$, $\sigma_{m}=m^{-2s}$ \footnote{We drop the potential multiplicities of the eigenvalues by simplicity.}.
For different values of $m$, Figure~\ref{fig:periodic_sobolev} illustrates the expected value of the $m$-th leverage score $\EX_{\VS} \tau_{m}^{\F}(\bm{x})$ (left panels) and the expected interpolation error $\EX_{\VS} \mathcal{E}(\mu_{e_{m}};\bm{x})^{2}$ (right panels), both as functions of $N$.
Remember that by Theorem~\ref{thm:main_result_1}:
\begin{equation}
\EX_{\VS} \mathcal{E}(\mu_{e_{m}};\bm{x})^{2} = \sigma_{m} \left(\sum\limits_{ U \in \: \UN} \prod\limits_{u \in U} \sigma_{u} \right)^{-1}  \sum\limits_{U \in \: \UNm} \prod\limits_{u \in U} \sigma_{u}.
\end{equation}
For numerical simulations, we make the following approximation

\begin{equation}\label{eq:error_numerical_approximation}
\EX_{\VS} \mathcal{E}(\mu_{e_{m}};\bm{x})^{2} \approx \sigma_{m} \left(\sum\limits_{\substack{U \subset \: [M]\\ |U|=N}} \prod\limits_{u \in U} \sigma_{u} \right)^{-1}  \sum\limits_{\substack{U \subset [M]\\ |U|=N, \: m \notin U}} \prod\limits_{u \in U} \sigma_{u},
\end{equation}
for an $M \geq N$ sufficiently large. The numerator and denominator of the right hand side of \eqref{eq:error_numerical_approximation} can be calculated using an efficient algorithm for the calculation of the elementary symmetric polynomials \citep{KuTa12}[Algorithm 7].

We observe that for low values of $s$, $\EX_{\VS} \tau_{m}^{\F}(\bm{x})$ depends smoothly on $N$. On the other hand, $\EX_{\VS} \tau_{m}^{\F}(\bm{x})$ undergoes a sharp transition at $N = m$ for high values of $s$: the reconstruction of the $m$-th eigenfunction is almost perfect for $N$ slightly larger than $m$. Moreover, $\EX_{\VS} \mathcal{E}(\mu_{e_{m}};\bm{x})^{2}$ respects the upper bound of Theorem~\ref{thm:main_result_1};
the constant $B$ of Proposition~\ref{prop:constant_bound} is small for high values of $s$ and converges to $e$ when $s \rightarrow +\infty$.
\begin{figure}[h]
    \centering
\includegraphics[width= 0.48\textwidth]{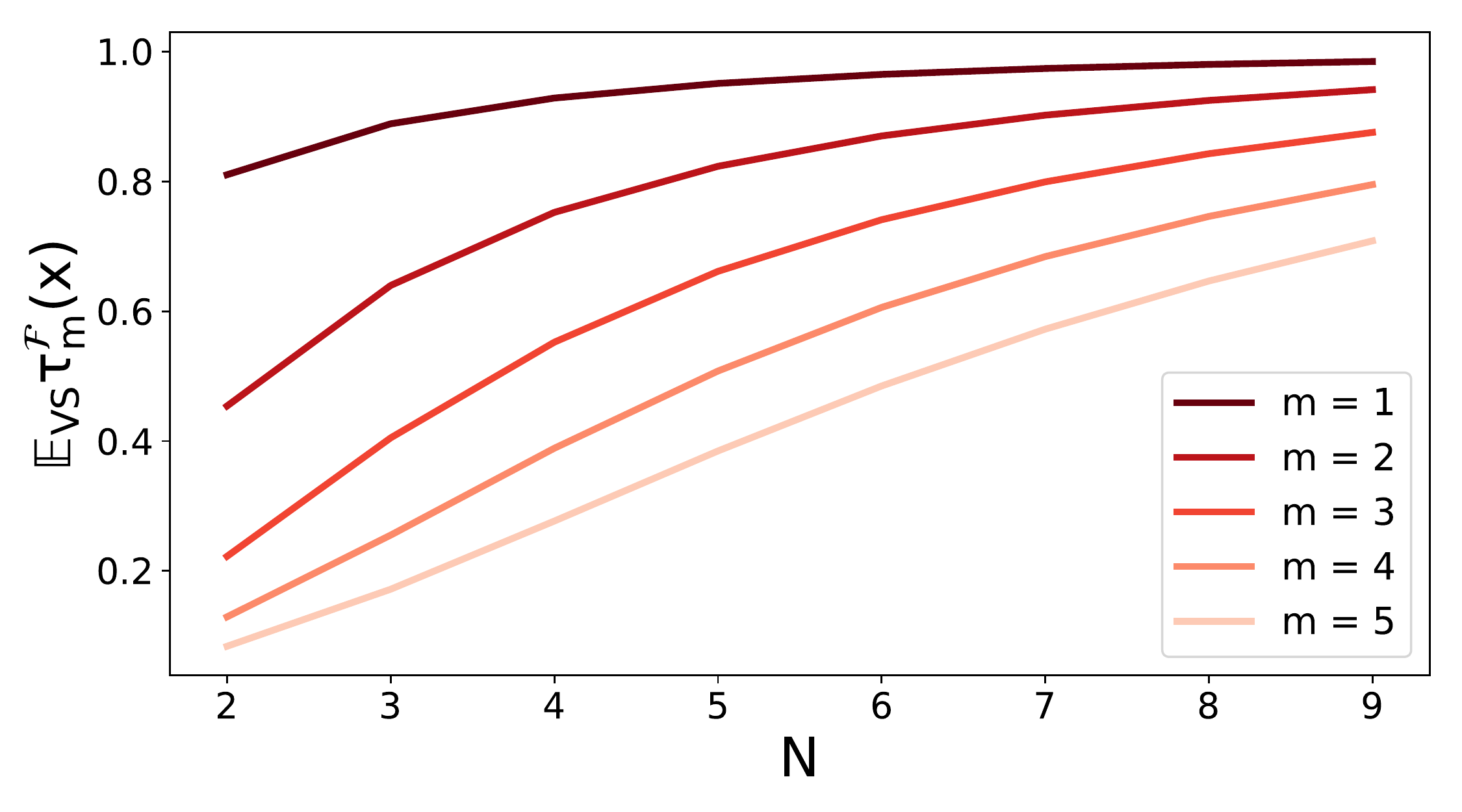}~\includegraphics[width= 0.48\textwidth]{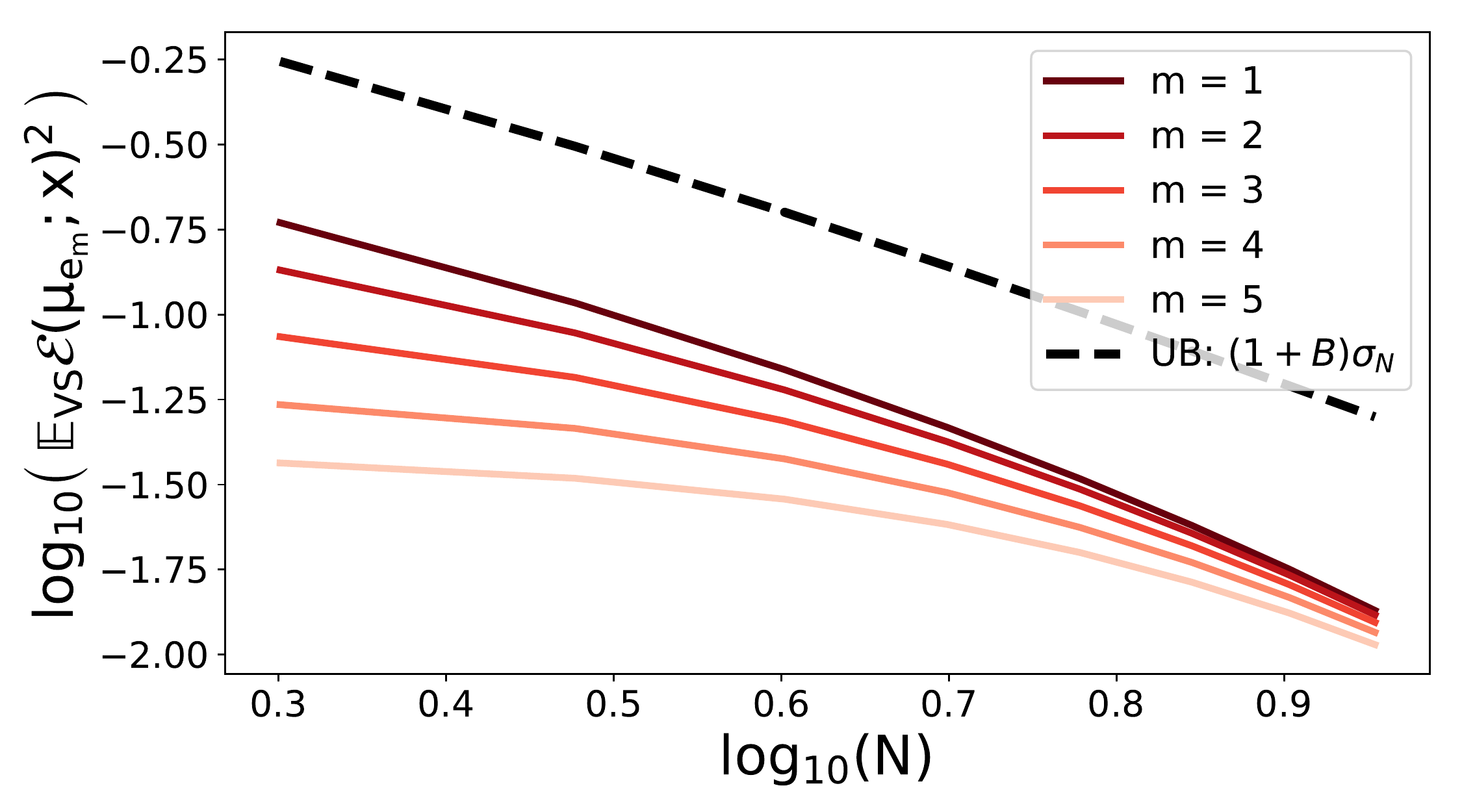}\\
\includegraphics[width= 0.48\textwidth]{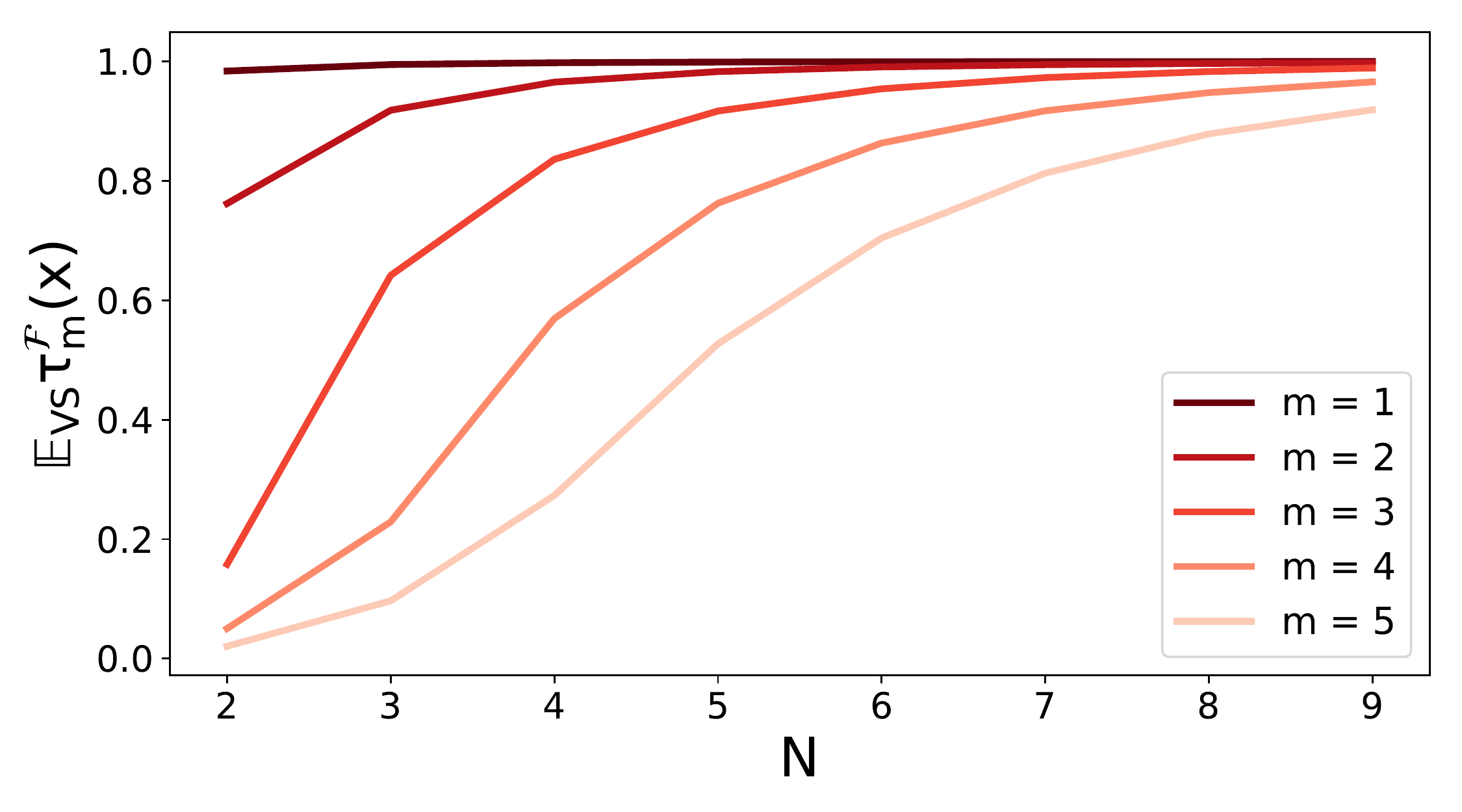}~\includegraphics[width= 0.48\textwidth]{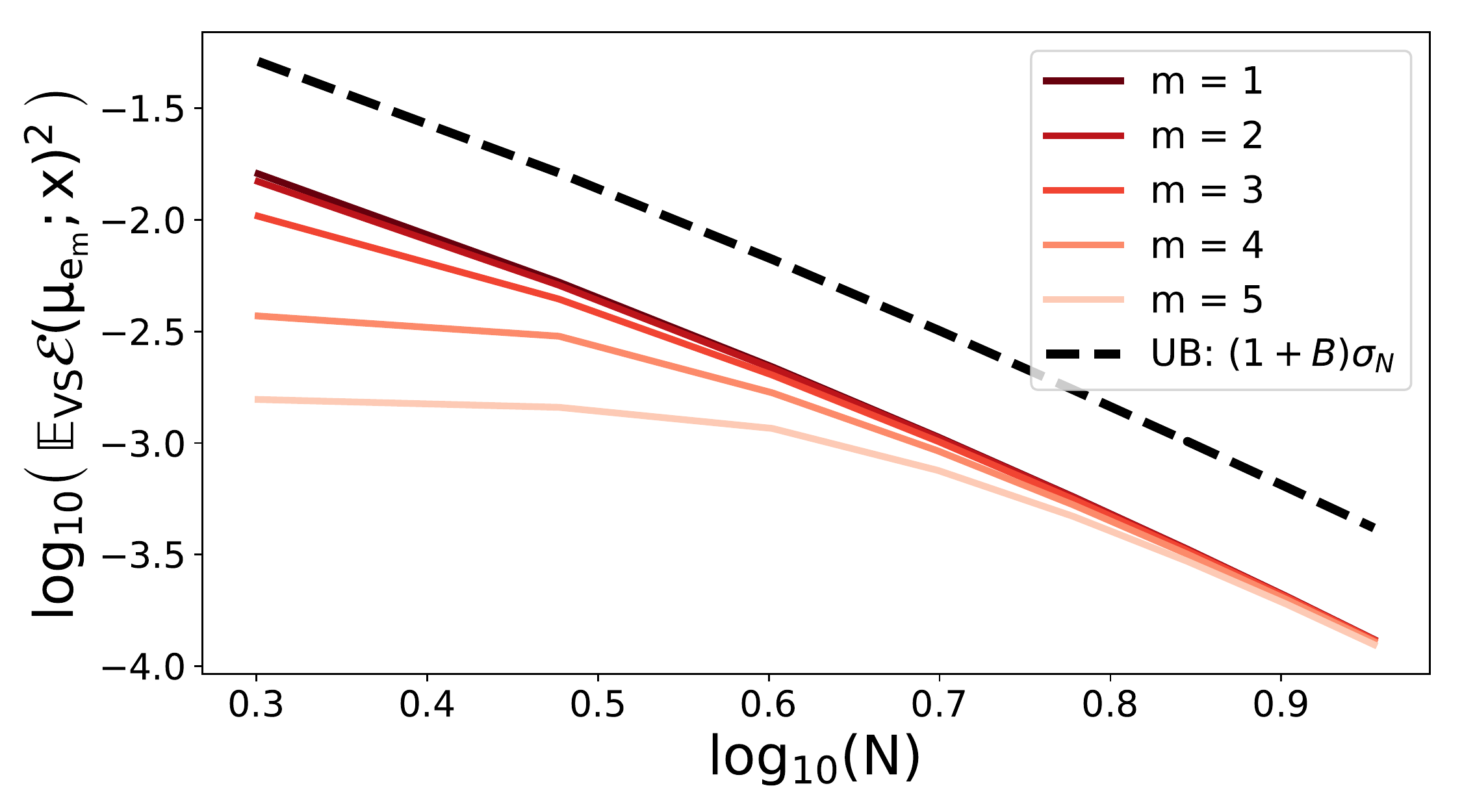}\\
\includegraphics[width= 0.48\textwidth]{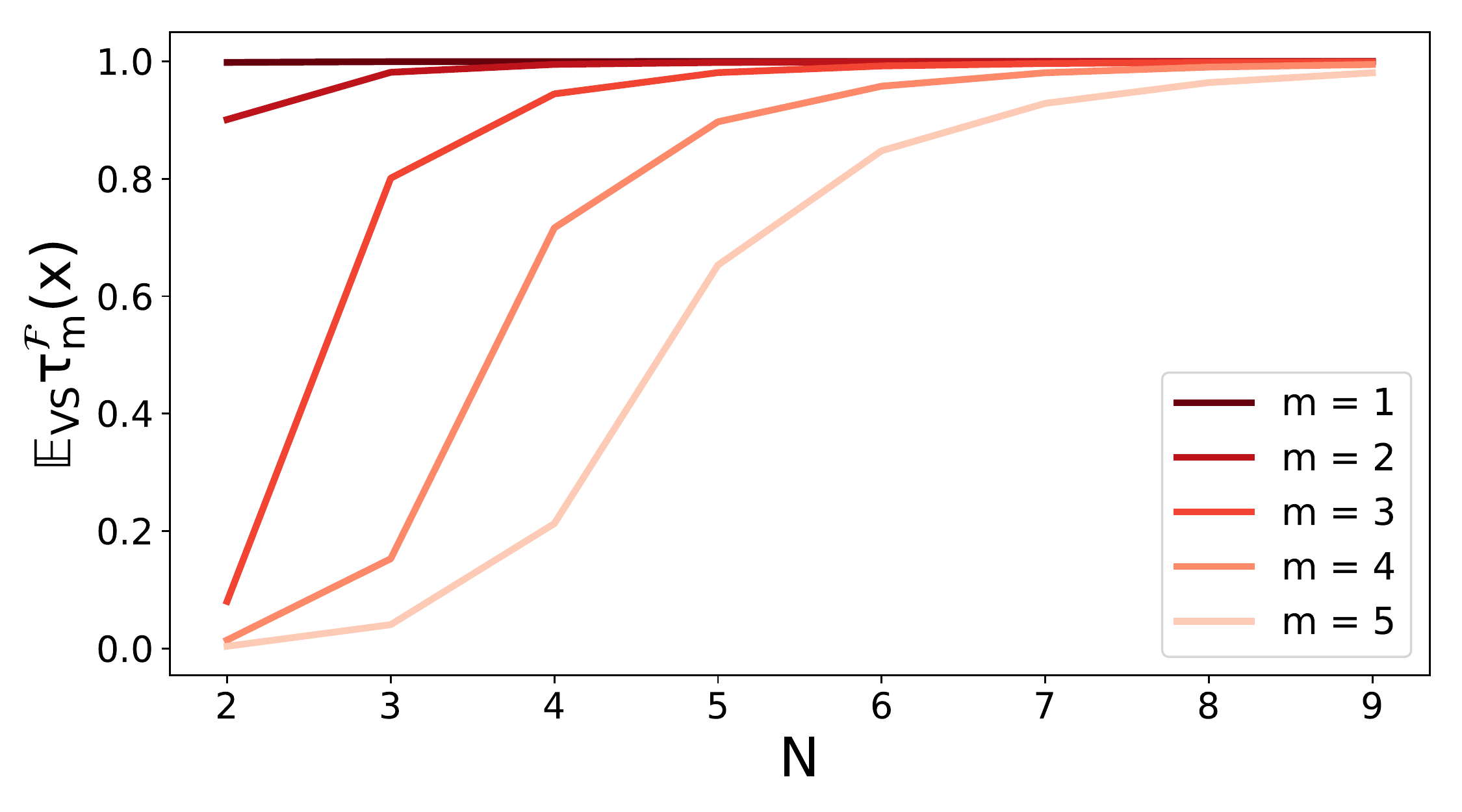}~\includegraphics[width= 0.48\textwidth]{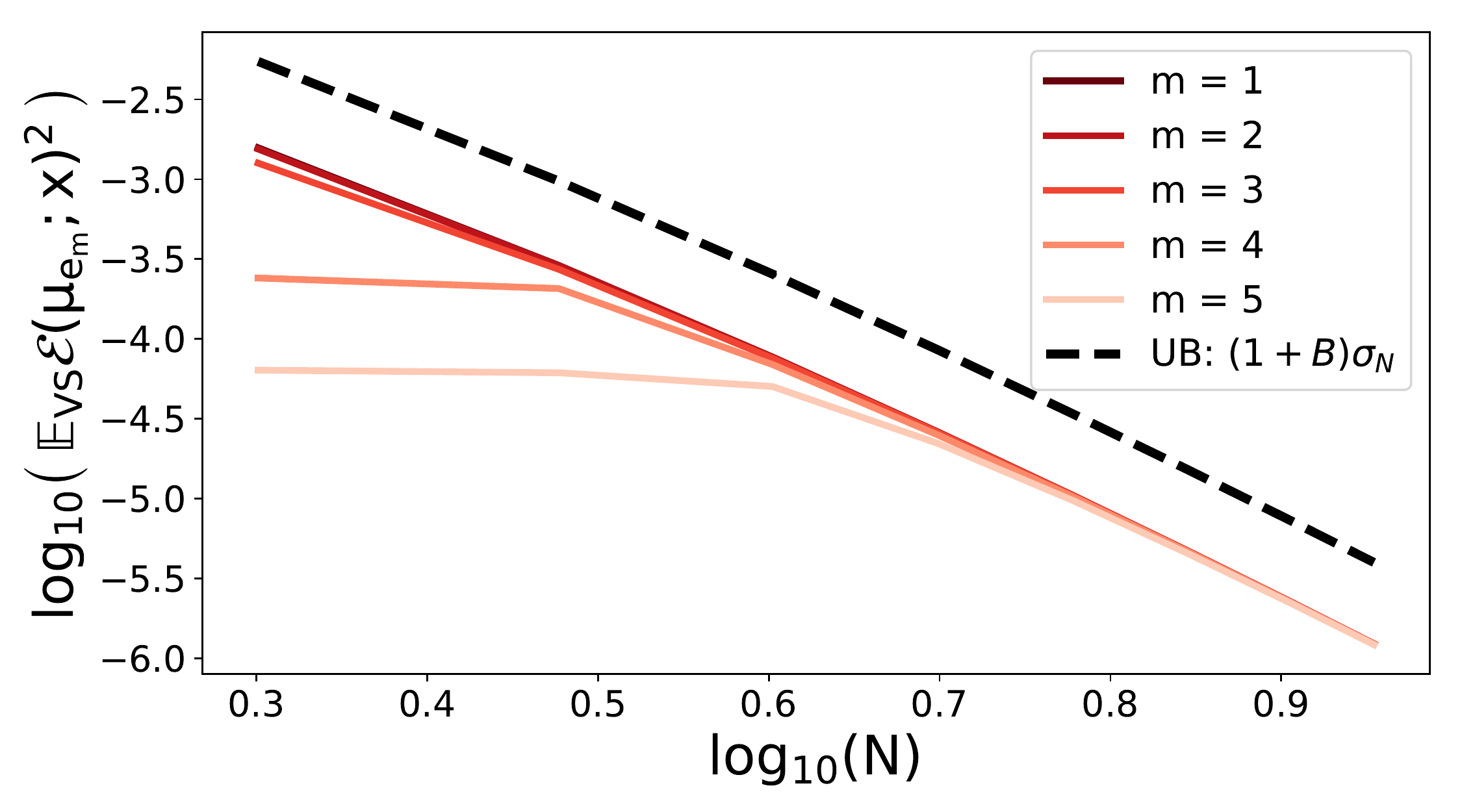}\\
\includegraphics[width= 0.48\textwidth]{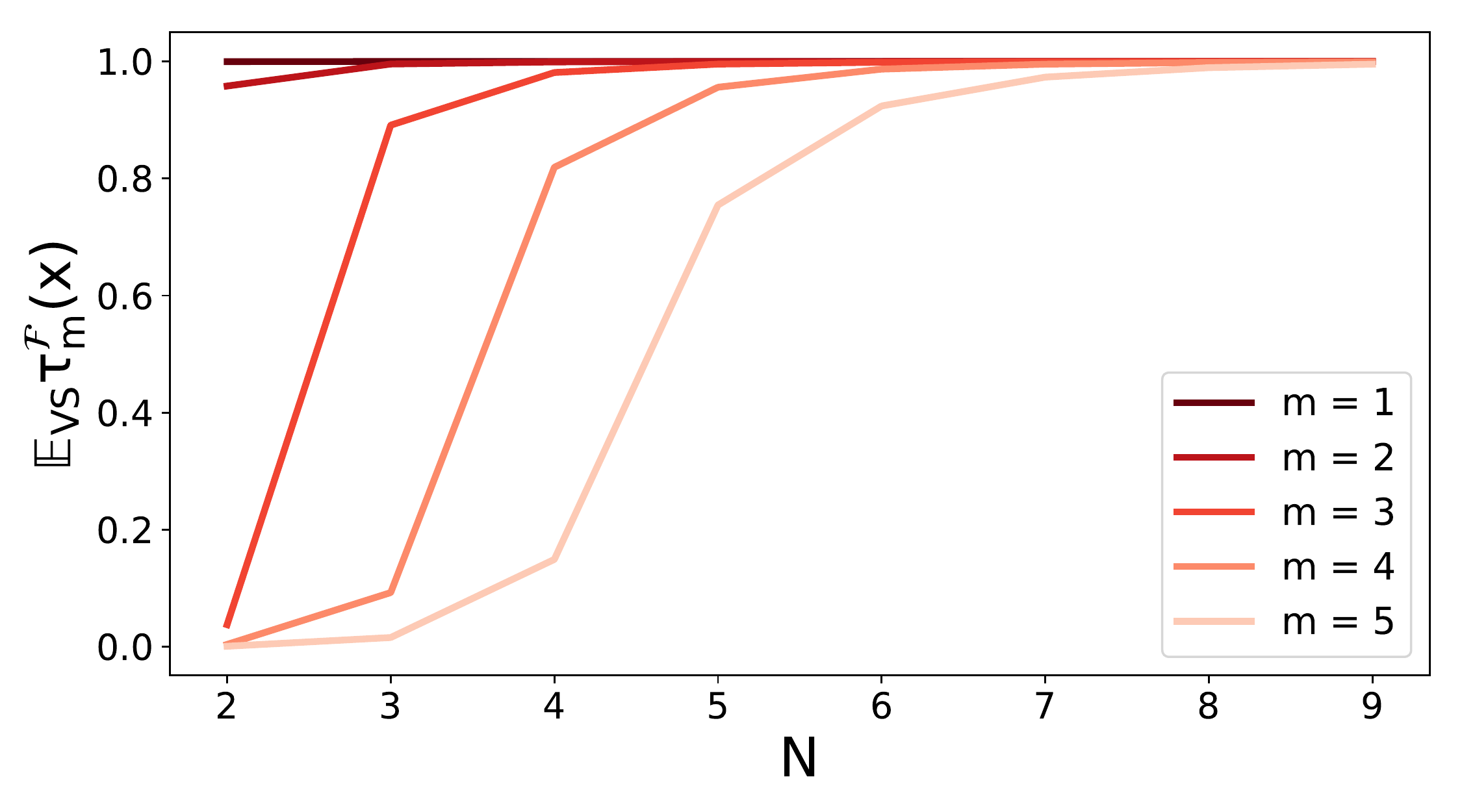}~\includegraphics[width= 0.48\textwidth]{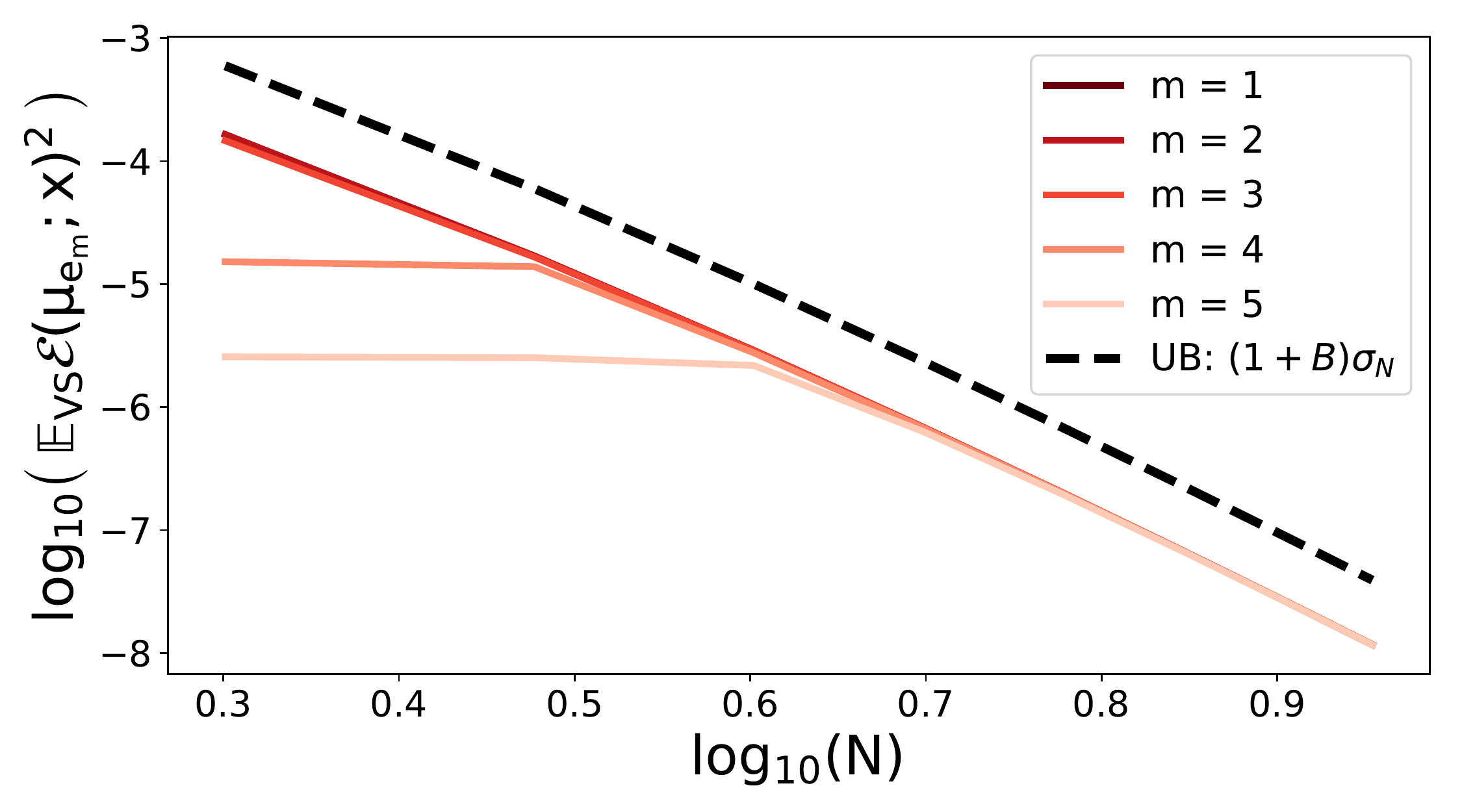}\\
\includegraphics[width= 0.48\textwidth]{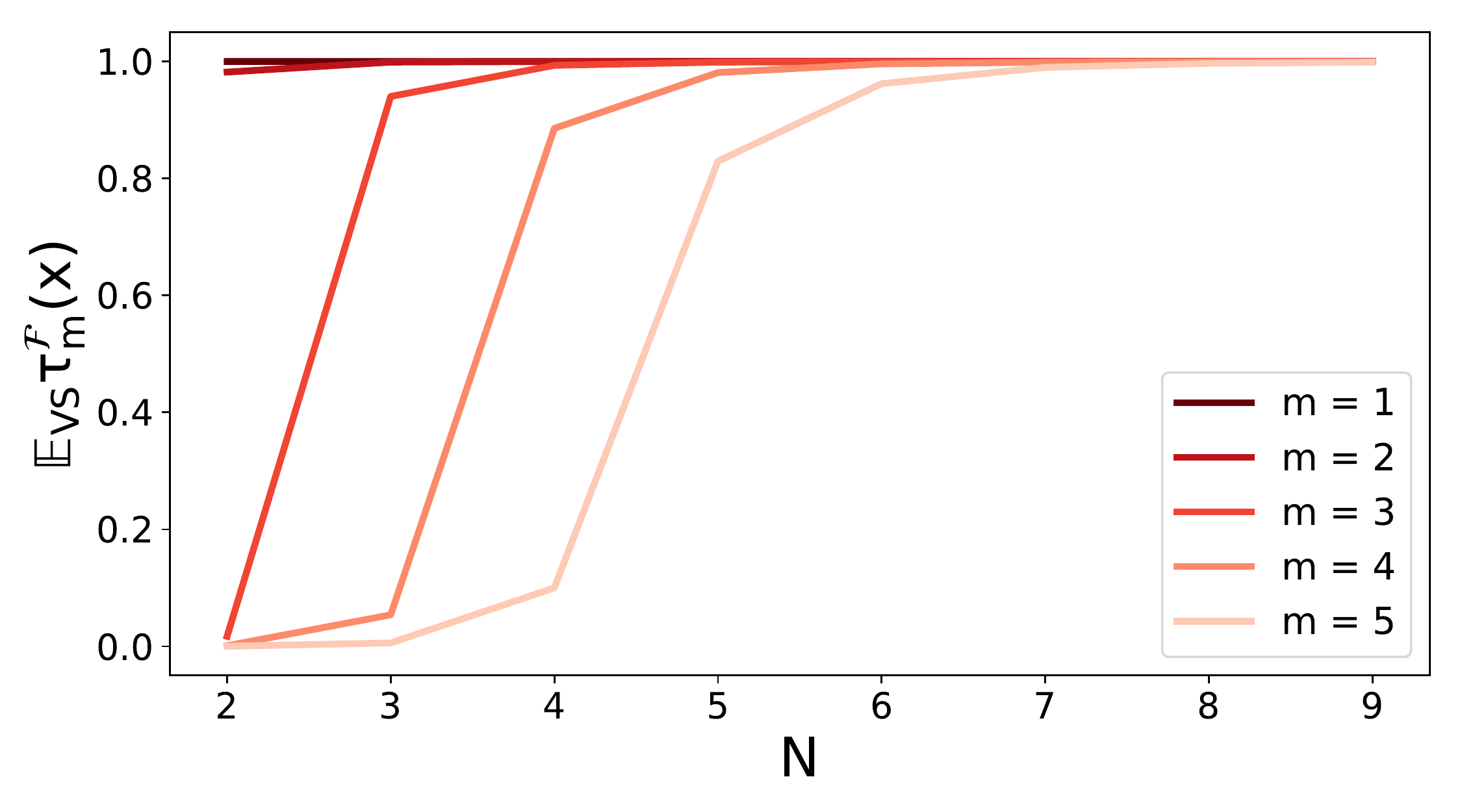}~\includegraphics[width= 0.48\textwidth]{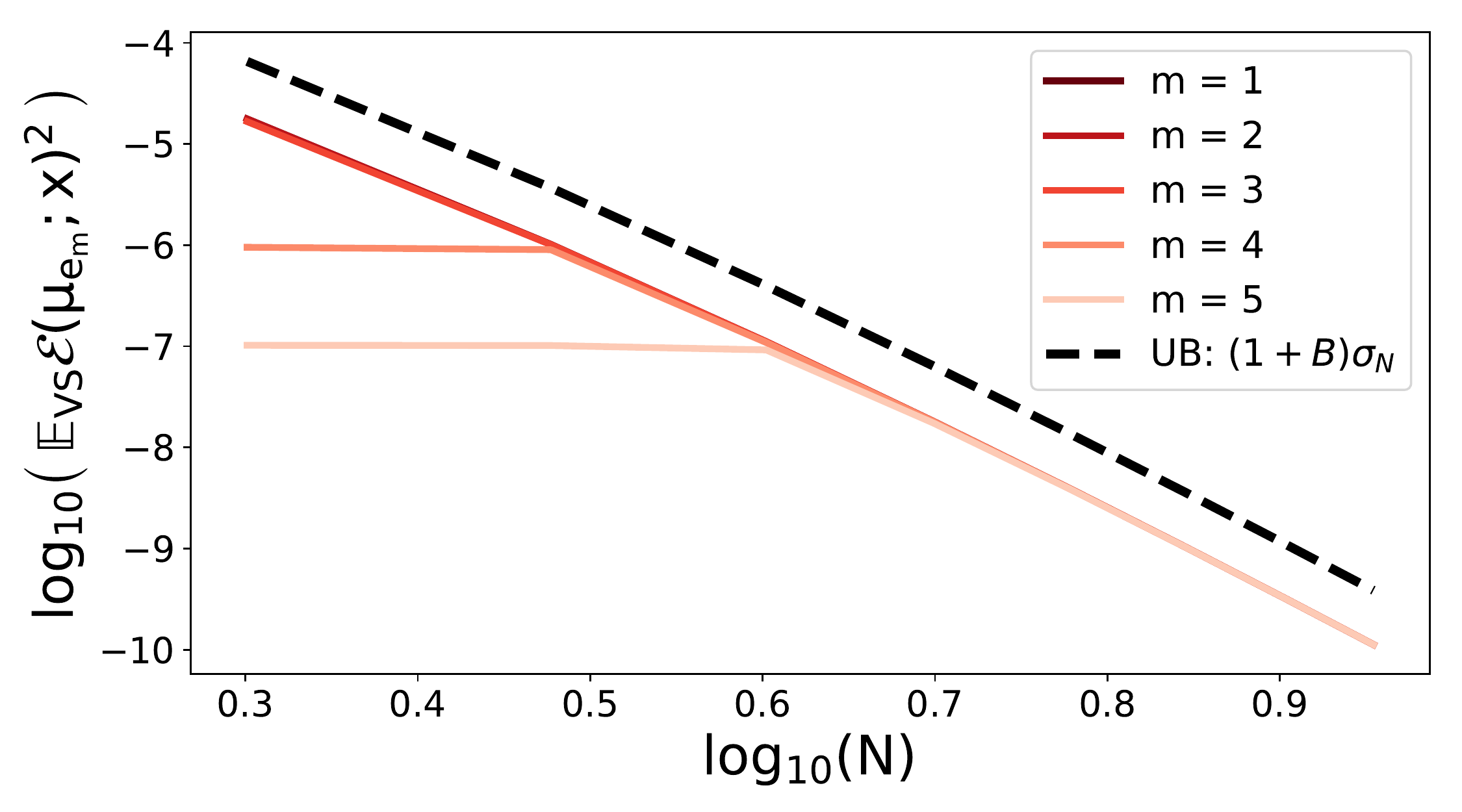}\\
\caption{The expected value of the $m$-th leverage score $\EX_{\VS} \tau_{m}^{\F}(\bm{x})$ (left panels) and the expected interpolation error $\EX_{\VS} \mathcal{E}(\mu_{e_{m}};\bm{x})^{2}$ (right panels), under the distribution of continuous volume sampling, for $m \in \{1,2,3,4,5\}$ and the uni-variate periodic Sobolev kernel.
Rows correspond to increasing values of the smoothness parameter $s=1,2,3,4,5$. \label{fig:periodic_sobolev}}
\end{figure}


\subsection{The uni-dimensional Gaussian spaces}

Consider now the RKHS generated by Gaussian kernel and the Gaussian measure. We take for all $m \in \Ns$, $\sigma_{m}=\alpha^{N}$  \citep{ZhWiRoMo97}, for some $\alpha \in [0,1[$.
Figure~\ref{fig:gaussian} illustrates the expected value of the $m$-th leverage score $\EX_{\VS} \tau_{m}^{\F}(\bm{x})$ (left panels) and the expected interpolation error $\EX_{\VS} \mathcal{E}(\mu_{e_{m}};\bm{x})^{2}$ (right panels), both as functions of $N$, for different values of $m$ and $\alpha \in \{0.7,0.5,0.2\}$. The numerical simulation of
$\EX_{\VS} \mathcal{E}(\mu_{e_{m}};\bm{x})^{2}$ uses again the approximation \eqref{eq:error_numerical_approximation}.

We make the same observations on the dependency of $\EX_{\VS} \tau_{m}^{\F}(\bm{x})$ on $N$ as in the Sobolev case. The rougher the kernel (i.e., the lower the value of $\alpha$), the smoother the transition of $\EX_{\VS} \tau_{m}^{\F}(\bm{x})$ as a function of $N$. Moreover, $\EX_{\VS} \mathcal{E}(\mu_{e_{m}};\bm{x})^{2}$ respects the upper bound of Theorem~\ref{thm:main_result_1};
the constant $B$ of Proposition~\ref{prop:constant_bound} is small for low values of $\alpha$ and converges to $0$ when $\alpha \rightarrow 0$.

\begin{figure}[h]
    \centering
\includegraphics[width= 0.48\textwidth]{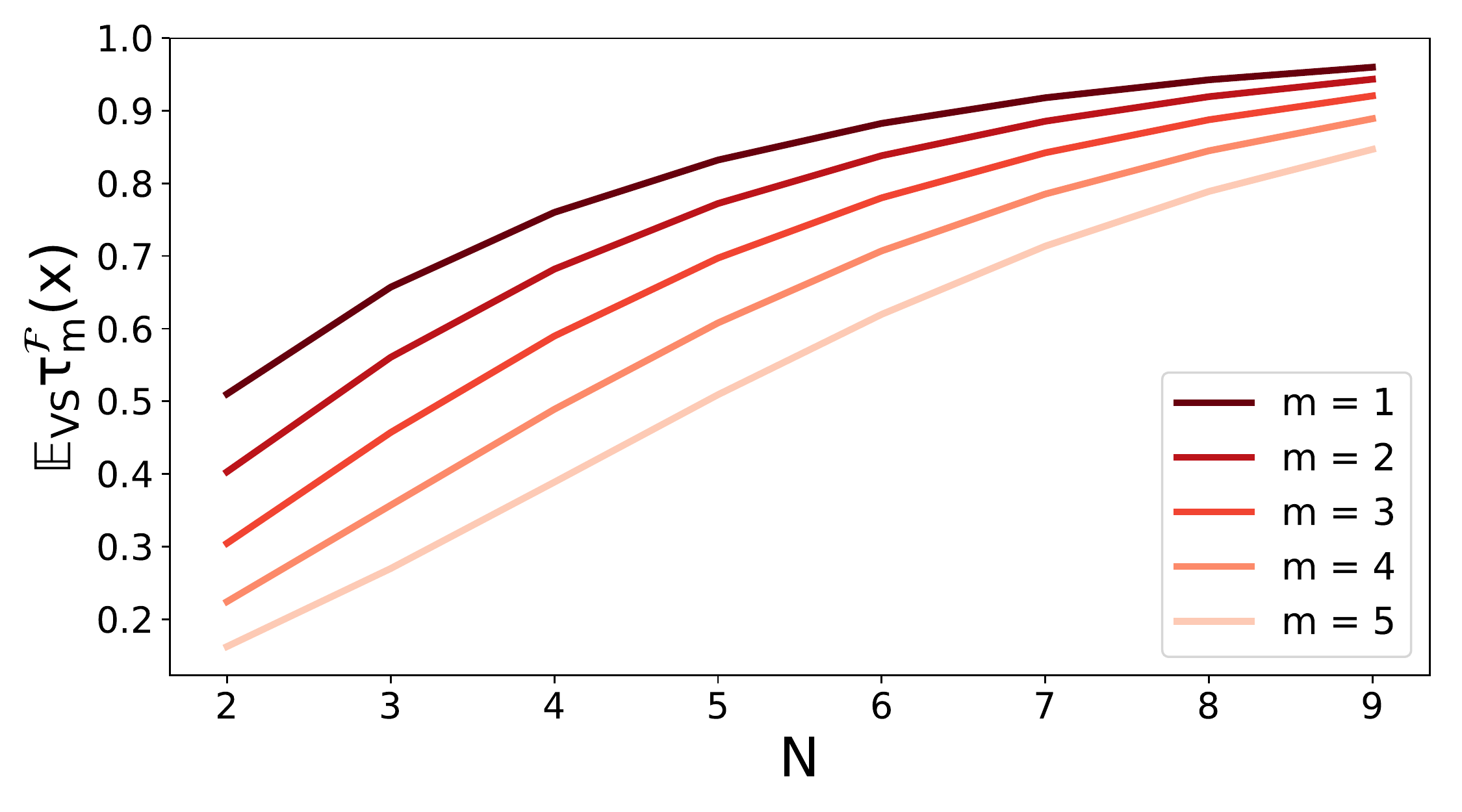}~\includegraphics[width= 0.48\textwidth]{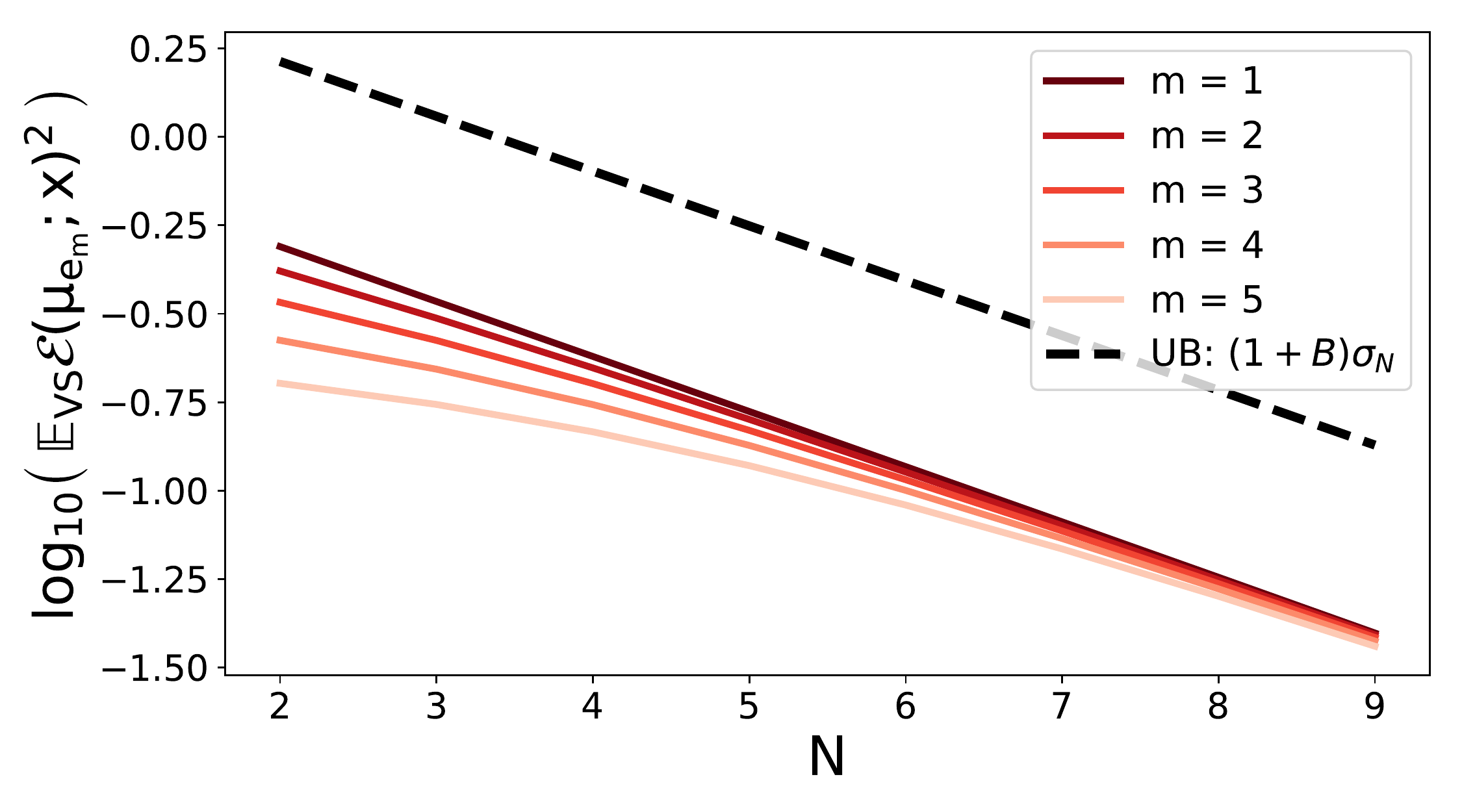}\\
\includegraphics[width= 0.48\textwidth]{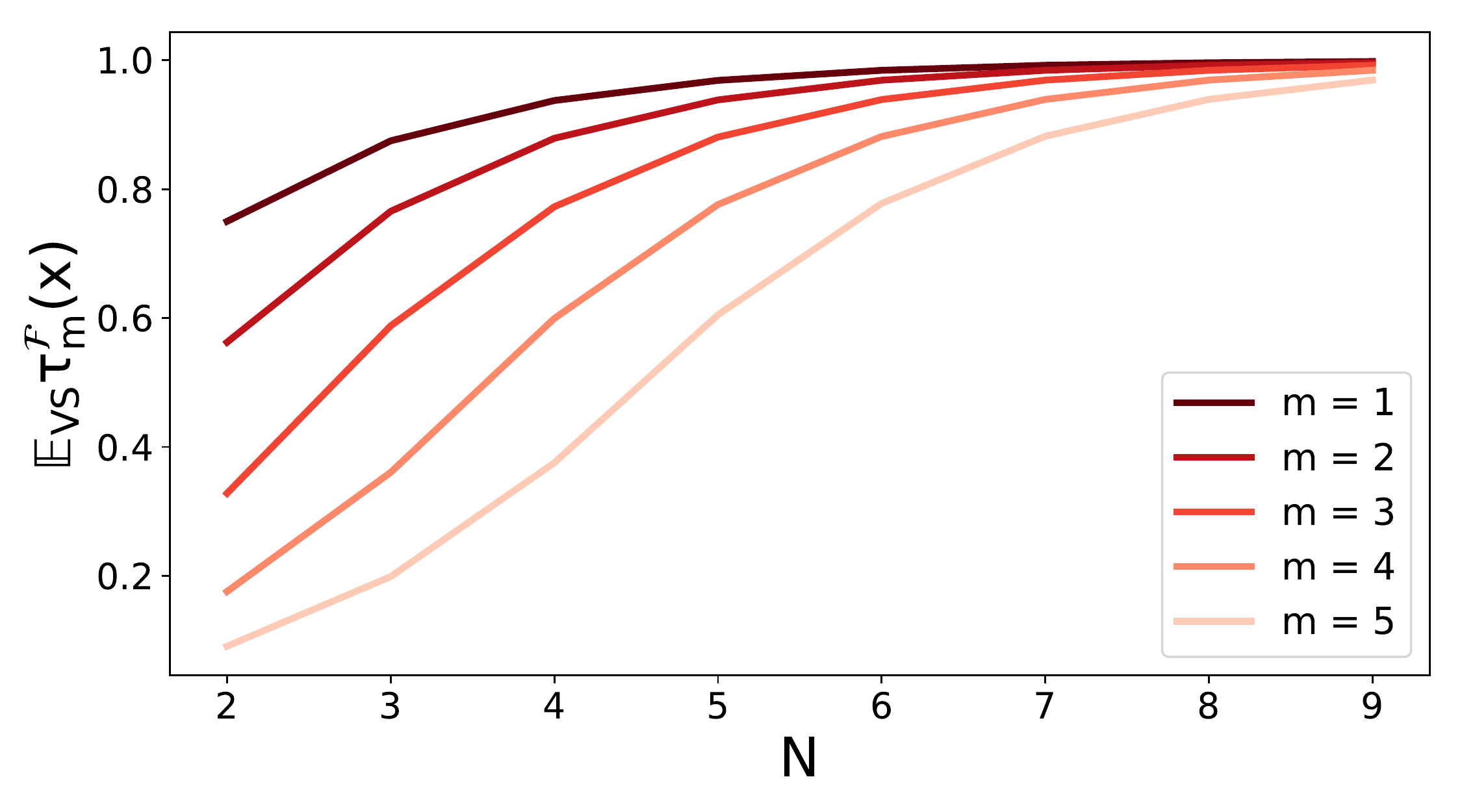}~\includegraphics[width= 0.48\textwidth]{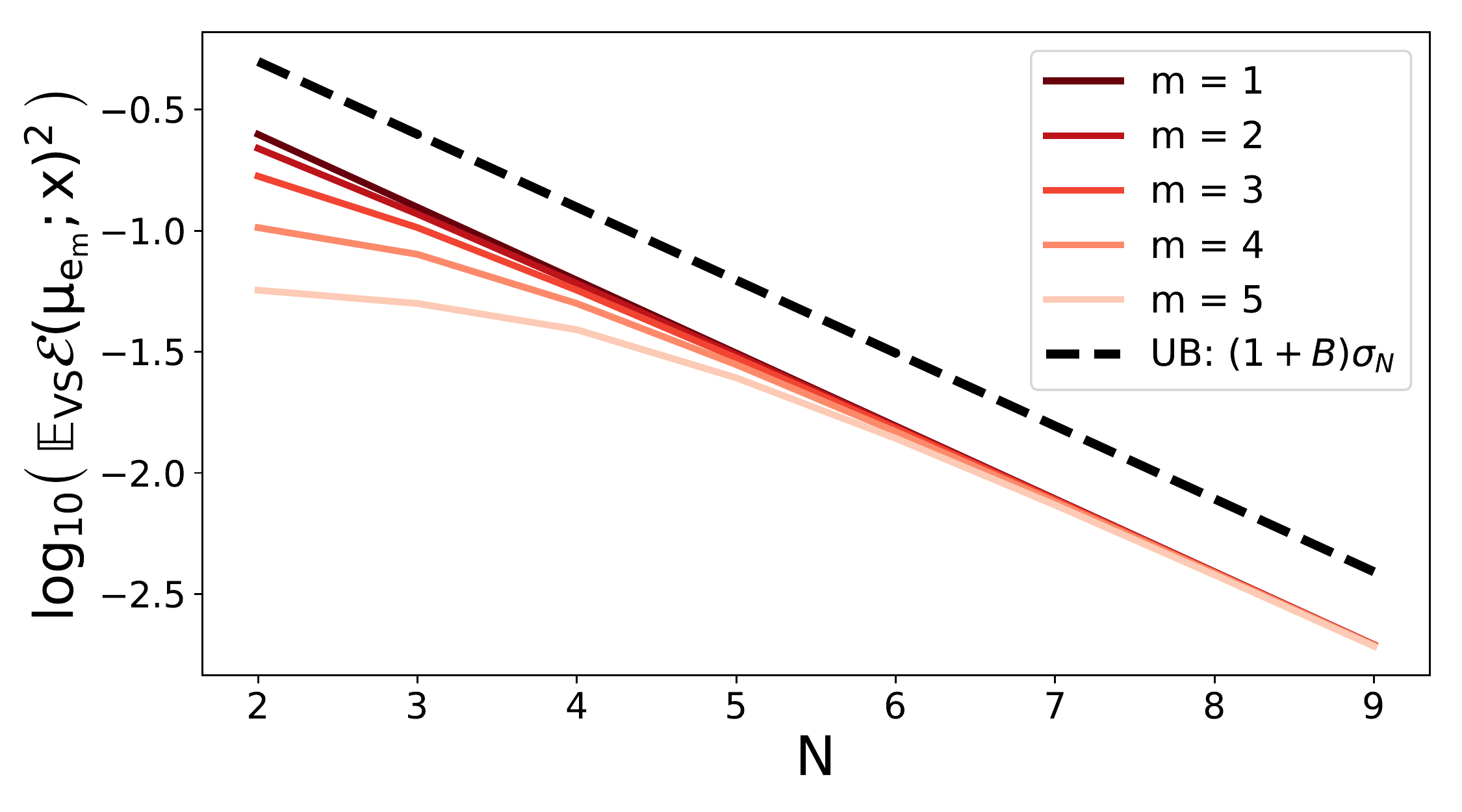}\\
\includegraphics[width= 0.48\textwidth]{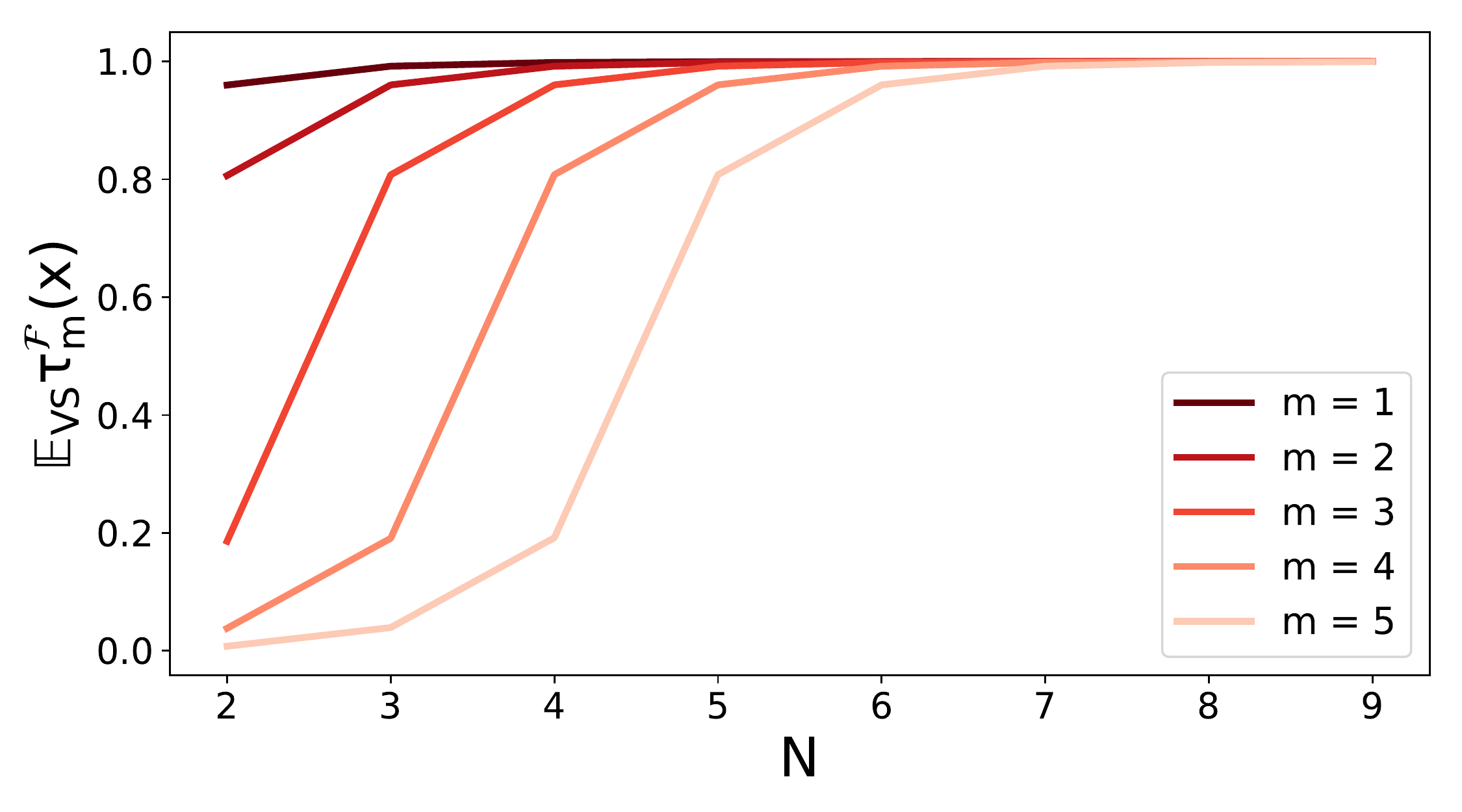}~\includegraphics[width= 0.48\textwidth]{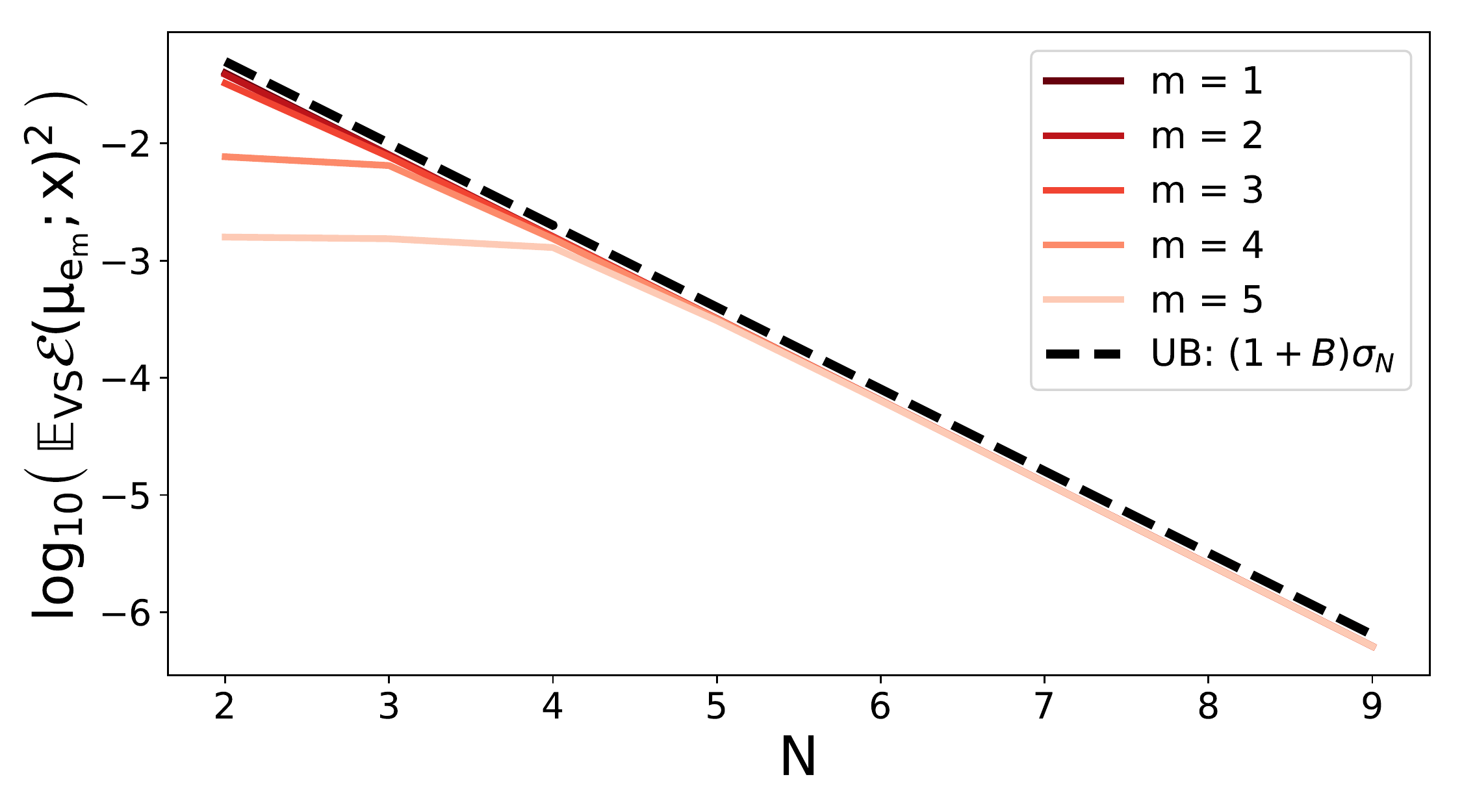}\\
\caption{The expected value of the $m$-th leverage score $\EX_{\VS} \tau_{m}^{\F}(\bm{x})$ and the expected interpolation error $\EX_{\VS} \mathcal{E}(\mu_{e_{m}};\bm{x})^{2}$ under the distribution of continuous volume sampling for $m \in \{1,2,3,4,5\}$. Every row corresponds to a uni-dimensional Gaussian space ($\sigma_{m} = \alpha^{m}$) with a parameter $\alpha \in \{0.7,0.5,0.2\}$. \label{fig:gaussian}}
\end{figure}


\end{document}